%% file: deblur.tex
\documentclass[sigconf]{acmart}

\usepackage{booktabs} 

\setcopyright{rightsretained}
\usepackage{amsmath}
\usepackage{epsfig}
\usepackage{graphicx}
\usepackage{algorithm}
\usepackage{algorithmic}
\usepackage{subfigure}
\usepackage{float}
\usepackage{subeqnarray}
\usepackage{cases}
\usepackage{amssymb}
\usepackage{amsmath}
\usepackage{amsthm}
\usepackage{bm}

\newtheorem{thm}{Theorem}

\newtheorem{remark}{Remark}

\begin{document}
	\copyrightyear{2018}
	\acmYear{2018}
	\setcopyright{acmcopyright}
	\acmConference[MM '18]{2018 ACM Multimedia Conference}{October 22--26, 2018}{Seoul, Republic of Korea}
	\acmBooktitle{2018 ACM Multimedia Conference (MM '18), October 22--26, 2018, Seoul, Republic of Korea}
	\acmPrice{15.00}
	\acmDOI{10.1145/3240508.3240565}
	\acmISBN{978-1-4503-5665-7/18/10}
	
	\title{Learning Collaborative Generation Correction Modules \\for Blind Image Deblurring and Beyond}
	
	\author{Risheng Liu$^{1,2,*}$, Yi He$^{1,2}$, Shichao Cheng$^{2,3}$, Xin Fan$^{1,2}$, Zhongxuan Luo$^{1,2,3,4}$}
	
	\affiliation{
		\institution{$^1$DUT-RU International School of Information Science \& Engineering, Dalian University of Technology}
	}
	\affiliation{
		\institution{$^2$Key Laboratory for Ubiquitous Network and Service Software of Liaoning Province}
	}
	\affiliation{
		\institution{$^3$School of Mathematical Science, Dalian University of Technology}
	}
	\affiliation{
		\institution{$^4$Institute of Artificial Intelligence, Guilin University of Electronic Technology}
	}
	\email{{rsliu, xin.fan, zxluo}@dlut.edu.cn,{heyiking, shichao.cheng}@outlook.com}
	\thanks{$^*$Corresponding Author}
	\fancyhead{}
	\renewcommand{\shortauthors}{Risheng Liu, Yi He, Shichao Cheng, Xin Fan, Zhongxuan Luo}
	
	\begin{abstract}
		Blind image deblurring plays a very important role in many vision and multimedia applications. Most existing works tend to introduce complex priors to estimate the sharp image structures for blur kernel estimation. However, it has been verified that directly optimizing these models is challenging and easy to fall into degenerate solutions. Although several experience-based heuristic inference strategies, including trained networks and designed iterations, have been developed, it is still hard to obtain theoretically guaranteed accurate solutions. In this work, a collaborative learning framework is established to address the above issues. Specifically, we first design two modules, named Generator and Corrector, to extract the intrinsic image structures from the  data-driven and knowledge-based perspectives, respectively. By introducing a collaborative methodology to cascade these modules, we can strictly prove the convergence of our image propagations to a deblurring-related optimal solution.
		As a nontrivial byproduct, we also apply the proposed method to address other related tasks, such as image interpolation and edge-preserved smoothing. Plenty of experiments demonstrate that our method can outperform the state-of-the-art approaches on both synthetic and real datasets.
	\end{abstract}
	

	\maketitle
	
	\input{samplebody-conf}

	\bibliographystyle{ACM-Reference-Format}
	\bibliography{reference}
	
\end{document}

%% file: samplebody-conf.tex
\section{Introduction}
\label{sec:intro}
Blind image deblurring is a fundamental component in many multimedia and computer vision applications. This problem involves the estimation of latent sharp image and blur kernel from a blurry observation. The most commonly used formulation for the blurry phenomenon can be given as follows:
\begin{equation}
\mathbf{y}=\mathbf{u}\otimes \mathbf{k}+\mathbf{n},\label{eq:blur_model}
\end{equation}
where $\otimes$ denotes the convolution operator, $\mathbf{y}$, $\mathbf{u}$, $\mathbf{k}$, and $\mathbf{n}$ are the blurry observation, latent clear image, unknown blur kernel, and noises, respectively. 

\begin{figure}[tb]   
	\centering \begin{tabular}{c@{\extracolsep{0.5em}}c}							     \includegraphics[width=0.23\textwidth]{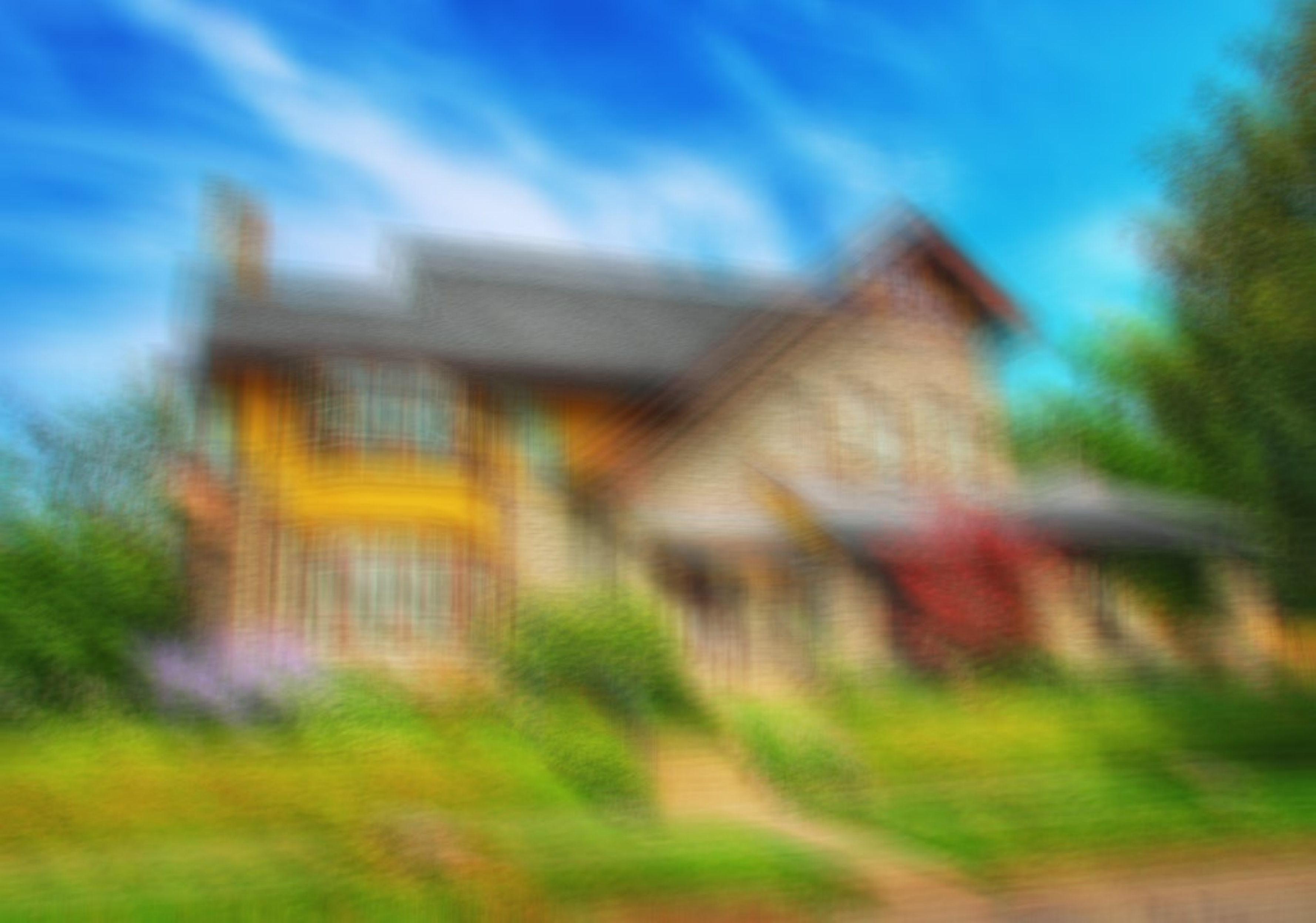}
		&\includegraphics[width=0.23\textwidth]{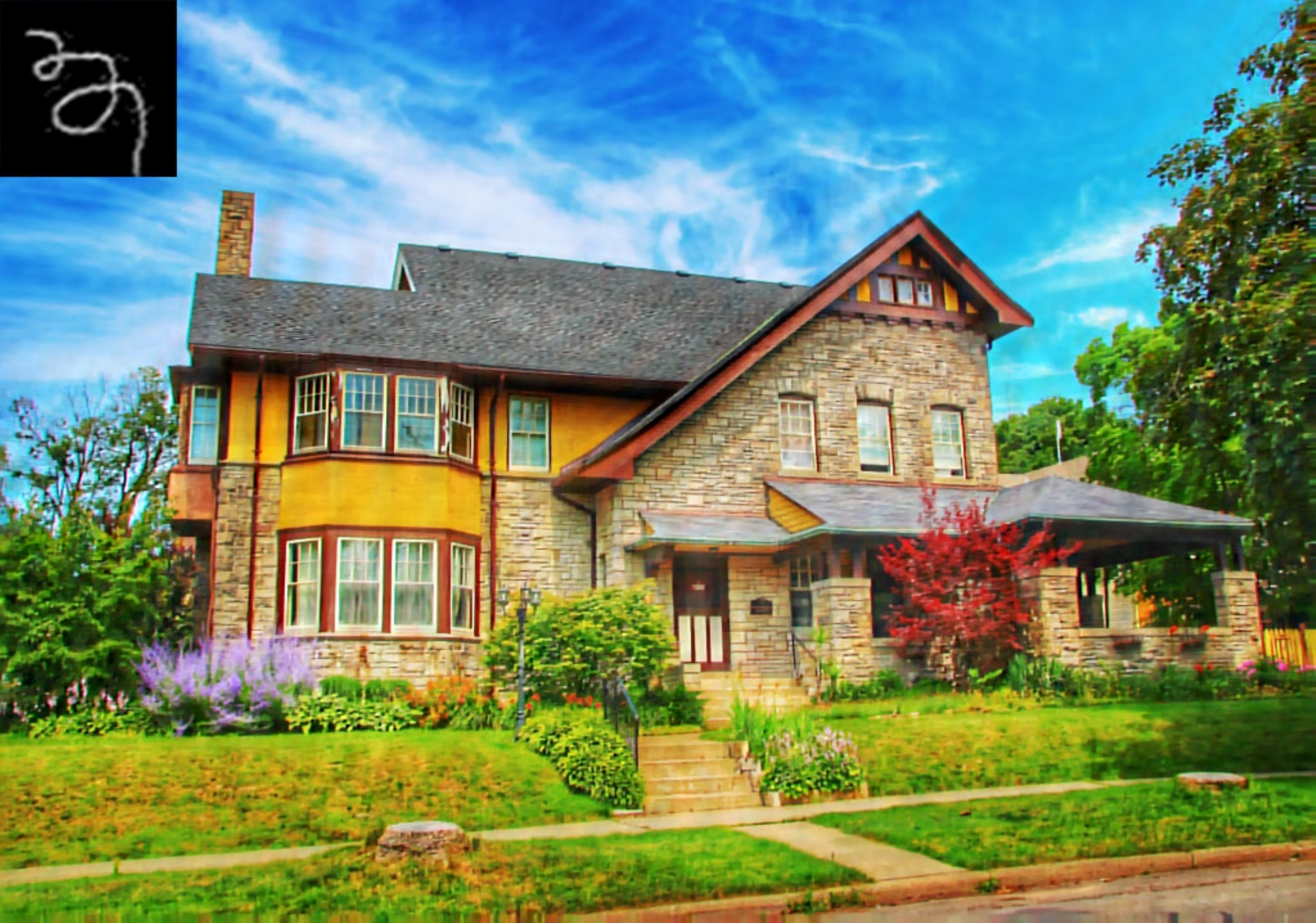}\\
		(a) Input & (b) Deblurring \\
		\multicolumn{2}{c}{\hspace{0em}\includegraphics[width=0.47\textwidth]{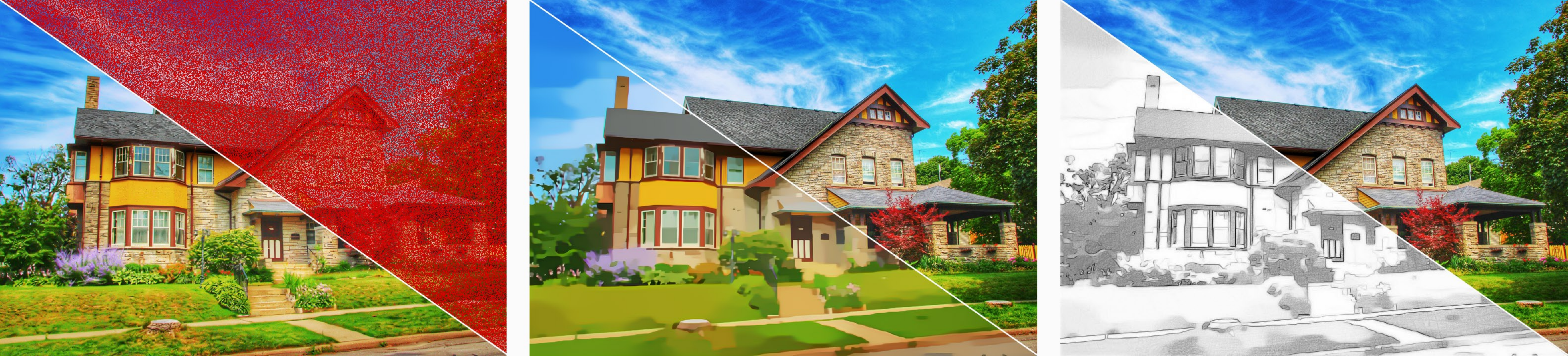}}\\
		\multicolumn{2}{c}{\hspace{0em}(c) Interpolation\hspace{1em}(d) Edge-preserved\hspace{1em}(e) Pencil sketch}\\
		\multicolumn{2}{c}{\hspace{9.5em} smoothing \hspace{4em}  rendering}
	\end{tabular}
	\caption{Illustrating the performance of GCM for blind image deblurring and other applications (e.g., image interpolation, edge-preserved smoothing and pencil sketch rending). On the bottom row, the regions above and below diagonal are the input and output of GCM for these problems, respectively.
	}
	\label{fig:first_figure}
\end{figure}

This problem is highly ill-posed, thus the main attention of most existing researches~\cite{levin2011efficient,sun2013edge,xu2010two,cheng2018designing} focuses on introducing various priors to regularize the solution space, which naturally suggests the Maximum a Posteriori (MAP)~\cite{perrone2014total,sun2013edge,pan2014deblurring} methodology for latent image estimation. Although straightforward, there are many problems with existing MAP-based deblurring approaches. For example, poor priors may lead to ineffective global minimum~\cite{fergus2006removing} and the standard optimization process can only obtain the suboptimal local solutions~\cite{xu2010two}. Therefore, generating useful solutions requires a delicate balancing of various factors such as dynamic noise levels, trade-off parameter values, and heuristic regularizations. Variational Bayesian (VB)~\cite{babacan2009variational,levin2011efficient,babacan2012bayesian} strategy that marginalizes over the whole image space can lead to more accurately focus on the kernel estimation process. Unfortunately, these VB models often involve integrals and hidden variables, thus their inferences are more challenging and time-consuming. 

Recently, some works try to train deep networks to directly estimate the clear images from blurred observations~\cite{Nah2016Deep,sun2015learning,nimisha2017blur}. It can be seen that these methods completely discard the physical principles from models. So they are completely dependent on the training data. However, it is indeed difficult to collect or generate sufficient clear/blur image training pairs in real-world scenarios. Besides, 
the current end-to-end network learning strategies can only be used to remove small blurs and are sensitive to corrupted observations.

\begin{figure*}[tb]
	\centering
	\includegraphics[width=0.9\textwidth]{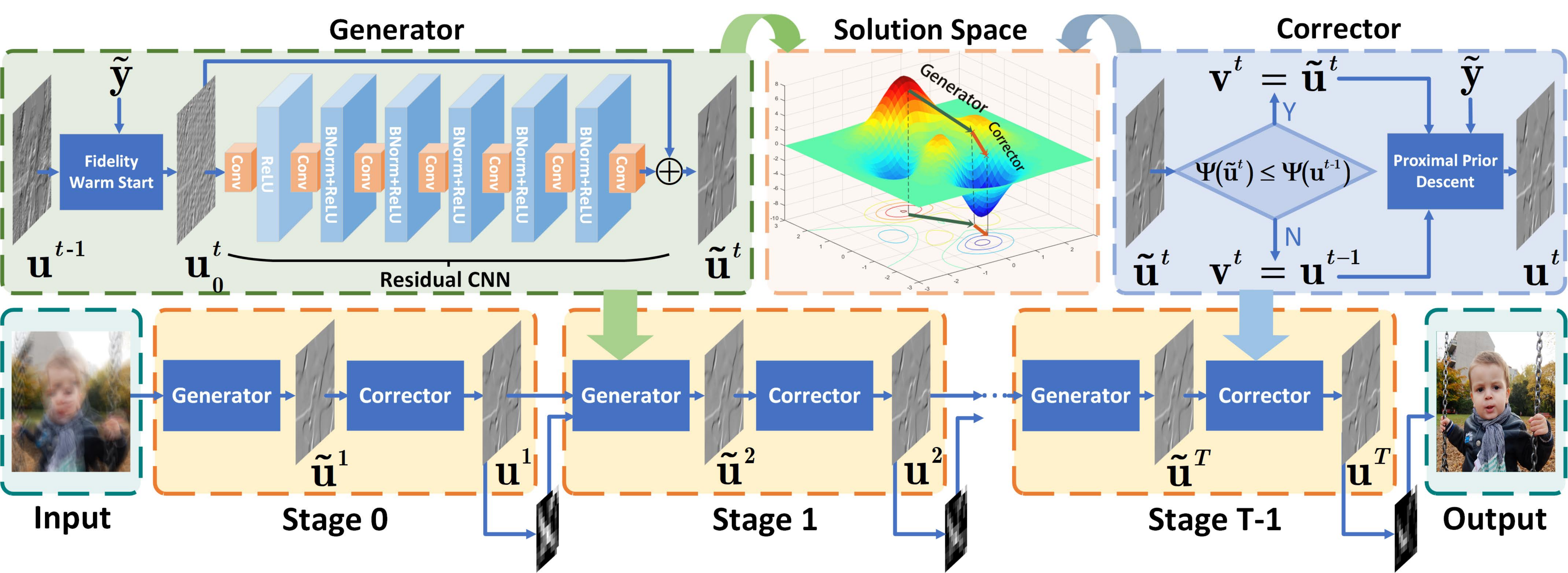}
	\caption{The pipeline of our collaborative GCM. We present the structures of the Generator and Corrector and illustrate how to navigate our model with collaborative effects to obtain the optimal solution on the top row. The main propagation for blind image deblurring is illustrated on the bottom row.}
	\label{fig:process}
\end{figure*}
In this work, we develop a novel collaborative learning framework, named Generation Correction Module (GCM), to integrate the advantages but avoid limitations of the MAP methodology and deep networks for blind image deblurring. Indeed, there are two fundamental building-blocks, (i.e., Generator and Corrector) in GCM. As for the Generator, we implement it with convolutional network architectures and learn parameters on collected training data to extract the latent sharp image structures. On the other hand, we also design a Corrector based on the mathematical image modeling to navigate our image estimation. By introducing a principled strategy to cascade these modules, we finally obtain a convergent image propagation for blind image deblurring. Thanks to the high flexibility of our framework, it is also feasible to apply GCM to other related vision and multimedia applications, such as image interpolation, edge-preserved smoothing, and pencil sketch rendering. Fig.~\ref{fig:first_figure} shows the performance of GCM on various applications.

Our contributions can be summarized as follows:
\begin{itemize}
	\item We establish two fundamental propagative modules (i.e., Generator and Corrector) to respectively learn latent image structures from training data and investigate principled mathematical rules for image propagation navigation.
	\item We provide a collaborative methodology to cascade our Generators and Correctors and prove in theorem that the image sequences generated by GCM can converge to our desired latent solution. We emphasize that our analysis actually also provides a new perspective to design feedforward propagations for deep models.
	\item Extensive experiments demonstrate that GCM not only obtains state-of-the-art results on the problem of blind image deblurring, but also achieves very good performance on a series of other related vision and multimedia applications.
\end{itemize}

\section{Related Works}
\label{sec:related-work}
In this section, we briefly review some related works on the prior models and the existing inference strategies for blind image deblurring. 
Specifically, the most commonly used deblurring formulation is the following regularized variational minimization model
\begin{equation}
\min\limits_{\mathbf{u}}\Psi(\mathbf{u})=f(\mathbf{u}) +  \phi(\mathbf{u}),\label{eq:model-u}
\end{equation}
where $f$ and $\phi$ are the fidelity and prior terms, respectively. In general, $f$ is defined based on the convolution model in Eq.~\eqref{eq:blur_model} as
$f(\mathbf{u})=\frac{1}{2}\|\mathbf{u}\otimes\mathbf{k}-\mathbf{y}\|^2$, 
where $\mathbf{y}$ denotes our observations\footnote{Please notice that in this work $\mathbf{y}$ could represent the image in either the original pixel or gradient domains.}.
Then to alleviate the intrinsic indeterminacy, some prior assumptions are necessary to constrain the space of our candidate solutions.

\textbf{Image Prior Models:} In existing literatures, different types of image priors (i.e., $\phi$ in Eq.~\eqref{eq:model-u}) have been developed to regularize the solution space.
Most MAP approaches focus on designing explicit prior formulations to fit the latent image distributions. 
For example, 
Fergus \emph{et al.}~\cite{fergus2006removing} introduced the heavy-tailed distribution prior on image gradient domain. 
Perrone \emph{et al.}~\cite{perrone2014total} used the Total Variation (TV) prior as the regularization. 
Sun \emph{et al.}~\cite{sun2013edge} aimed to learn patch priors from natural images which can choose the sharp image from blurry ones.
In~\cite{pan2014deblurring}, Pan \emph{et al.} adopted a simple $\ell_0$ prior on both intensity and gradient to handle text images. 
As for VB, the work in~\cite{levin2011efficient} took all possible latent images into consideration and tried to select the best kernel by marginalizing all of them. Babacan \emph{et al.}~\cite{babacan2012bayesian} used super-Gaussian sparse image prior to build a general and flexible method for blind image deblurring. In fact, these manually designed priors often require additional and delicate efforts to balance their trade-off parameters, correct the iteration errors and dynamic noise levels. Very recently, some plug-and-play and network-based priors \cite{Chan2017Plug,zhang2017learning,Zhang2016Learning,liu2017proximal,liu2018toward} have also been introduced to iteratively regularize the latent image estimation process. However, these existing implicit priors can only be used for non-blind image restoration tasks. 

\textbf{Heuristic Inference Strategies:} 
As mentioned above, due to the ill-posedness and complex regularization strategies, standard optimization schemes are indeed not efficient for blind image deconvolution problem. For example, it has been proved in \cite{levin2011efficient} that the exact optimization strategy with poor priors may lead to degenerate global solutions (a.k.a, no-blur solution). Therefore, different heuristic reformulations of subproblems and additional regularizers with turning parameters are introduced for the inference process. 
For example, some works adjust the trade-off parameters~\cite{pan2016blind,cho2009fast} or iteratively change the prior terms~\cite{zuo2016learning} based on experiences to manually control the optimization process to avoid trivial solutions.
Very recently, the learnable strategies~\cite{liu2016learning,chen2015learning,schmidt2016cascades,zhang2017learning,Zhang2016Learning,Kruse2017Learning,liu2018learningaggregated} have also been introduced to help estimate the sharp image structures. However, both the manually designed tricks and trained networks will break the convergence guarantees of the standard optimization schemes. Thus we cannot obtain any theoretical guarantees for existing blind image deblurring methods. Moreover, designing/training these heuristic strategies need extremely delicate skills and extensive experiences.

\textbf{Other Related Applications}: It should be noticed that the schematic variational energy in Eq.~\eqref{eq:model-u} can also be utilized to formulate other computer vision and multimedia tasks.
For example, by defining $f$ with physical rules of different problems and enforcing other task-related priors for $\phi$, a variety of applications, such as image interpolation and edge-preserved smoothing, can all be formulated by Eq.~\eqref{eq:model-u}.

\section{The Collaborative Modules}
In this section, we propose a collaborative framework to learn Generation and Correction Modules (GCM) for latent image propagation. The strict theoretical analysis on GCM is also established at the end of this section.

\subsection{Generator with Fidelity Warm Start}
Inspired by the success of deep networks in visual processing areas, we would like to first establish our Generator as a parameterized network architecture (denoted as $\mathcal{N}$), i.e., at $t$-th stage, we consider 
\begin{equation}
\tilde{\mathbf{u}}^{t+1}=\mathcal{N}^t(\mathbf{u}_0^{t+1};\bm{\omega}^t),
\end{equation} 
where $\bm{\omega}^t$ is the learnable parameters, $\mathbf{u}_0^{t+1}$ and $\tilde{\mathbf{u}}^{t+1}$ are the input and output of the $t$-th Generator (the left zone of the top row in Fig.~\ref{fig:process}), respectively. As for the structure of $\mathcal{N}$, we just adopt a residual CNN module, which consists of seven cascaded ``Convolution+ReLU" blocks.
Following each convolutional layer, a batch normalization trick is also introduced for a stable training process. 

Rather than directly considering the output of the last state (i.e., $\mathbf{u}^t$) as $\mathbf{u}_0^{t+1}$, here we design a fidelity based warm start technique to initialize it as follows
\begin{equation}
\mathbf{u}_0^{t+1}=\arg\min\limits_{\mathbf{u}}  f(\mathbf{u}) + \gamma\|\mathbf{u}-\mathbf{u}^t\|^2,\label{eq:warm-start}
\end{equation}
where $\gamma>0$ is a parameter. It is easy to understand that Eq.~\eqref{eq:warm-start} actually provides a trade-off between the last updated variable (i.e., $\mathbf{u}^t$) and the physical rules of the task (i.e., $f$), thus provide a nice guidance for image propagation. By calculating the closed-form solution of Eq.~\eqref{eq:warm-start} with Fast Fourier Transform~(FFT)~\cite{pan2014deblurring}, we can also consider the warm-start process as our first model-based layer of the Generator.

\subsection{Corrector by Proximal Prior Descent}
Since generating the latent image structure is a highly ill-posed problem, 
only performing Generator may not guarantee the exact recovery of our desired optimal solution. Moreover, no prior knowledge is enforced into the current scheme, thus it is natural to introduce another module to incorporate our prior assumptions of the latent image structure into the propagation. 
Thus we aim to design an architecture to correct the propagation error of the Generator. 

Specifically, our Corrector is designed based on the general variational energy in Eq.~\eqref{eq:model-u} and a monotony-based criterion on the propagated image sequence. That is, we first formulate our prior as $\phi$ in Eq.~\eqref{eq:model-u}\footnote{In general, $\phi$ can be defined to reveal our assumptions of the desired distribution for the latent images. In this work, we just adopt the hype-Laplacian prior~\cite{krishnan2009fast}, thus result to $\ell_{p}$-norm as $\phi$ to navigate image propagation. }. Then by checking the objectives of $\tilde{\mathbf{u}}^{t+1}$, we define a monitor variable $\mathbf{v}^{t+1}$ as $\mathbf{v}^{t+1}=\tilde{\mathbf{u}}^{t+1}$ if $\Psi(\tilde{\mathbf{u}}^{t+1})\leq \Psi(\mathbf{u}^t)$ and $\mathbf{v}^{t+1}=\mathbf{u}^{t}$ otherwise. Then the formal updating of $\mathbf{u}^{t+1}$ can be obtained by optimizing Eq.~\eqref{eq:model-u} with $\mathbf{v}^{t+1}$ by the following proximal-gradient scheme
\begin{equation}
\mathbf{u}^{t+1}\in\mathtt{prox}_{ \phi}^{\mu^t}\left(\mathbf{v}^{t+1}-\mu^t\nabla f(\mathbf{v}^{t+1})\right),\label{eq:pg}
\end{equation}
where $\mathtt{prox}_{ \phi}^{\mu^t}$ denotes the proximal operation\footnote{Please refer to ~\cite{Zuo2013A} for the calculations of the general $\ell_p$-norm related proximal operations.} of $\phi$ with step size $\mu^t>0$.

Intuitively, we first have that the proposed Corrector actually provides a simple methodology to navigate the image propagations to guarantee the monotony of our objectives. More importantly, we will demonstrate in the following that thanks to the proposed Corrector, we can obtain strictly proved nice convergence properties of our GCM.

\subsection{GCM with Theoretical Guarantee}
We first summarize the complete GCM framework in Alg.~\ref{alg:deconvolution} and express the pipeline of GCM in Fig.~\ref{fig:process}. Please notice that due to the CNN-based Generator, GCM is indeed not a standard optimization scheme. \emph{Thus existing convergence analysis is not available for the propagations generated by our GCM.} But fortunately, we will demonstrate in the following theorem that even with the inexact and learnable architectures, the convergence of GCM can still be strictly guaranteed.

\begin{thm}\label{theo:convergence}
	Let $\{\mathbf{u}^t\}$ be the image sequence generated by our GCM (i.e., Alg.~\ref{alg:deconvolution}). Then we have that the objectives (i.e., $\Psi(\mathbf{u}^t)$) are non-increasing, i.e., $\Psi(\mathbf{u}^{t+1})\leq\Psi(\mathbf{u}^{t})$ for $t=0,1,\cdots$. Moreover, any accumulation point of $\{\mathbf{u}^t\}$ is just the critical point of $\Psi$ (i.e., it satisfies the first-order necessary optimal condition of Eq.~\eqref{eq:model-u}). 
\end{thm}
\begin{proof}
Please see \textbf{Appendix}~\ref{app:appendix} for the proof.	
\end{proof}
\begin{remark}
	First, it should be emphasized that Theorem~\ref{theo:convergence} actually reveals that even with the CNN-based Generator, we can still obtain some nice convergence guarantees to GCM. Please notice that our results are even no less than these mathematically designed first-order numerical algorithms in nonconvex optimization areas (e.g., \cite{li2015accelerated,li2017convergence}). 
	
	On the other hand, from the deep learning perspective, our GCM actually provides a simple and generic methodology to guide the design of network architectures to obtain the convergent feedforward variable propagations. Thus our above theoretical results should also provide insights to other related learning, vision and multimedia areas.
\end{remark} 

\begin{algorithm}
	\caption{Generation Collaboration Module (GCM)}\label{alg:deconvolution}
	\begin{algorithmic}[1]
		\REQUIRE The observation $\mathbf{y}$ and necessary parameters.
		\ENSURE Latent image estimation $\mathbf{u}^T$.		
		\STATE Initialization $\mathbf{u}^0 = \mathbf{y}, \gamma, 0<\mu^t < 1/L $;
		\FOR{$t=0,\dots,T-1$}
		\STATE \% \textbf{Generator (i.e., Steps 4-5)}:
		\STATE $\mathbf{u}_0^{t+1} = \arg\min\limits_{\mathbf{u}}  f(\mathbf{u}) + \gamma\|\mathbf{u}-\mathbf{u}^t\|^2$;
		\STATE $\tilde{\mathbf{u}}^{t+1}=\mathcal{N}^t(\mathbf{u}_0^{t+1};\bm{\omega}^t) $;
		\STATE \% \textbf{Corrector (i.e., Steps 7-12)}: 
		\IF {$\Psi(\tilde{\mathbf{u}}^{t+1}) \leq \Psi(\mathbf{u}^{t})$ }	\label{step:criterion-begin}
		\STATE $\mathbf{v}^{t+1} = \tilde{\mathbf{u}}^{t+1}$;
		\ELSE
		\STATE $\mathbf{v}^{t+1} = \mathbf{u}^t $; 
		\ENDIF \label{step:criterion-end}
		\STATE $\mathbf{u}^{t+1} \in \mathtt{prox}_{ \phi}^{\mu^t}\left(\mathbf{v}^{t+1}-\mu^t\nabla f(\mathbf{v}^{t+1})\right)$. \label{step:pg_u_new}
		\ENDFOR	
	\end{algorithmic}
\end{algorithm}

\section{Applications}
Now we demonstrate how to apply our GCM to address blind image deblurring. Thanks to the flexibility of our Generation and Correction modules, GCM indeed can also be applied to address other related multimedia applications.

\subsection{Blind Image Deblurring}
As discussed in Sec.~\ref{sec:related-work}, generating latent image with rich salient edges and sharp structures often plays important role in blind image deblurring. Therefore, we would like to train our Generator in image gradient domain as follows. We first calculate the warm start with the fidelity $f(\mathbf{u}) = \|\mathbf{u}\otimes\mathbf{k}-\mathbf{y}\|^2 $, where $\mathbf{y}$ denotes the observation in gradient domain. Then we train the network architectures $\mathcal{N}$ using the solution of the warm start process (as input) and the gradient of clear images (as output). As for Corrector, we adopt the nonconvex $\ell_{0.8}$-norm as our prior term since it can properly preserve the rich structure information of the latent image.

Then it is natural to nest our GCM based image propagation into the kernel estimation process. Here we just follow the most commonly used strategies (e.g., \cite{xu2010two,pan2014deblurring}) to update $\mathbf{k}$ at $k$-th stage as follows
\begin{equation}
\mathbf{k}^{t+1}=\arg\min\limits_{\mathbf{k}\in\Delta}\|\mathbf{u}^{t+1}\otimes\mathbf{k}-\mathbf{y}\|^2 + \eta\|\mathbf{k}\|^2,\label{eq:solve-k}
\end{equation}
where $\Delta=\{\mathbf{k}|\mathbf{1}^T\mathbf{k}=1,[\mathbf{k}]_i\geq 0\}$ denotes the unit simplex, $\eta$ is a trade-off parameter and $\mathbf{u}^{t+1}$ is the output of the $t$-th Corrector. The most widely used coarse-to-fine strategy~\cite{sun2013edge,pan2014deblurring,pan2016blind} is also adopted to improve the robustness of the deblurring process. Finally, the latent image can be obtained by any given non-blind deconvolution method with the final estimated kernel.

\subsection{Byproduct Applications}
As nontrivial byproducts, we would also like to demonstrate how to apply GCM to other image related vision and multimedia applications.

\textbf{Image Interpolation}: The purpose of this task is to remove corrections (e.g., text, blocking or mask) from the partially invisible observation. When using GCM to deal with this problem, the main difference with deblurring is the fidelity model. That is, we just set the fidelity as $f(\mathbf{u}) = \frac{1}{2} \|\mathbf{u} \odot \mathbf{M} - {\mathbf{y}} \|^2$ for Generator. Here $\odot$ denotes the pixel multiplication, $\mathbf{M}$ is the mask matrix and ${\mathbf{y}}$ is the observation in the image domain. Please notice that the training phase is also performed on the original image domain for this task. 

\textbf{Edge-preserved Smoothing}: This is a fundamental image preprocessing step, which aims to remove the redundant textures and noises while preserving the main structures. Many multimedia applications, such as texture and edge extraction \cite{Xu2012Structure} and pencil sketch rending \cite{Lu2012Combining} are based on the results of edge-preserved smoothing. 
As for this task, we just consider the fidelity as $f(\mathbf{u}) = \frac{1}{2} \|\mathbf{u} - \mathbf{y} \|^2$ for Generator and set $\phi$ as $\ell_{0}$-norm for Corrector.

\section{Discussions}
Here we would like to discuss and highlight some important aspects of GCM. 

\subsection{Theoretically Convergent Ensemble Framework for Image Modeling}
Indeed, our GCM can be viewed as a general framework to integrate deep network architectures (i.e., Generator) and physical principles (i.e., Corrector) to address not only blind image deblurring, but also other related vision and multimedia tasks. The main advantage against existing heuristic ensemble strategies (e.g., \cite{zhang2017learning,Kruse2017Learning}) is that we can strictly prove the convergence of the hybrid propagations in GCM (i.e., Theorem~\ref{theo:convergence}), while till now no theoretical guarantees can be provided for these naive combinations in existing works.

\subsection{Analogy to Adversarial Learning Methodology}
Since there exist two cascaded modules in our GCM (i.e., Generator and Corrector), which is similar to that in the popular adversarial learning methods (i.e., Generator and Discriminator in Generative Adversarial Network, GAN for short), we would like to provide some brief comparisons between these two learning methodologies. First, there exists a network-based Generator in both GCM and GAN. But as for the other module, one may just adopt a heuristic classification network as the Discriminator in GAN, while we design our Corrector based on the physical rules (with optimization issue) for the given task. More importantly, different from the adversarial criterion in GAN, there actually exists an implicit collaborative relationship between our Generator and Corrector. Finally, till now it is also difficult to analyze the intrinsic properties of the feedforward propagation generated by GAN. In contrast, we have strictly proved the convergence of our proposed collaborative GCM image propagation .

\section{Experimental Results}
We first conduct experiments to verify the mechanism of GCM. Then a range of results are demonstrated to evaluate GCM on blind image deblurring. Finally, we show the performance of GCM on other related applications, such as image interpolation and edge-preserved smoothing.

\subsection{Experimental Setup} 
To provide fair comparisons on the blind image deblurring problem, we adopt ~\cite{zoran2011learning} as the final non-blind deconvolution process for all the compared methods. 
We also execute these approaches with their default parameter settings. For the training data of GCM, we randomly select 800 natural images from the training set of the ImageNet database~\cite{deng2009imagenet}. Then we crop these images into small patches with the size $35\times35$ and adopt data augmentation to enhance the generalization ability of our network. 
Finally, Adam ~\cite{kingma2014adam} is performed to train our Generator. As for the other algorithmic parameters of GCM, we empirically set $\gamma=4e-3$, $\mu^{t}=1e-6$, and $L=2$. We perform all the experiments on a PC with 8 cores Intel i7 CPU, NVIDIA GTX 1060 GPU and 32 GB RAM.

\subsection{Model Verification}
\label{sec:model_verification}
To verify the efficiency of our proposed collaborative learning strategies, we first compare the deblurring performance of GCM with different settings on an example image.
In Fig.~\ref{fig:GCM}, we denote the naive cascade of the designed Generator and Corrector as ``Generator" and ``Corrector'', respectively. While ``GCM'' denotes our principled ensemble of these two modules as that in Alg.~\ref{alg:deconvolution}. It can be seen in Fig.~\ref{fig:GCM} (a)-(b) that the Corrector almost failed on the blind deblurring problem (i.e., the lowest quantitative scores). This is because that we may not obtain the intrinsic sharp image structures only using the model-based iterations. The Generator obtained a better performance than the Corrector. However, there always exist oscillations at the first several stages. In contrast, we can observe that the ensemble of Generator and Corrector (i.e., GCM) obtained the best quantitative performance.
The visual comparisons of the finally restored images in Fig.~\ref{fig:GCM} (c)-(e) also verified that our collaborative learning strategy obtain much better results than the naive cascade of either Generator or Corrector.

\begin{figure}[tb]
	\centering \begin{tabular}{c@{\extracolsep{0.2em}}c@{\extracolsep{0.2em}}c}
		\multicolumn{3}{c}{\includegraphics[width=0.45\textwidth]{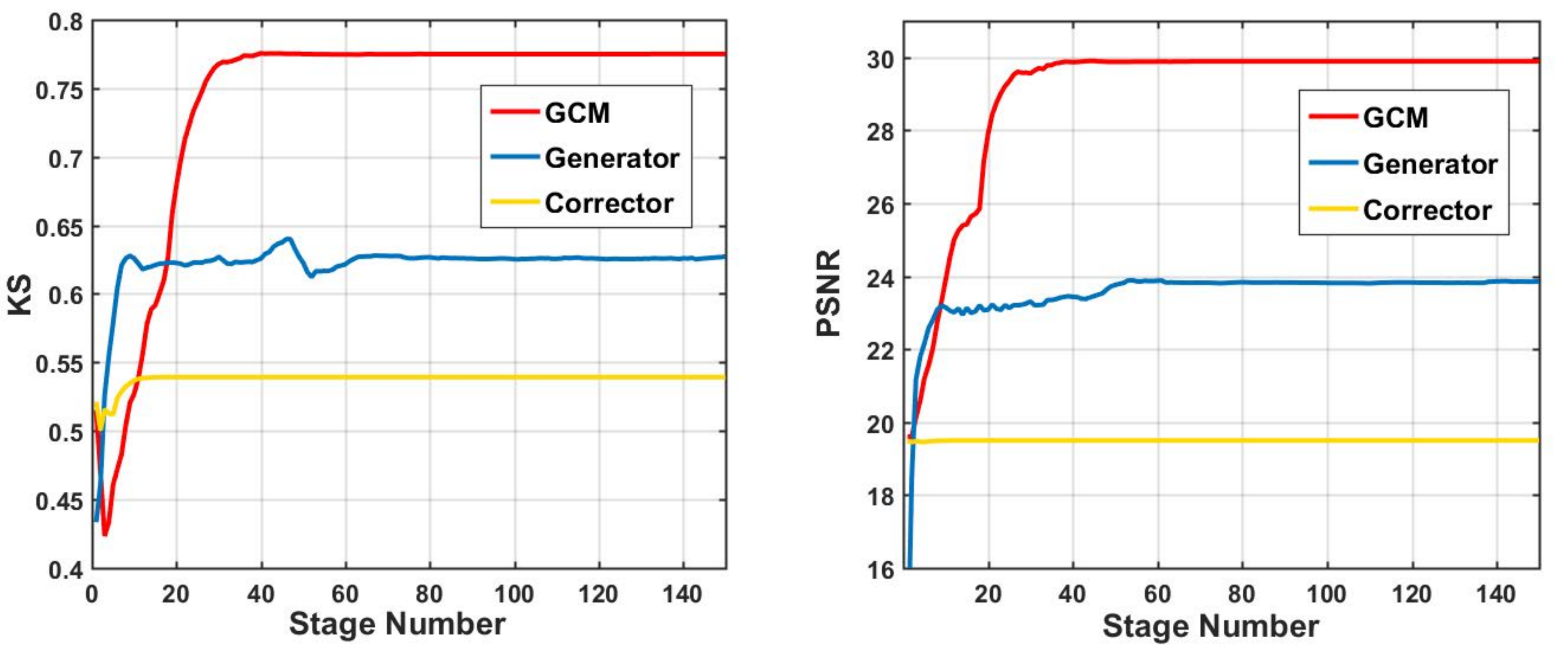}}\vspace{1em}\vspace{-1em}\\
		\multicolumn{3}{c}{\hspace{2em}(a) KS curve\vspace{0.5em} \hspace{6em}(b) PSNR curve}\\
		\includegraphics[width=0.15\textwidth]{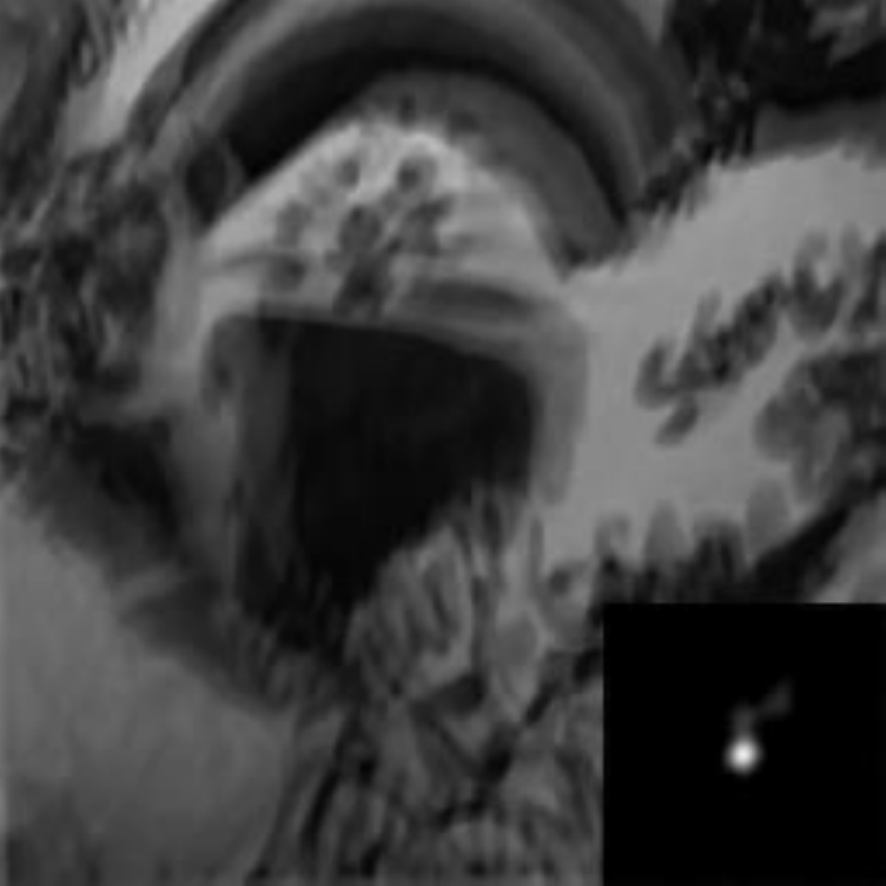}
		&\includegraphics[width=0.15\textwidth]{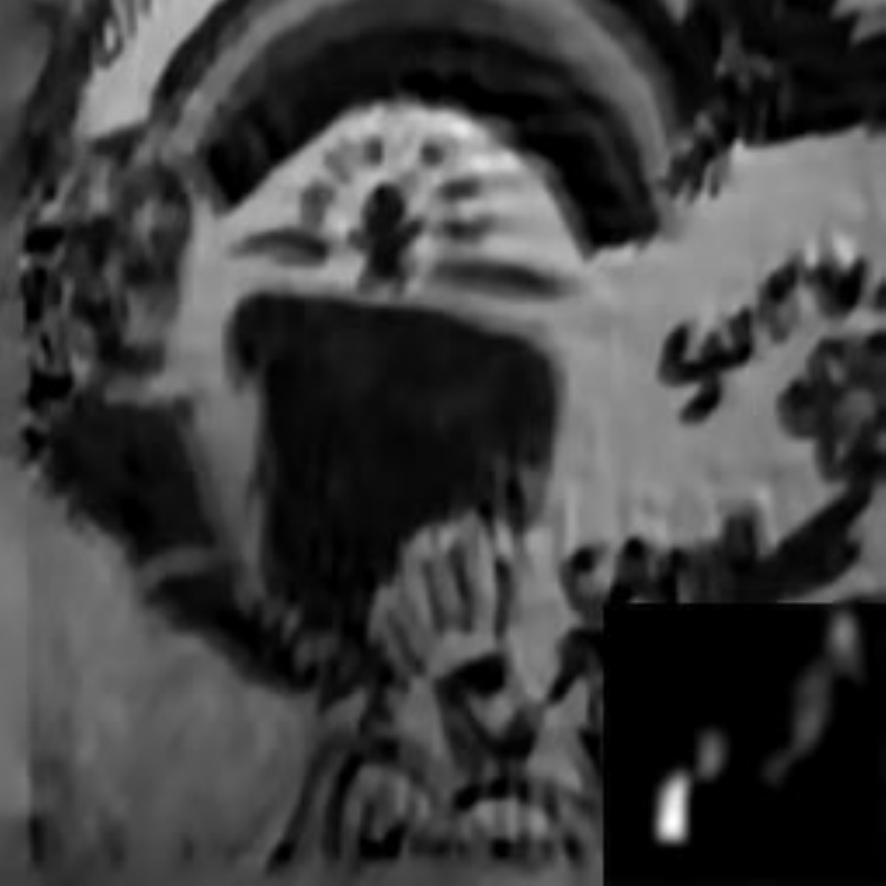}			
		&\includegraphics[width=0.15\textwidth]{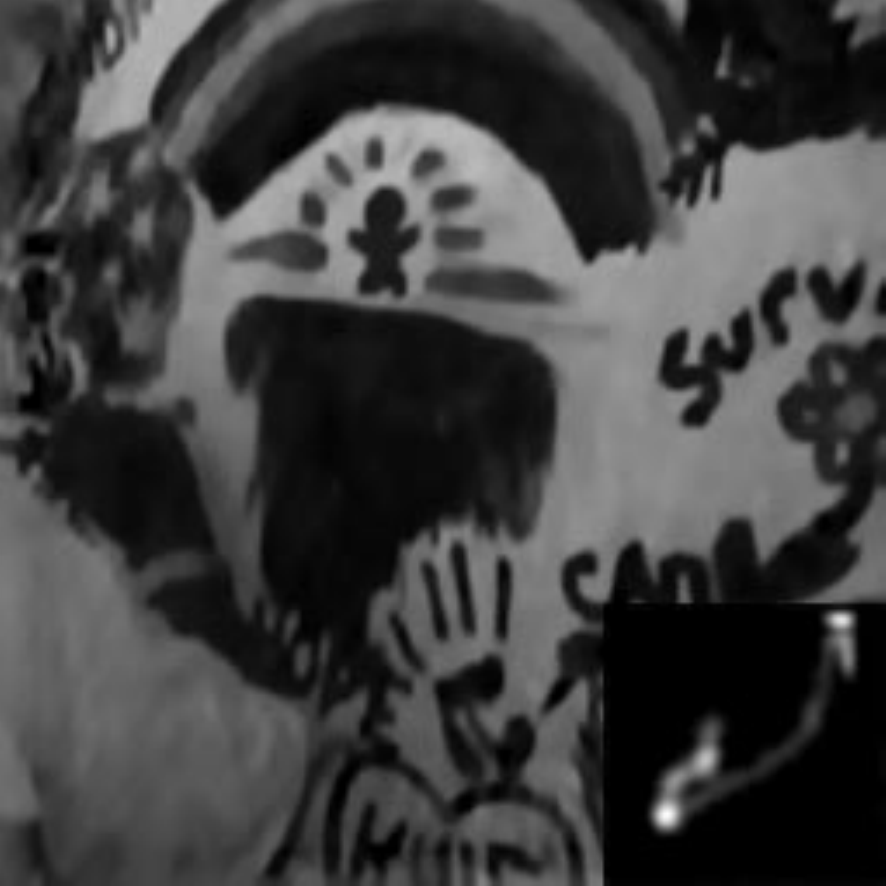}\\
		19.47~/~0.59 &23.34~/~0.71&\textbf{29.62~/~0.90}\\
		(c) Corrector & (d) Generator & (e) GCM  \\
	\end{tabular}
	\caption{Comparisons among Generator, Corrector and their ensemble (i.e., GCM). We plot the curves of Kernel Similarity (KS) and PSNR on subfigures (a) and (b), respectively. The final results of these strategies are also presented in subfigures (c)-(e). Quantitative metrics (PSNR / SSIM) are reported below each result.}
	\label{fig:GCM}
\end{figure}

\begin{figure*}[htb]
	\centering \begin{tabular}{c@{\extracolsep{0.2em}}c@{\extracolsep{0.2em}}c@{\extracolsep{0.2em}}c@{\extracolsep{0.2em}}c}
		\includegraphics[width=0.18\textwidth]{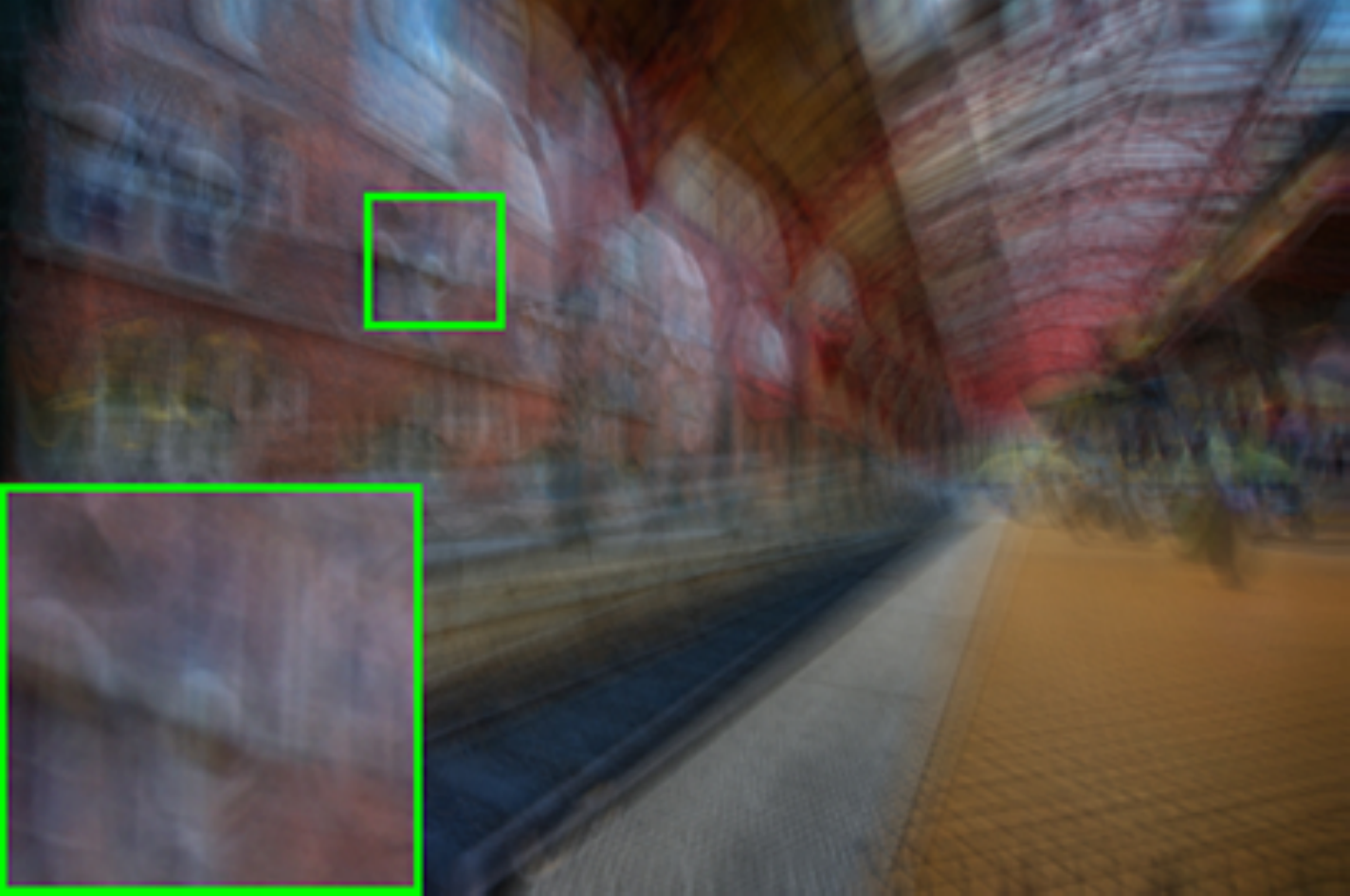}
		&\includegraphics[width=0.18\textwidth]{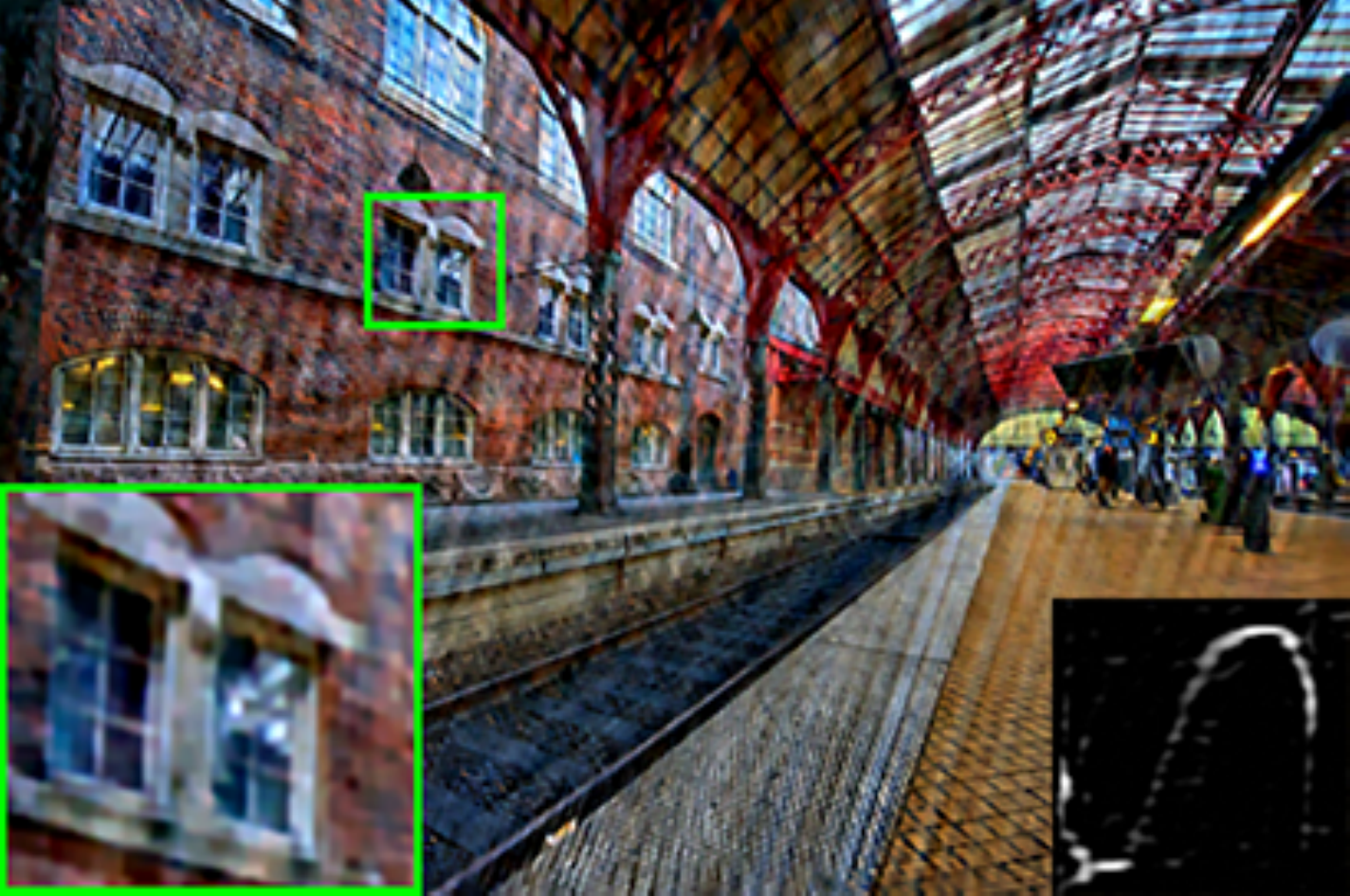}
		&\includegraphics[width=0.18\textwidth]{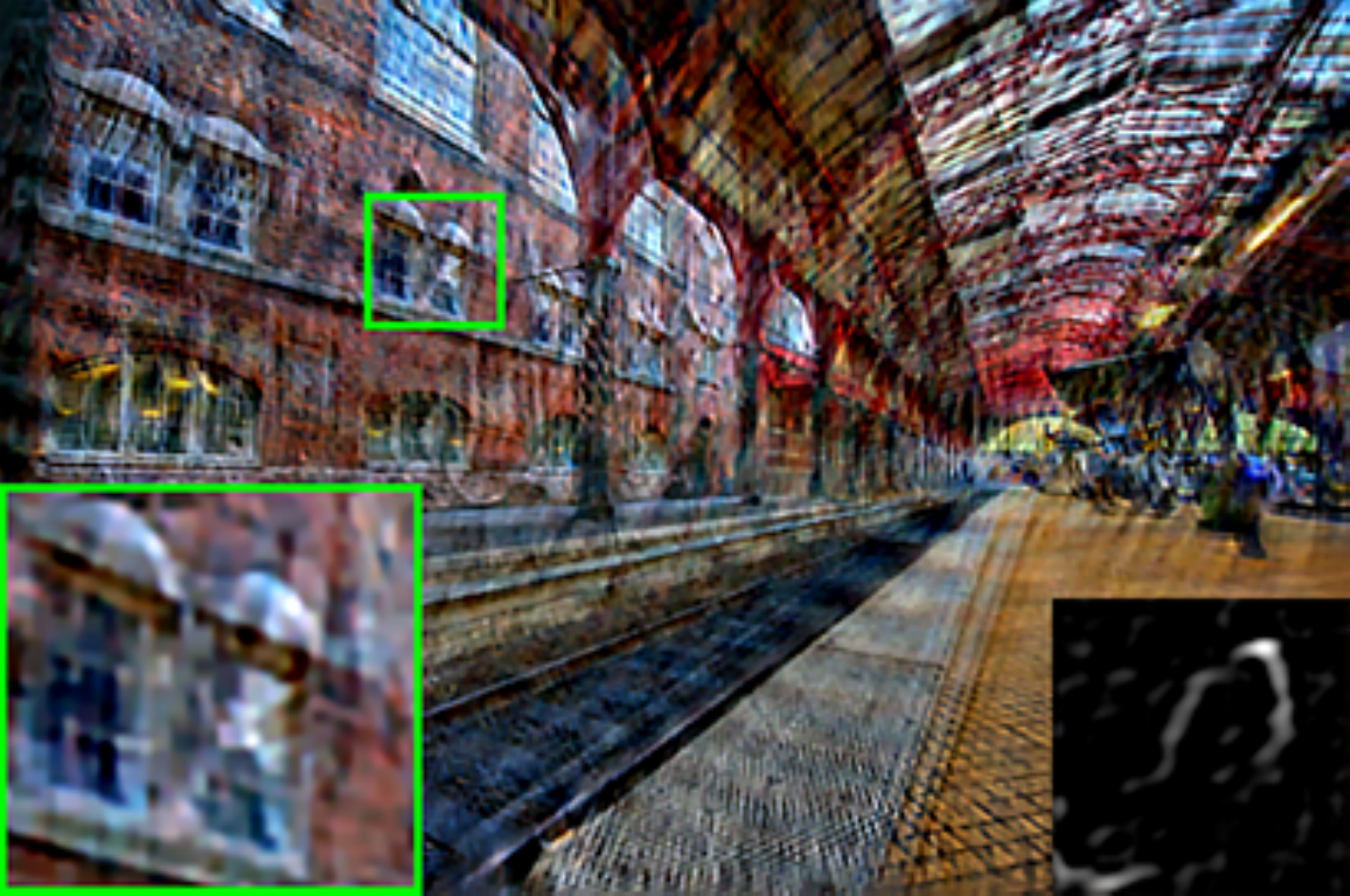}
		&\includegraphics[width=0.18\textwidth]{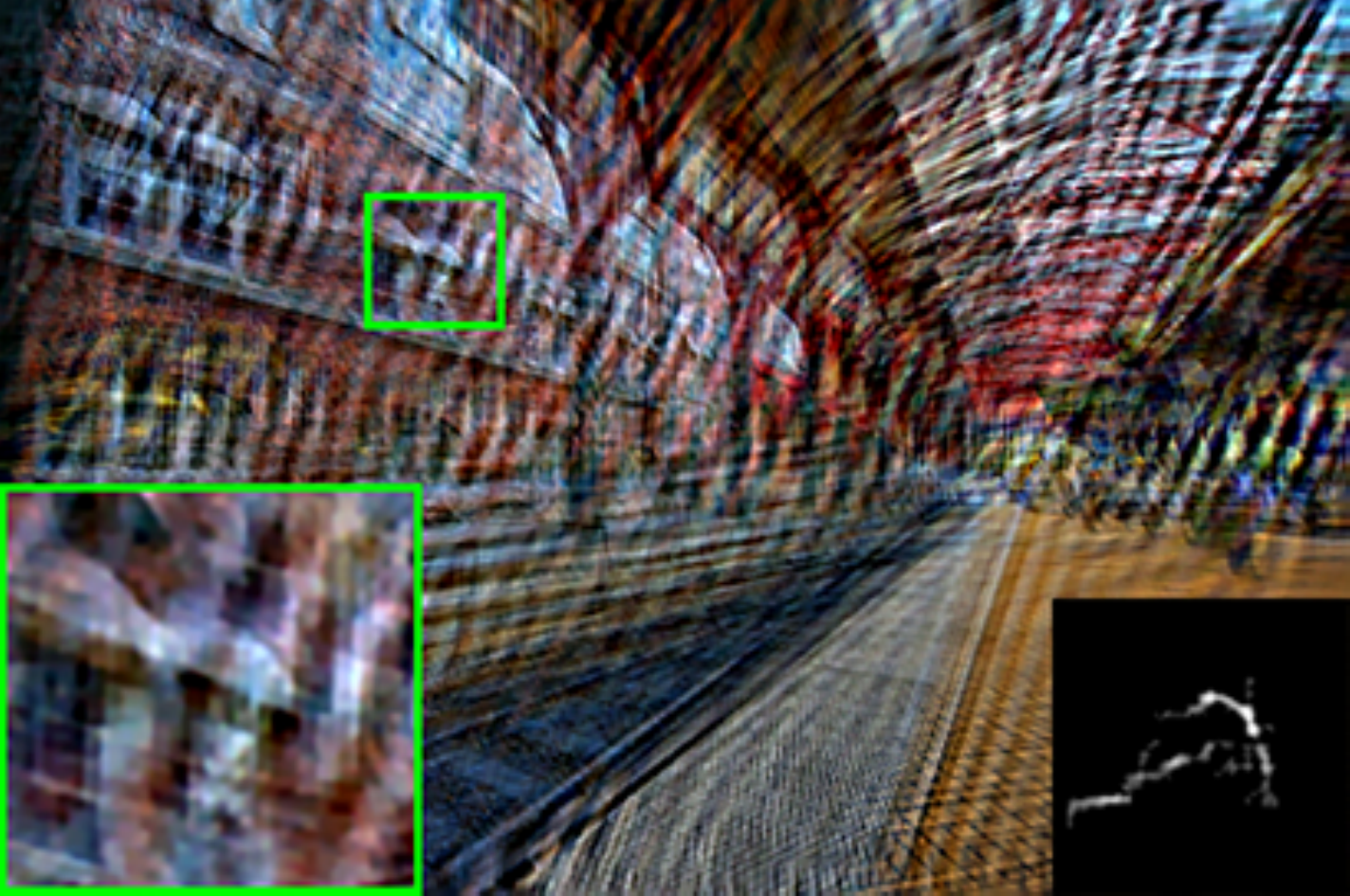}			
		&\includegraphics[width=0.18\textwidth]{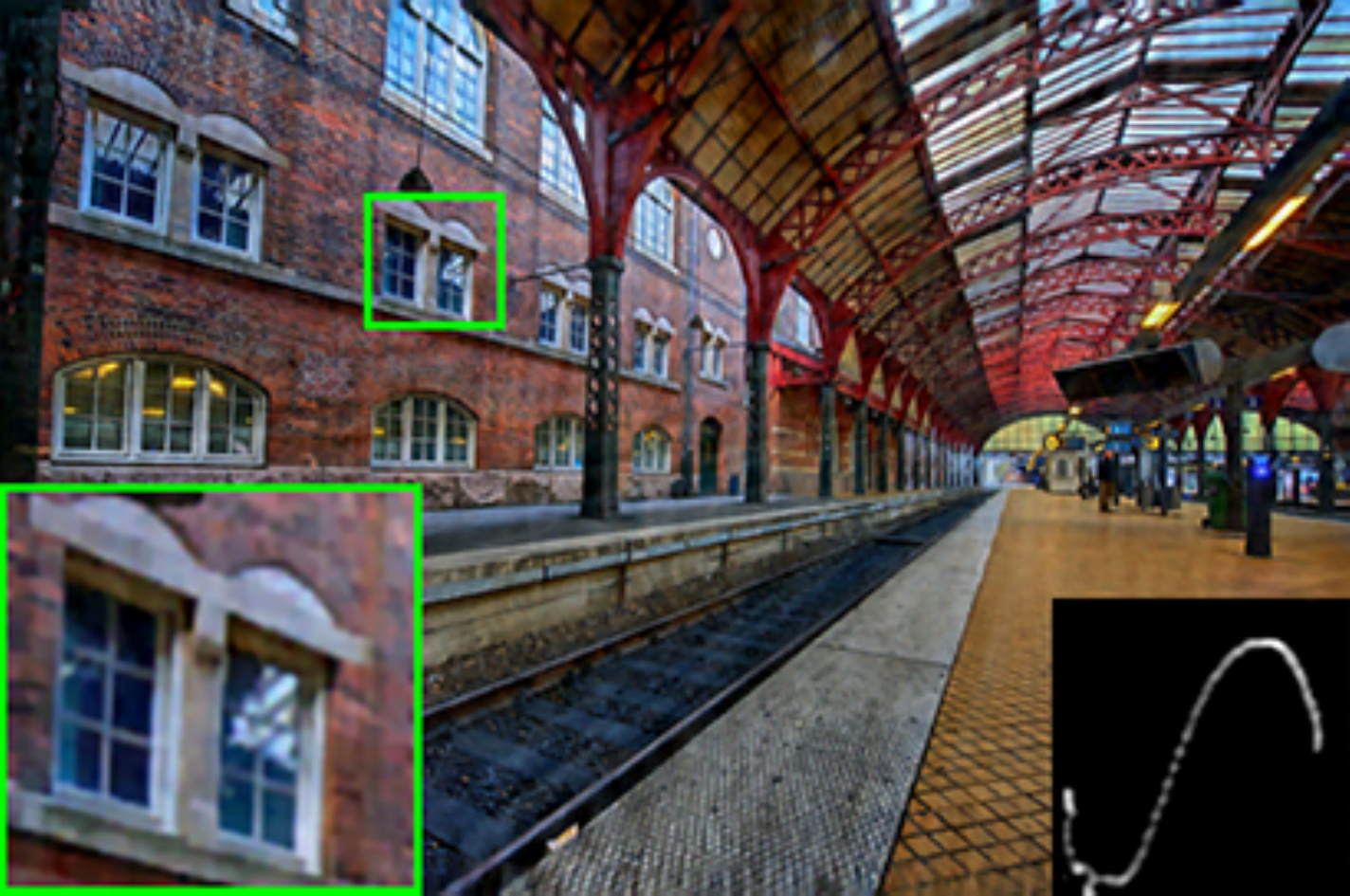}\\
		PSNR~/~SSIM &17.72~/~0.42 &15.80~/~0.27&14.17~/~0.15&\textbf{19.18~/~0.49}\\
		\includegraphics[width=0.18\textwidth]{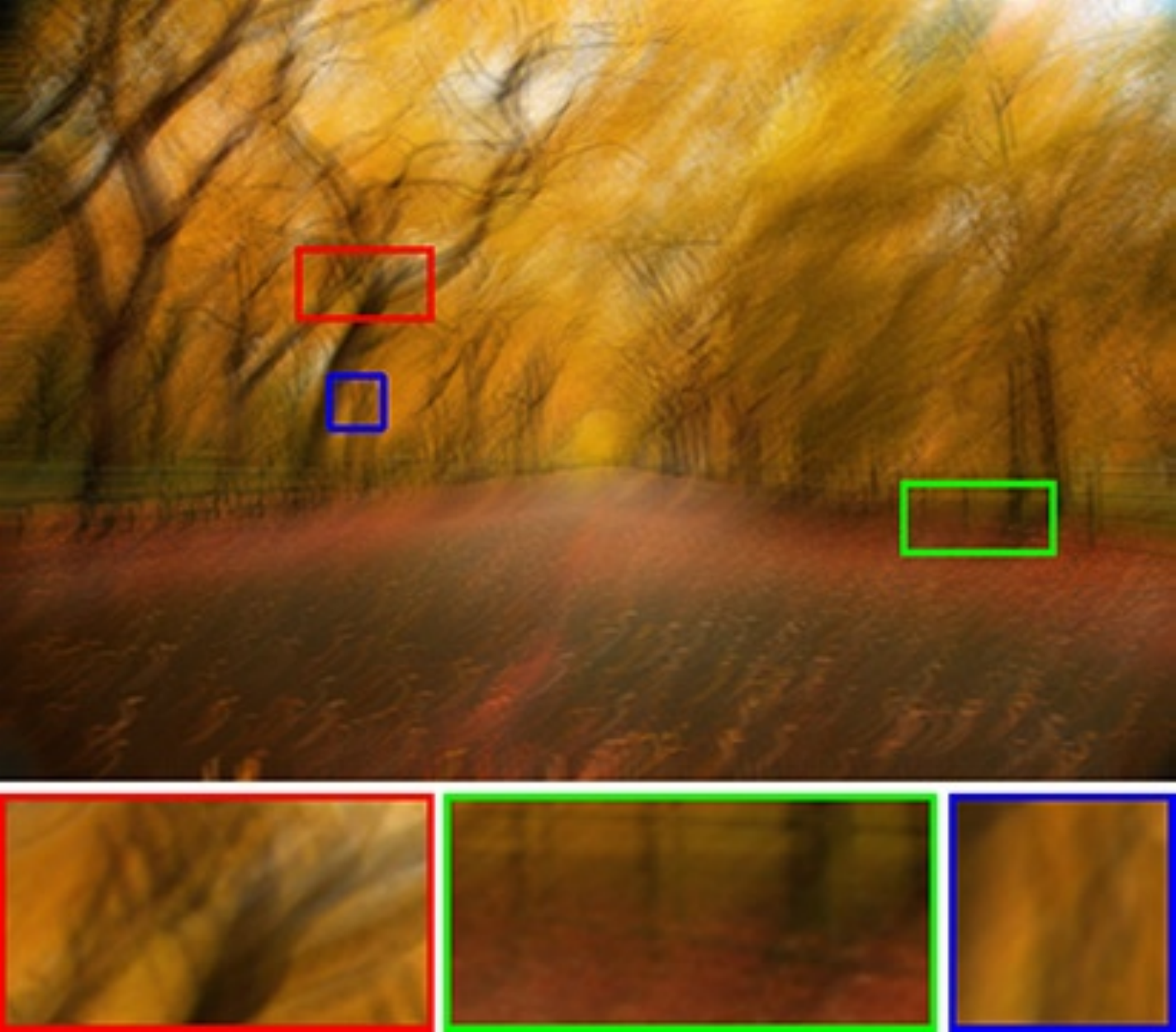}
		&\includegraphics[width=0.18\textwidth]{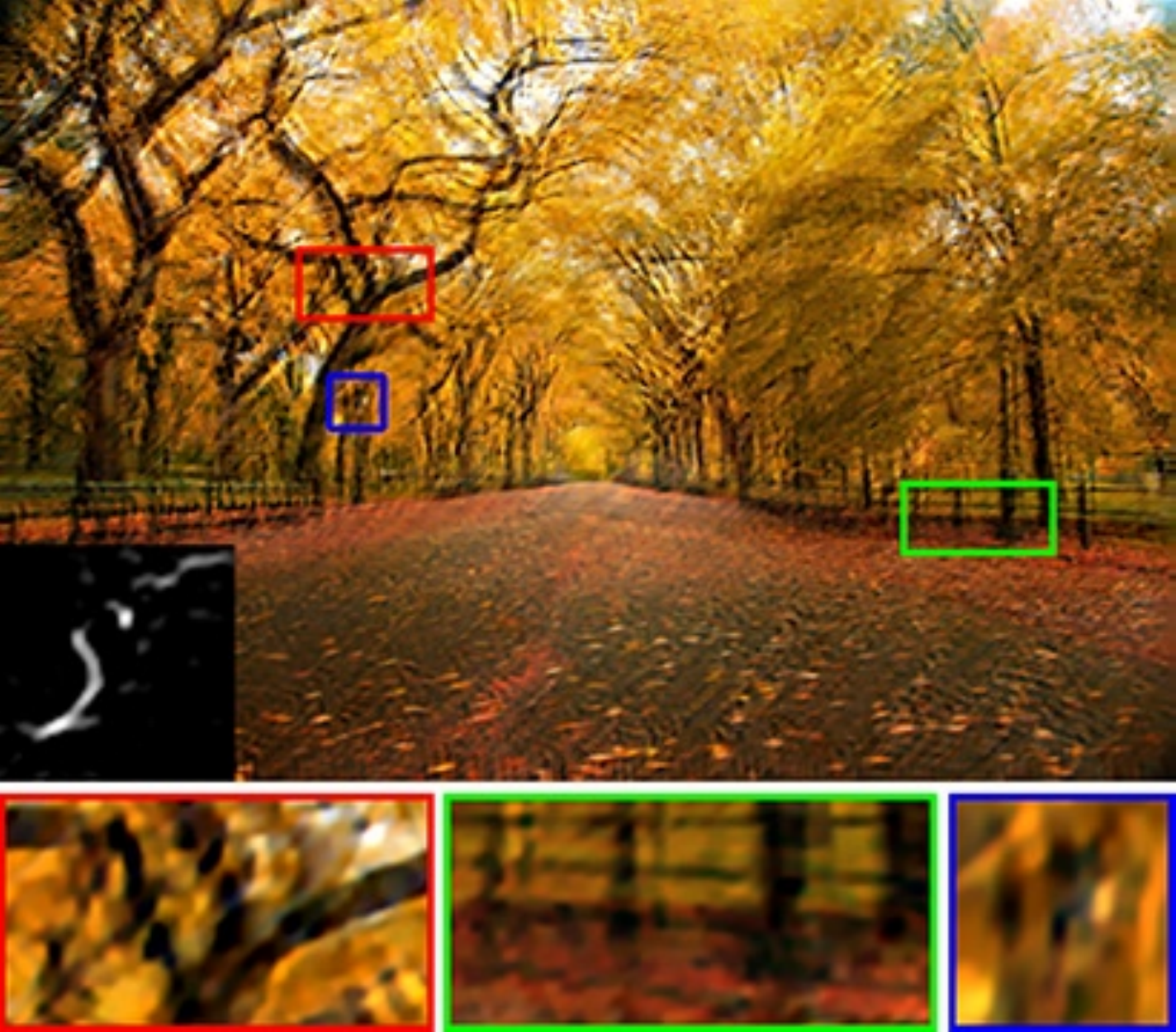}
		&\includegraphics[width=0.18\textwidth]{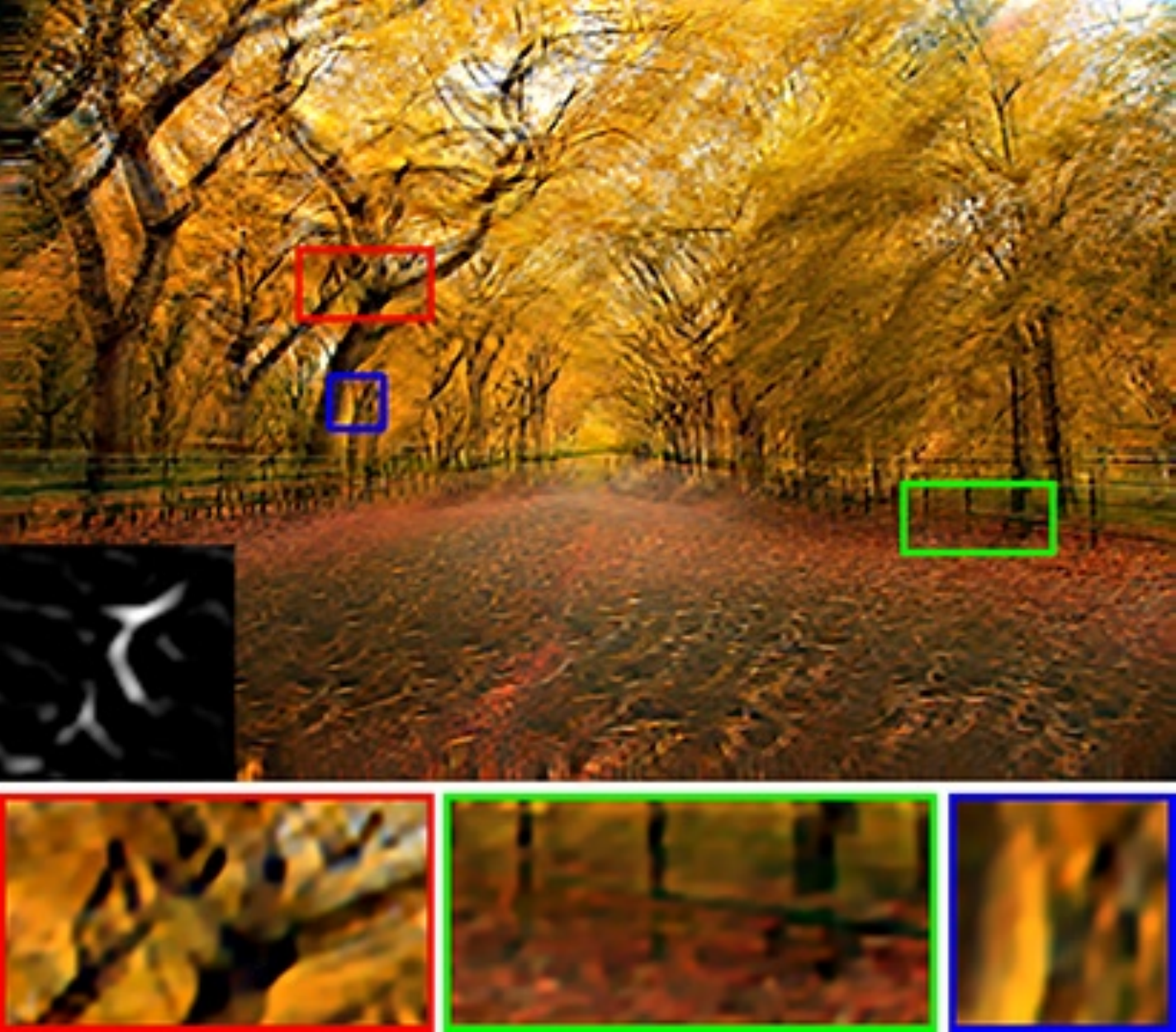}
		&\includegraphics[width=0.18\textwidth]{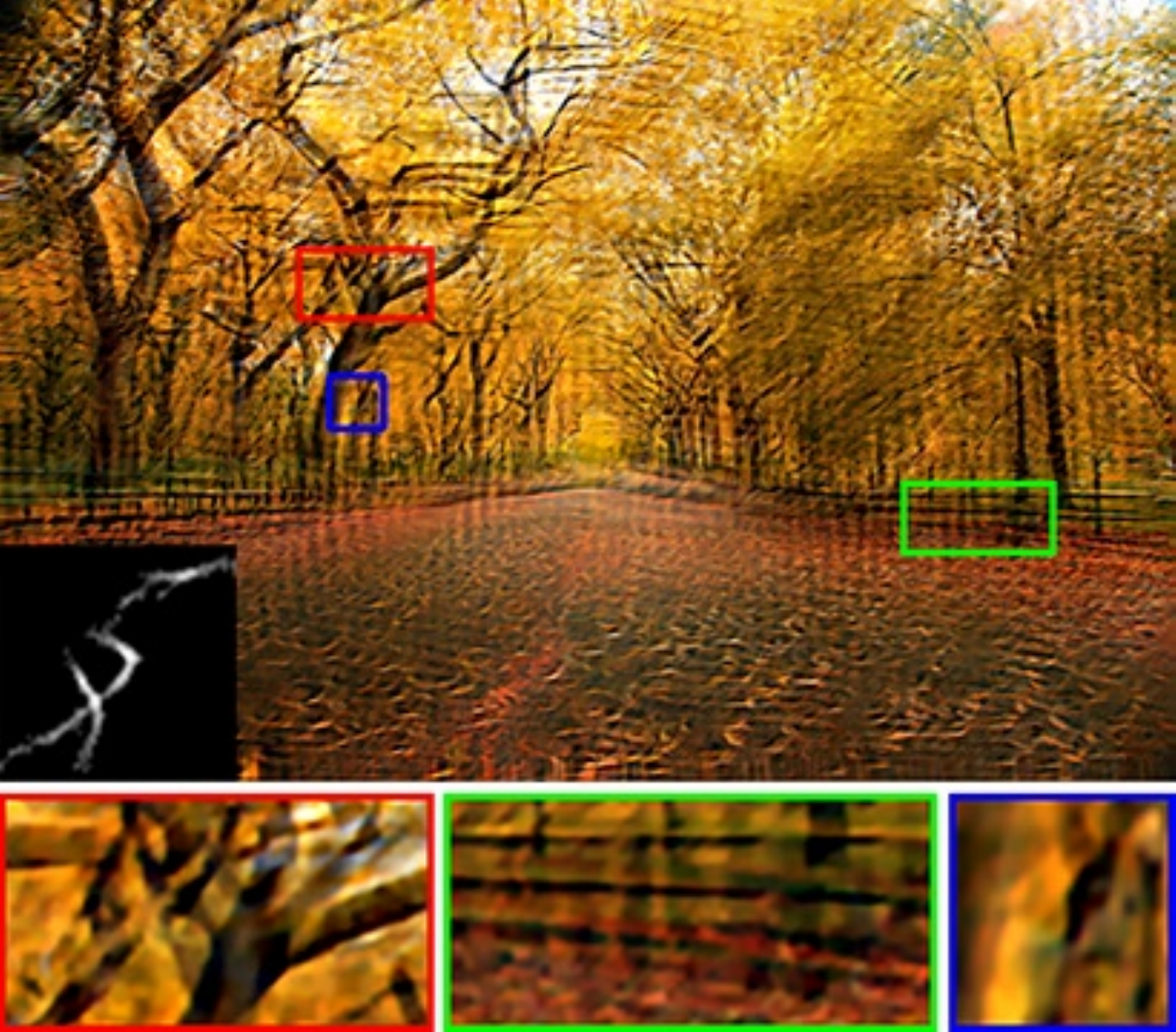}	
		&\includegraphics[width=0.18\textwidth]{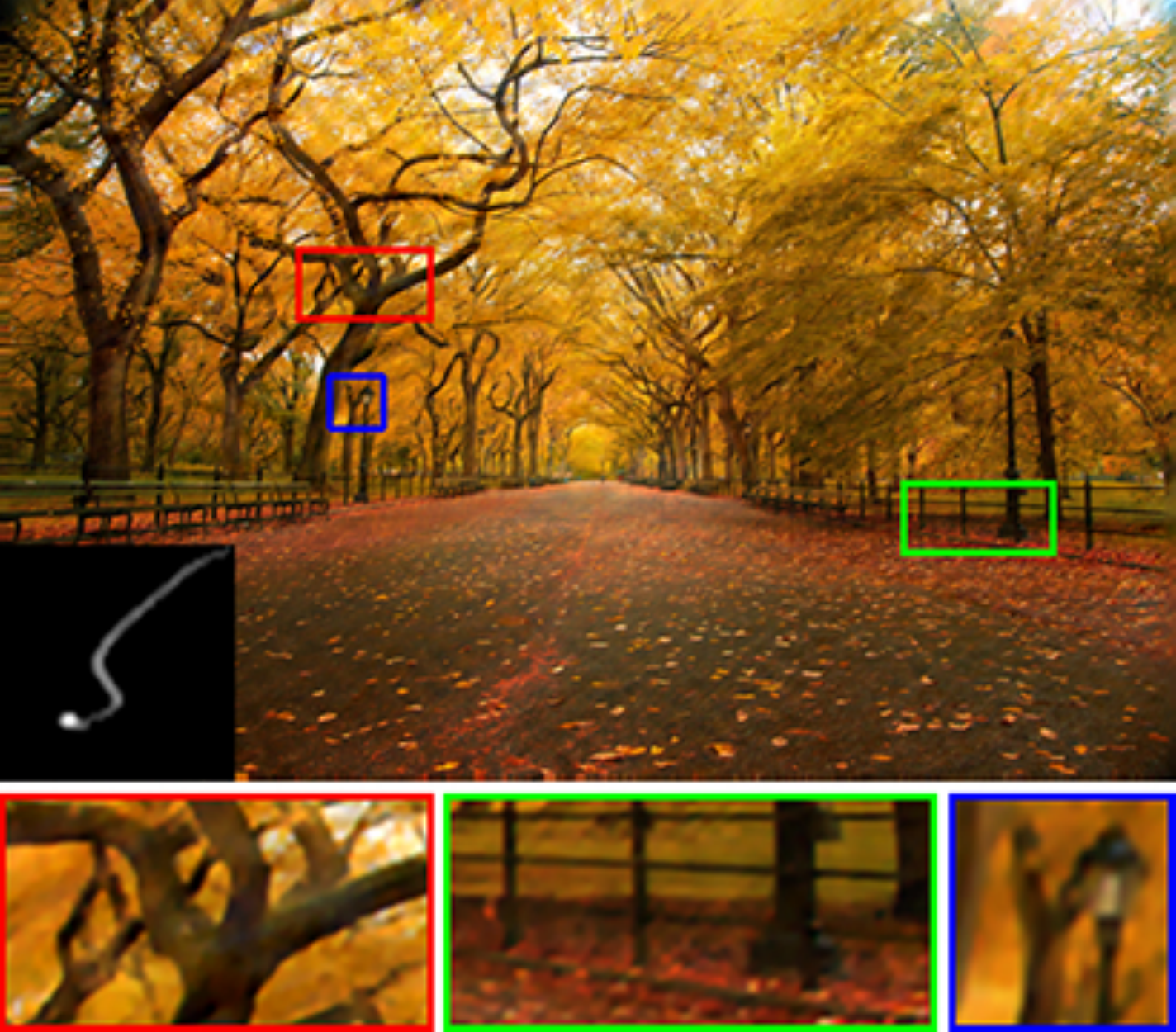}\\
		PSNR~/~SSIM  & 18.36~/~0.35 & 17.12~/~0.34 & 16.38~/~0.33 &\textbf{21.82~/~0.61}\\
		(a) Input & (b) Levin \emph{et al.}~\cite{levin2011efficient} & (c) Perrone \emph{et al.}~\cite{perrone2014total}& (d) Sun \emph{et al.}~\cite{sun2013edge} & (e) Ours \\
	\end{tabular}
	\caption{Comparisons on synthetic blurry images. Quantitative metrics are reported below each result. 
	}
	\label{fig:synthetic_image}
\end{figure*}

\begin{figure*}[htb]
	\centering \begin{tabular}{c@{\extracolsep{0.2em}}c@{\extracolsep{0.2em}}c@{\extracolsep{0.2em}}c@{\extracolsep{0.2em}}c}
		\includegraphics[width=0.18\textwidth]{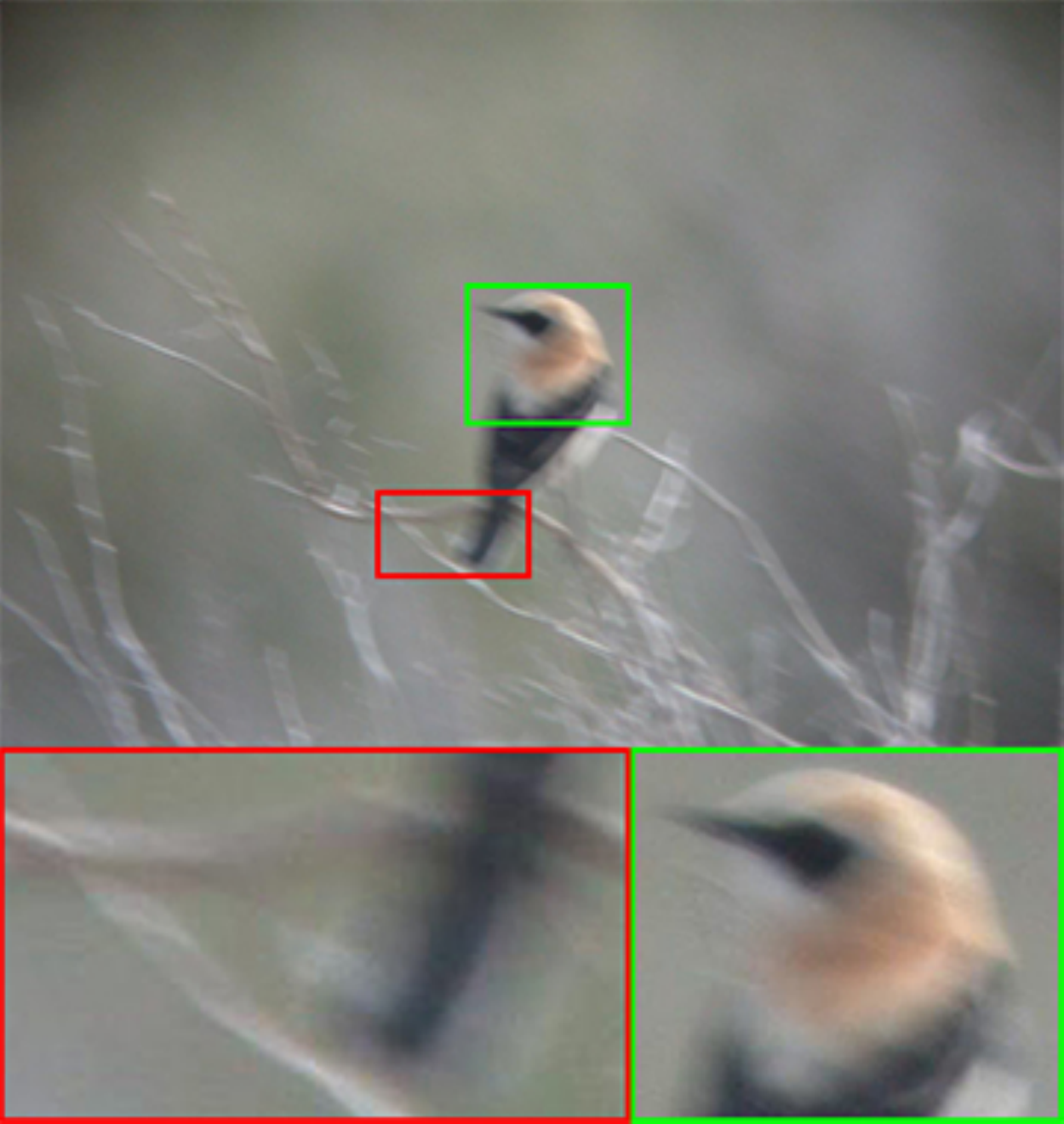}
		&\includegraphics[width=0.18\textwidth]{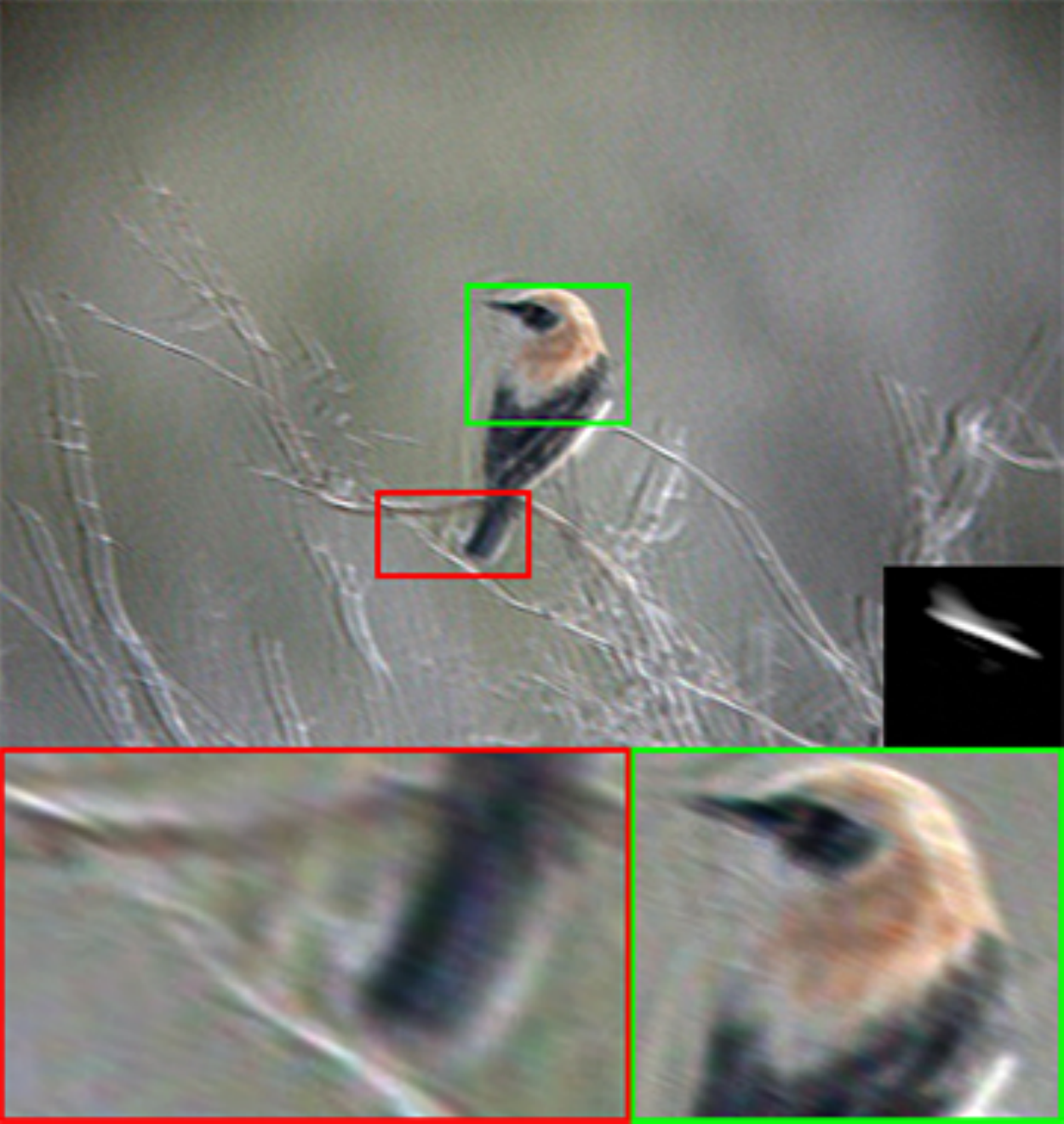}
		&\includegraphics[width=0.18\textwidth]{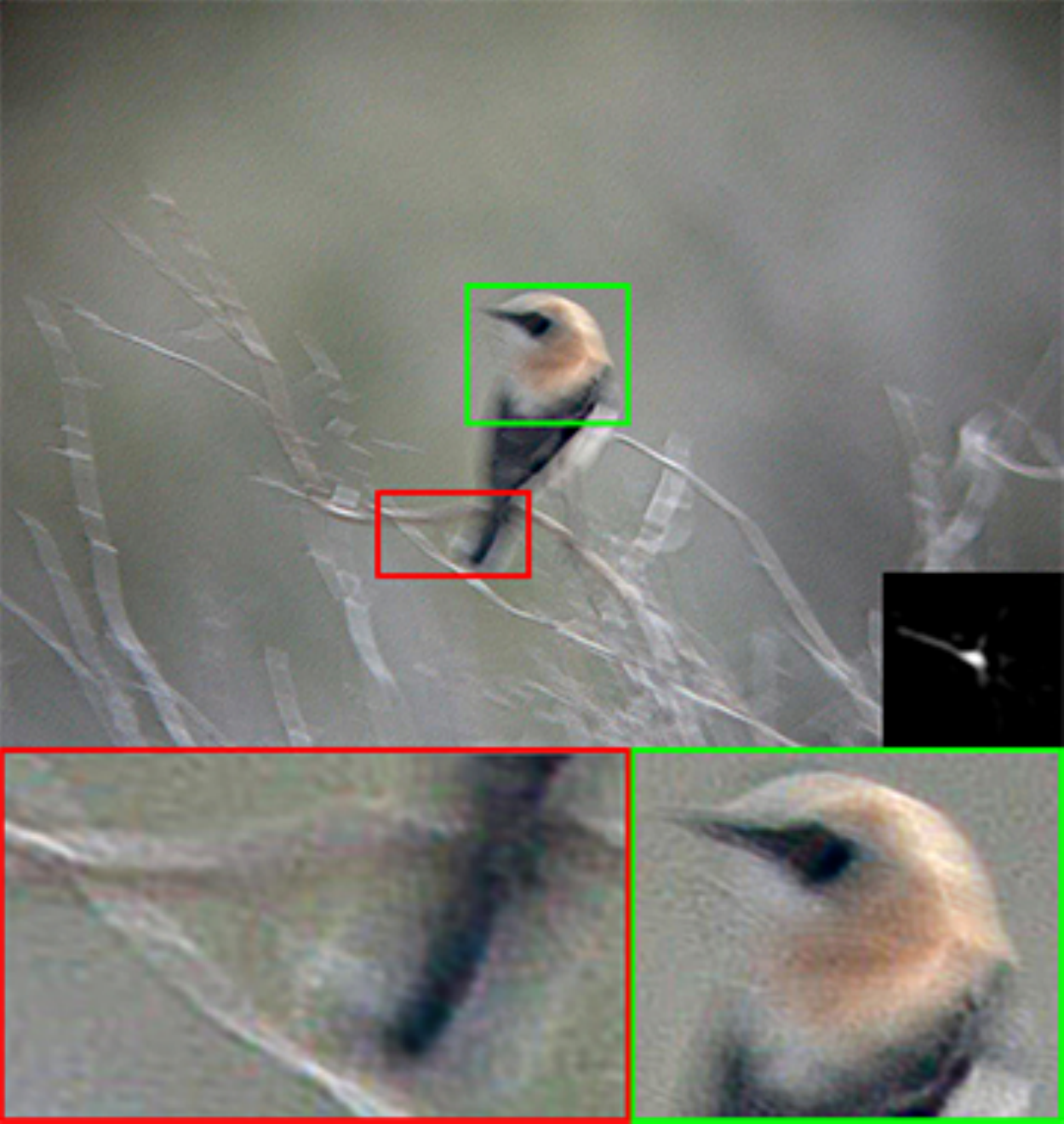}
		&\includegraphics[width=0.18\textwidth]{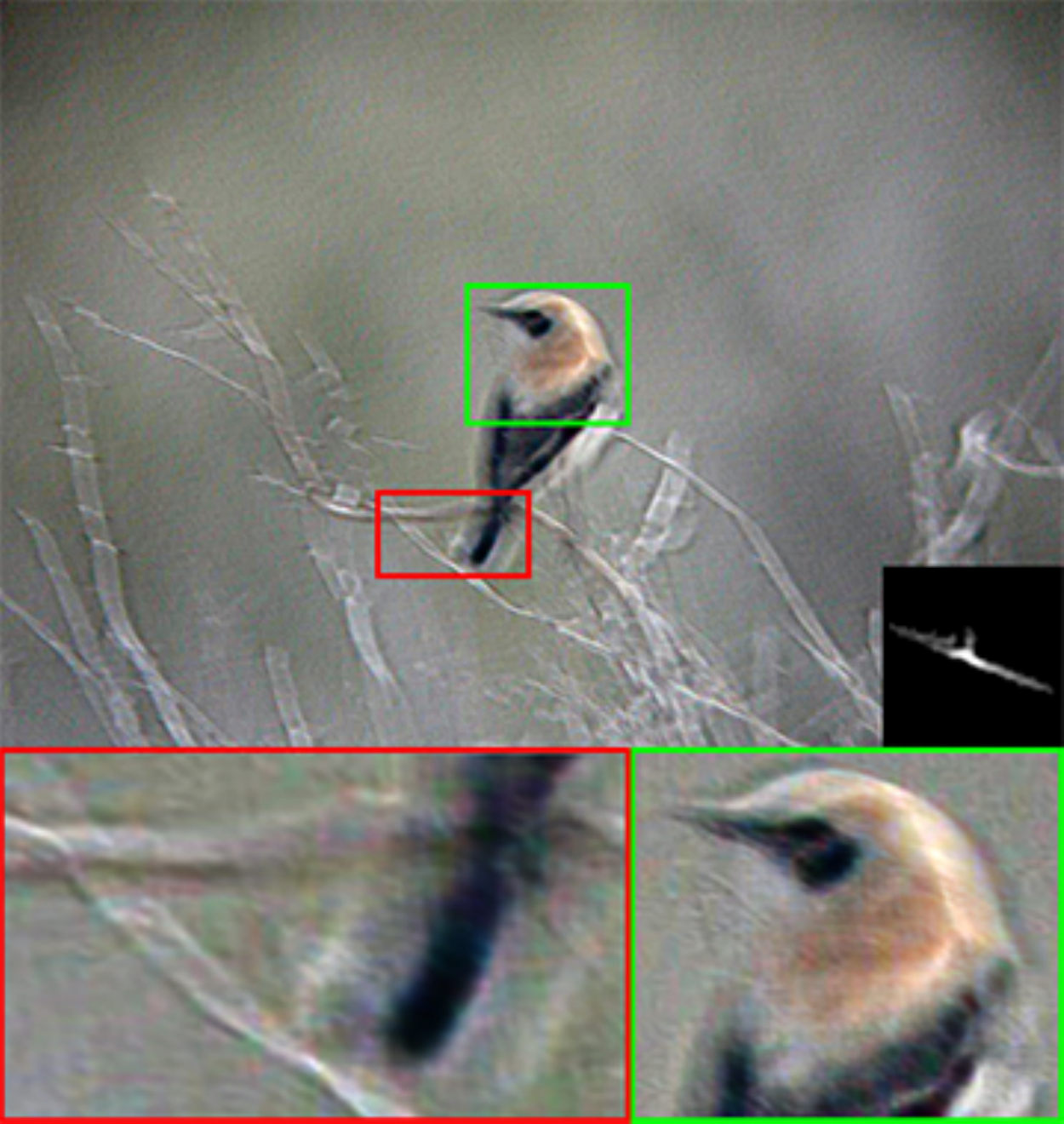}
		&\includegraphics[width=0.18\textwidth]{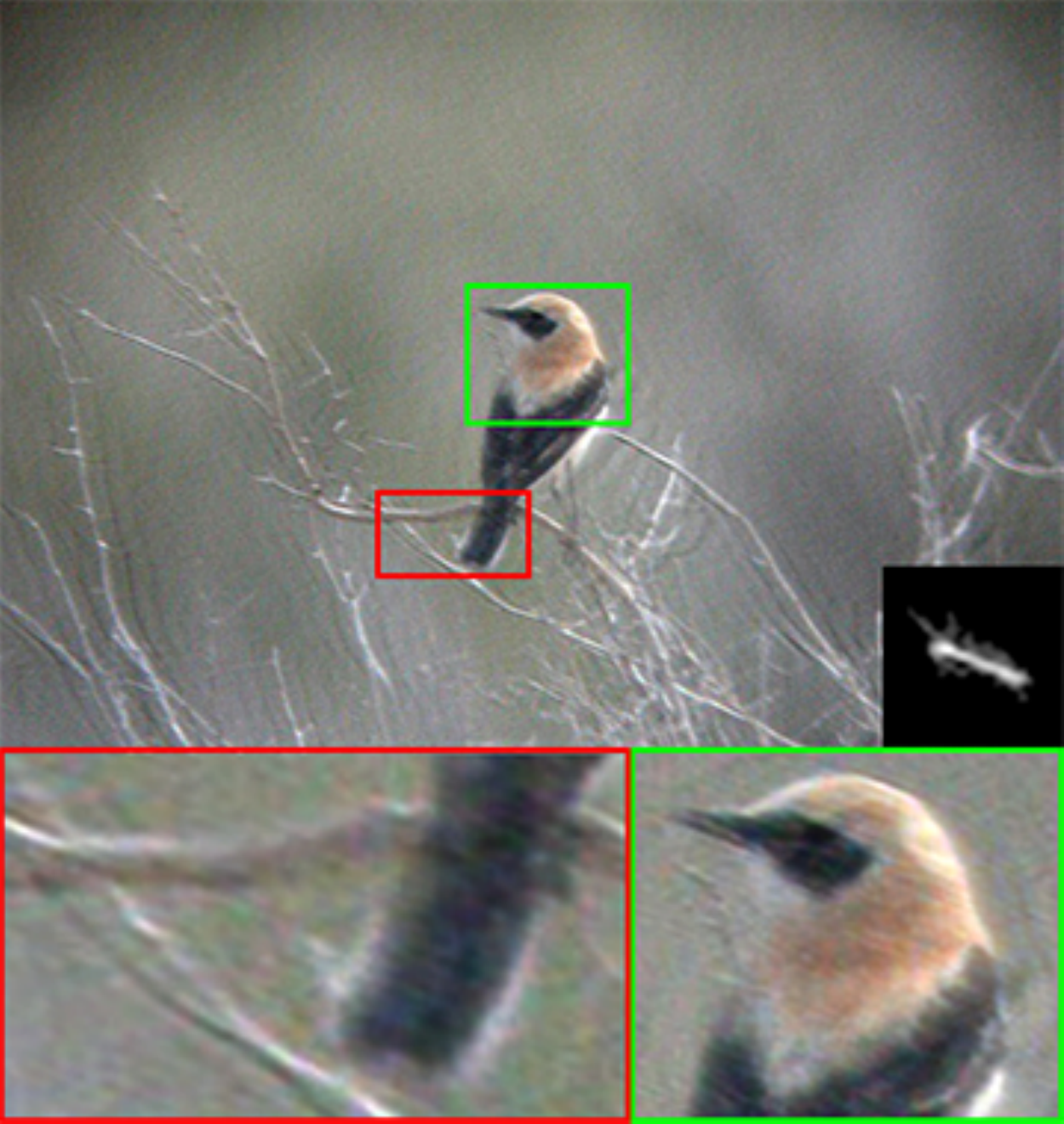}\\
		\includegraphics[width=0.18\textwidth]{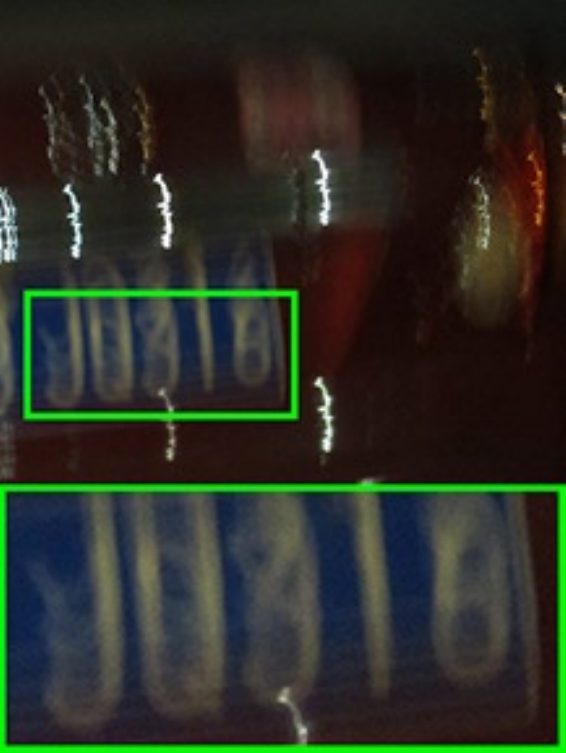}
		&\includegraphics[width=0.18\textwidth]{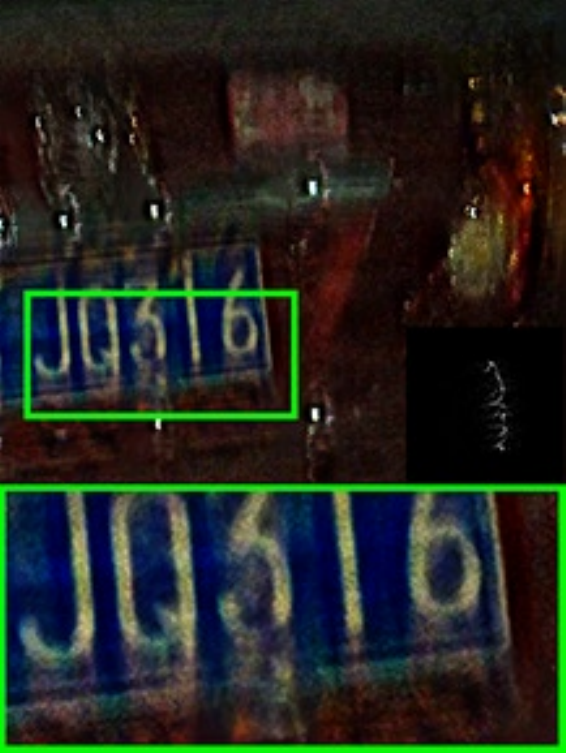}
		&\includegraphics[width=0.18\textwidth]{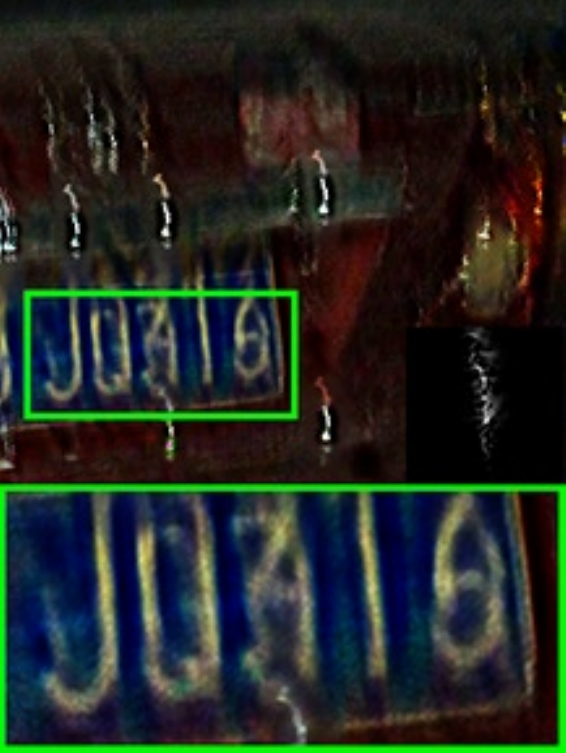}
		&\includegraphics[width=0.18\textwidth]{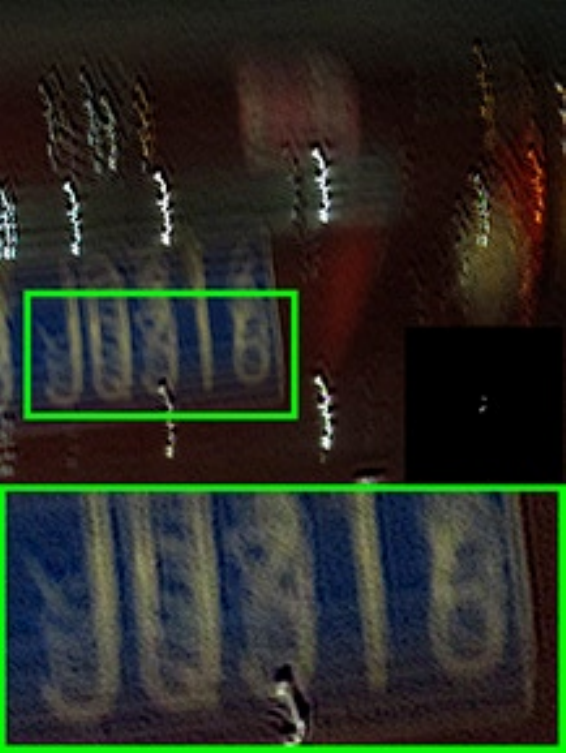}		
		&\includegraphics[width=0.18\textwidth]{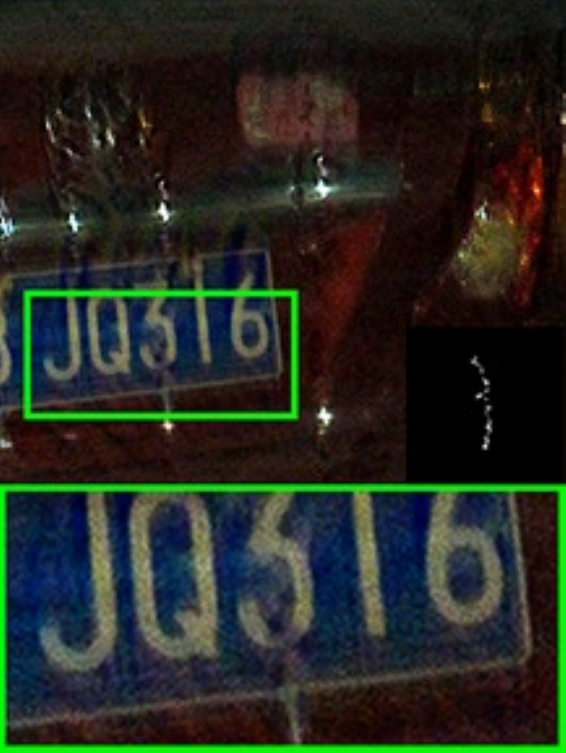}\\ 
		(a) Input & (b) Levin \emph{et al.}~\cite{levin2011efficient}& (c) Perrone \emph{et al.}~\cite{perrone2014total}& (d) Sun \emph{et al.}~\cite{sun2013edge} & (e) Ours \\
	\end{tabular}
	\caption{Visual comparisons on challenging real-world blurry images.}
	\label{fig:real-world_image}
\end{figure*}
\subsection{Blind Image Deblurring}

We then evaluate GCM on blind image deblurring problem.

\subsubsection{Synthetic Dataset}
We first consider synthetic test data and compare GCM with several state-of-the-art blind deblurring methods (e.g., \cite{krishnan2011blind,zhang2013multi,levin2011efficient,perrone2014total,sun2013edge}) on widely used Levin \emph{et al.}'s dataset~\cite{levin2009understanding}, including 32 blurry images with size $255\times255$, which are produced by 4 clear images and 8 blur kernels. 
Quantitative scores (e.g., PSNR, SSIM, Error Ratio (ER) and the average time) are reported in Table.~\ref{tab:data_set_result}. It is easy to see that our GCM achieves the best quantitative performance among all the compared methods. 

In Fig.~\ref{fig:synthetic_image}, we plot the visual results of our method together with the top 3 compared approaches (in Table~\ref{tab:data_set_result}) on two more challenging blurry examples (collated by Lai \emph{et al.}~\cite{lai2016comparative}). Notice that the sizes of these images and the blur kernels are much larger than that in Levin \emph{et al.}'s dataset. We observe that even performed well in Levin \emph{et al.}'s dataset, Sun \emph{et al.}'s method cannot obtain good deblurring results on this experiment. This is mainly because that the patch prior in that method is sensitive to image contents thus with less robustness~\cite{lai2016comparative}. In contrast, our collaborative GCM can successfully extract the latent sharp structure, so that achieves the best quantitative and qualitative results.

\begin{table}
	\caption{Quantitative results on Levin \emph{et al.}'s image set (lower ER is better).}
	\begin{tabular}{c@{\extracolsep{1.6em}}c@{\extracolsep{1.6em}}c@{\extracolsep{1.6em}}c@{\extracolsep{1.6em}}c}
		\toprule
		Method          &~~PSNR~~&~~SSIM~~&~~ER~~&~~TIME~~	  \\
		\midrule
		\cite{krishnan2011blind}     & 24.87  &  0.74  &  2.05  & 23.78 \\
		\cite{zhang2013multi}        & 28.01  &  0.86  &  1.25  & 37.45 \\
		\cite{levin2011efficient}    & 29.03  &  0.89  &  1.40  & 41.77\\
		\cite{perrone2014total}      & 29.27  &  0.88  &  1.35  & 113.70\\
		\cite{sun2013edge}           & 29.71  &  0.90  &  1.32  & 209.47 \\
		Ours            &\textbf{30.36}  &\textbf{0.91} &\textbf{1.21} & \textbf{6.17}\\
		\bottomrule
	\end{tabular}
	\label{tab:data_set_result}
\end{table}
\subsubsection{Real Blurry Images}
We also evaluate the compared methods on real-world blurry images (collected by \cite{lai2016comparative}) in Fig.~\ref{fig:real-world_image}. We can see on the top row that GCM can recover more details (e.g., the tail of bird with legible), compared with other methods. On the bottom row, it is also easy to observe that the numbers in license plate have been successfully recovered by GCM, while the visual quality of other results are bad for recognition. 

\begin{figure}[tb]
	\centering \begin{tabular}{c@{\extracolsep{0.2em}}c}
		\includegraphics[width=0.21\textwidth]{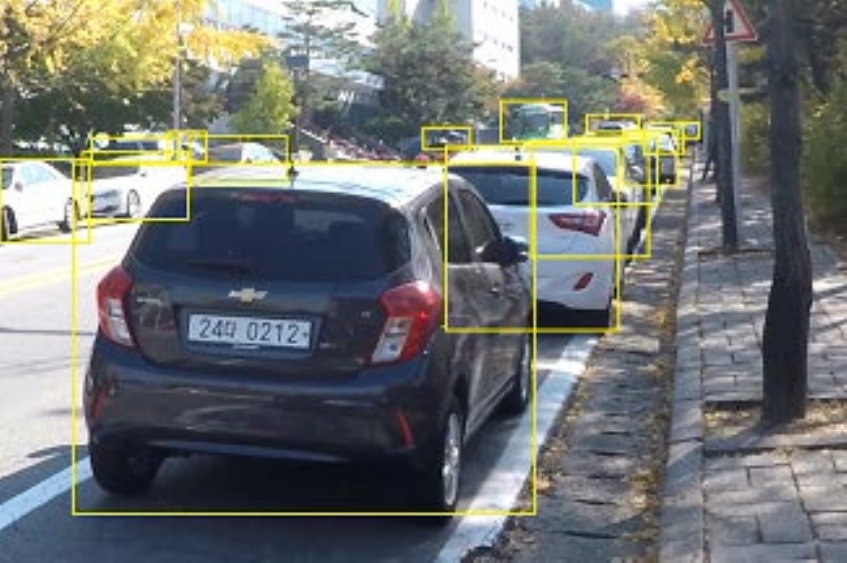}
		&\includegraphics[width=0.21\textwidth]{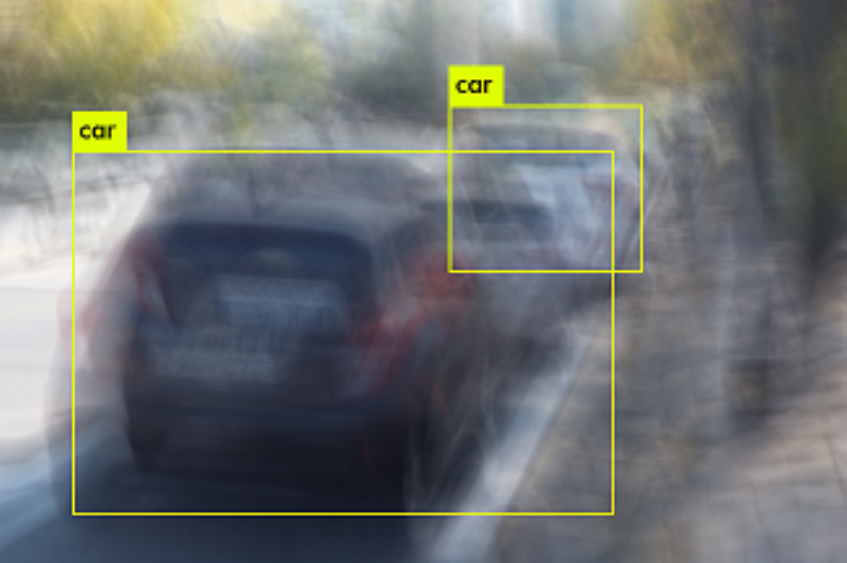}\\
		(a) Ground Truth &(b) Blurry \\
		\includegraphics[width=0.21\textwidth]{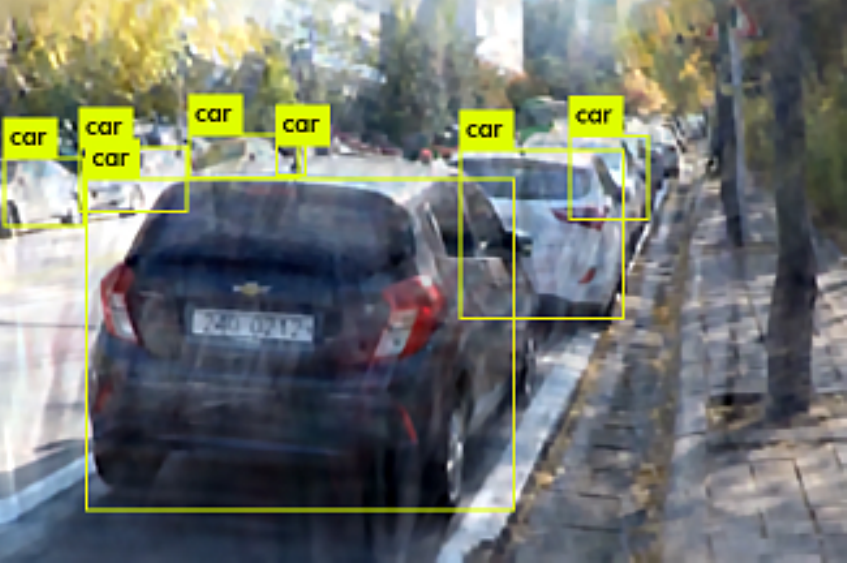}
		&\includegraphics[width=0.21\textwidth]{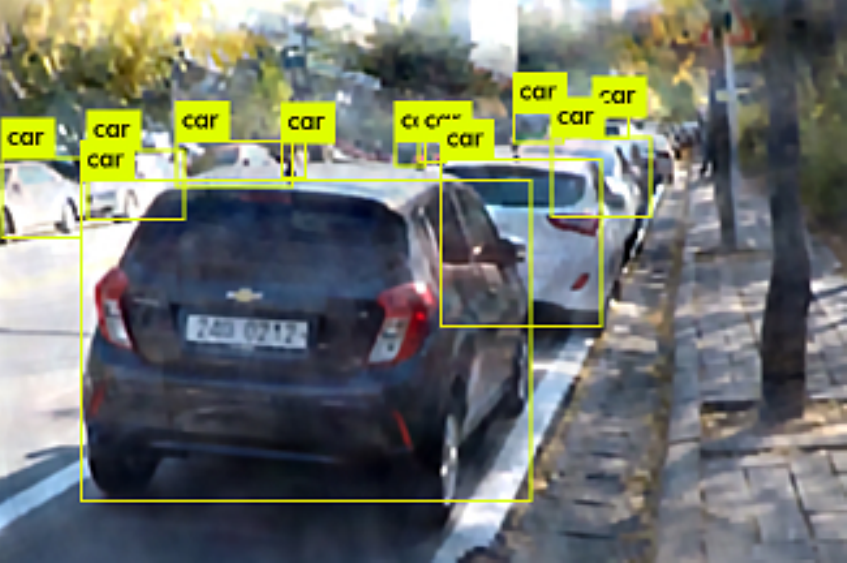}\\
		(c) Levin~\emph{et al.}~\cite{levin2011efficient}  &(d) Perrone~\emph{et al.}~\cite{perrone2014total} \\
		\includegraphics[width=0.21\textwidth]{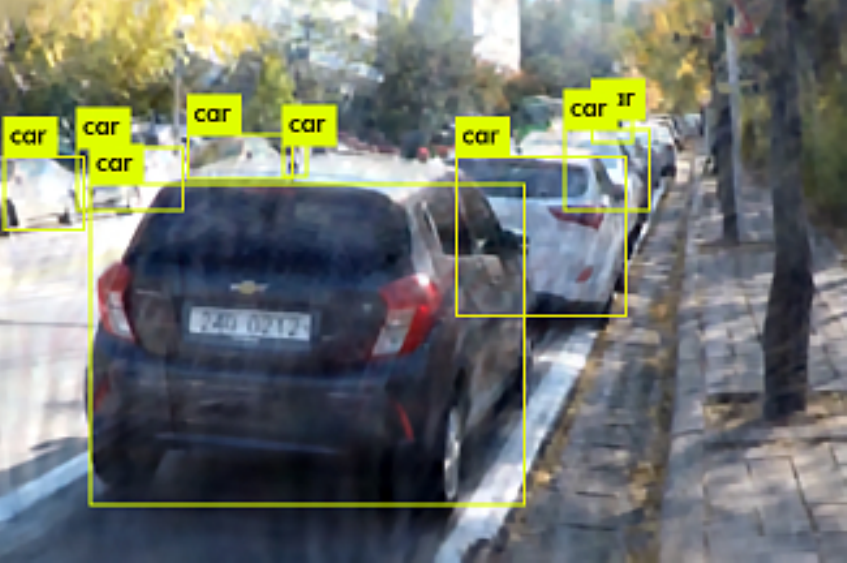}
		&\includegraphics[width=0.21\textwidth]{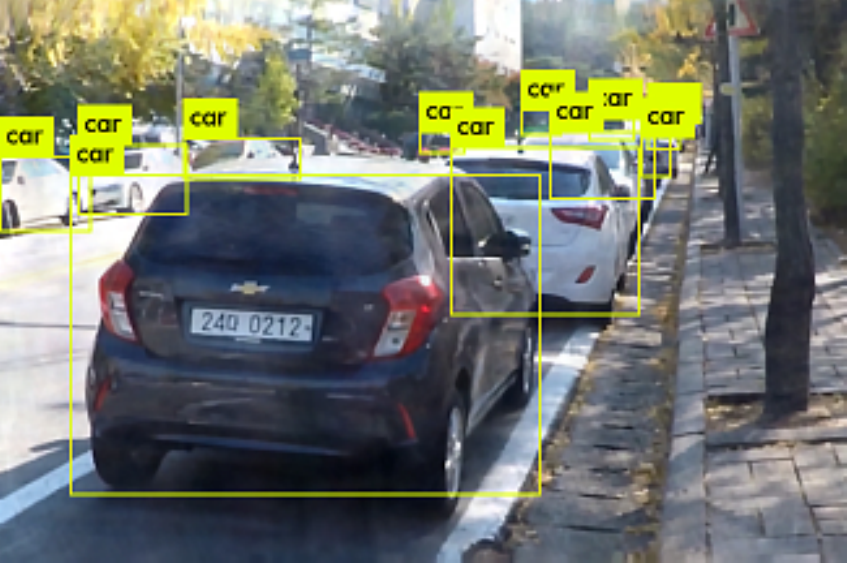}\\ 
		(e) Zhang~\emph{et al.}~\cite{zhang2013multi} & (f) Ours \\	
	\end{tabular}
	\caption{Car detection results on a street image. We illustrate the human labeled ground truth and the results detected on the blurred image in subfigures (a) and (b), respectively.  The detection results based on different deblurring algorithms are illustrated in subfigures (c)-(f). 
	}
	\label{fig:object detection}
\end{figure}

\subsubsection{Detection by Deblurring}
It is known that motion blurs caused by the shaking of capture device often reduce the performance of detection algorithm. See Fig.~\ref{fig:object detection} (b) for an example. Thus a natural strategy to evaluate the effectiveness of deblurring algorithms is just to perform object detection on the restored images. 
In this experiment, we adopt the well-known YOLO object detection system~\cite{redmon2016you} on the test image taken from the GoPro dataset~\cite{Nah2016Deep}. 

Specifically, we first take the manually labeled result on the clear image as ground truth, then treat the blurry image and various deblurring results as the inputs of YOLO. The detection results are visualized in Fig.~\ref{fig:object detection}.
We can observe that some overlapped and partially occluded small-size cars cannot be found in other methods. In contrast,
the most car objects had been detected in our method. This is because GCM can suppress most artifacts and recover more details in the results.
In Table.~\ref{tab:detection_result}, we report the number of detected cars (Detected Cars for short), recall values, precision values and F1 scores as quantitative metrics to measure the performance of these deblurring methods. We can see that our GCM obtains the highest scores in all metrics. 
Since the performance of detection is tightly related to the deblurring performance, this experiment actually indicates that our method has the ability to improve the performance of real-world tasks.

\begin{table}[tb]
	\centering
	\caption{The quantitative results for the car detection experiment in Fig.~\ref{fig:object detection}.}
	\begin{tabular}{c@{\extracolsep{1.5em}} c@{\extracolsep{1.5em}} c@{\extracolsep{1.5em}} c@{\extracolsep{1.5em}} c@{\extracolsep{1.5em}} c}
		\toprule
		Metric&~~Blurred~~&~~\cite{levin2011efficient}~~&~~\cite{perrone2014total}~~&~~\cite{zhang2013multi}~~&~~Ours~~	  \\
		\midrule 
		Detected Cars        & 2  & 6 & 8& 7& \textbf{11} \\   
		Recall        &  0.11  &  0.33 & 0.44 & 0.39 & \textbf{0.61} \\
		Precision         &  \textbf{1.0}   &  0.86 & 0.73 & 0.88 & \textbf{1.0} \\
		F1 Score      &  0.20  &  0.48 & 0.55 & 0.54 & \textbf{0.76} \\
		\bottomrule
	\end{tabular}
	\label{tab:detection_result}
\end{table} 

\subsection{Other Applications}
To verify the flexibility of our collaborative modules, we also express the performance of GCM on other applications, including image interpolation and edge-preserved smoothing. 
\subsubsection{Image Interpolation}
The purpose of image interpolation is to recover an image in which some pixels are lost or deteriorated. 
To evaluate the performance of our method in this task, we compare GCM with other state-of-the-art image interpolation methods~\cite{jin2015annihilating,roth2009fields,he2014iterative,getreuer2012total} on both text and random missing pixels masks. The test images are randomly chosen from ImageNet dataset~\cite{deng2009imagenet}.
As shown in Fig.~\ref{fig:inpaintting result}, our method achieves a better performance than all these compared methods on both visual effects and quantitative metrics (PSNR / SSIM). 
On the top row of Fig.~\ref{fig:inpaintting result}, one can see that ALOHA~\cite{jin2015annihilating} and FoE~\cite{roth2009fields} failed on this test image. ISDSB~\cite{he2014iterative} and TV~\cite{getreuer2012total} can recover most missing regions. But their results are blurred and there still exist some missing pixels. In contrast, our method fills all the missing regions and the result looks more realistic. On the second row of Fig.~\ref{fig:inpaintting result}, we can see that GCM actually obtains a much clearer image with richer details than other compared methods.

\begin{figure*}[htb]
	\centering \begin{tabular}{c@{\extracolsep{0.2em}}c@{\extracolsep{0.2em}}c@{\extracolsep{0.2em}}c@{\extracolsep{0.2em}}c@{\extracolsep{0.2em}}c}
		\includegraphics[width=0.15\textwidth]{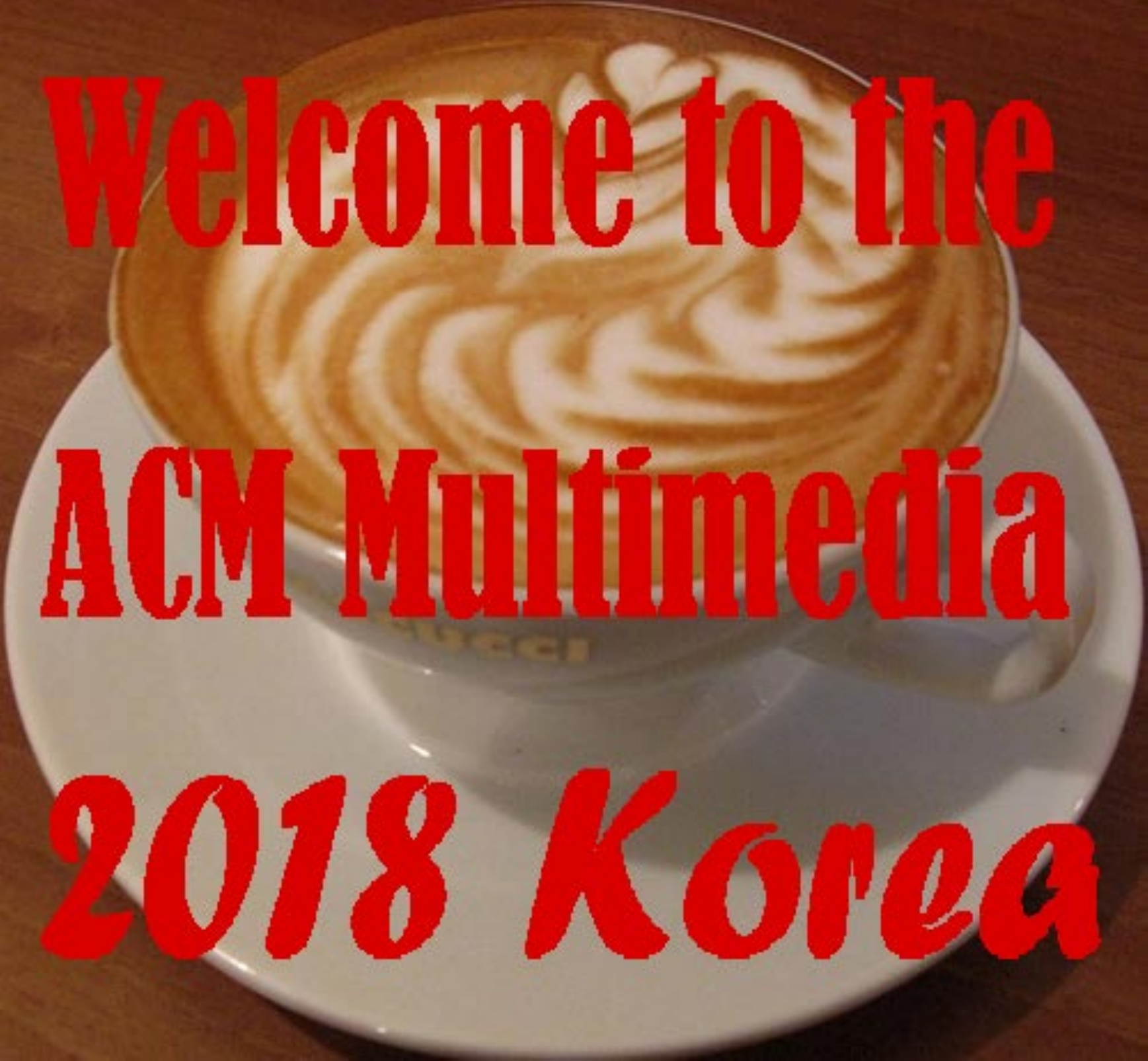}
		&\includegraphics[width=0.15\textwidth]{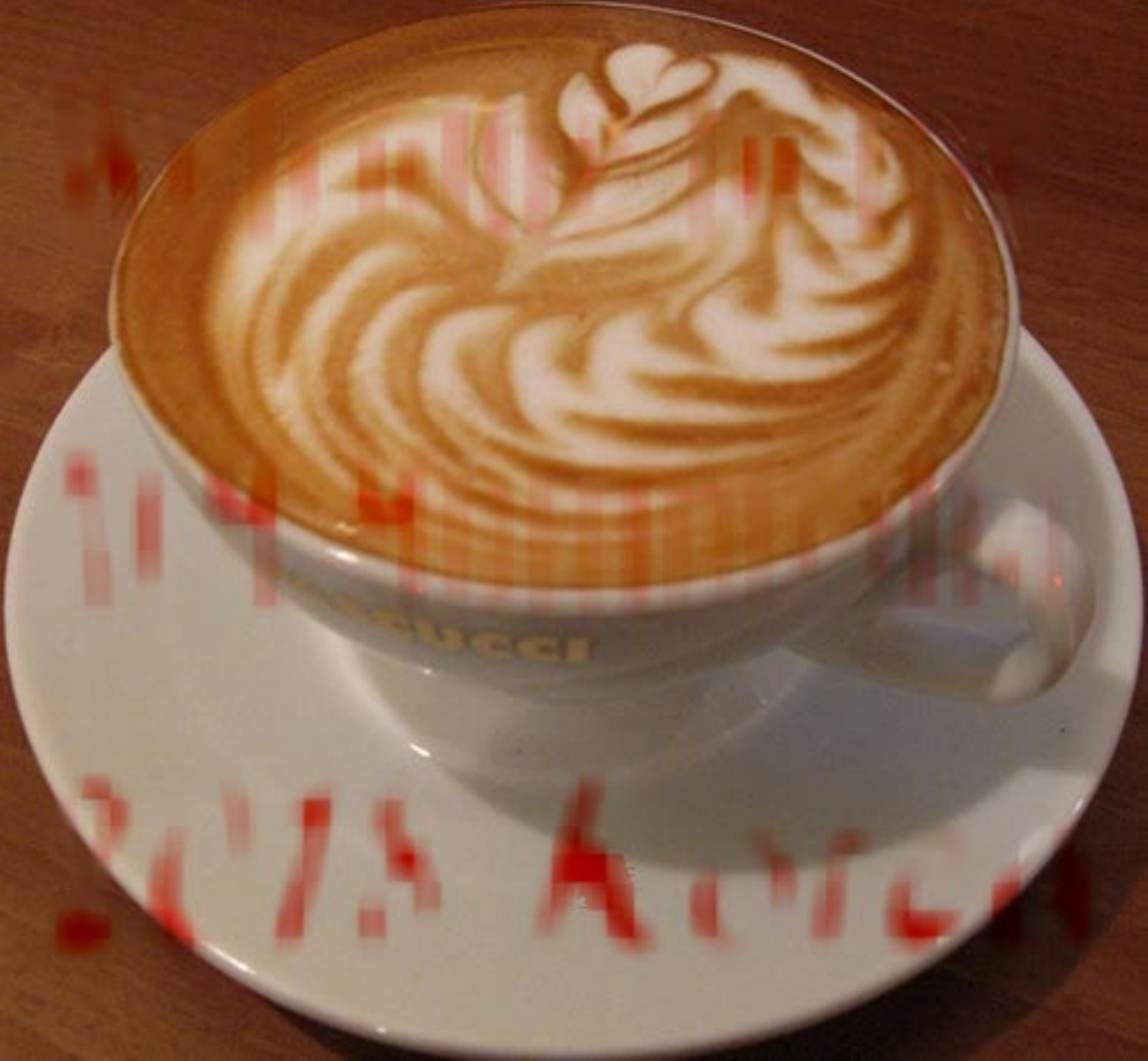}
		&\includegraphics[width=0.15\textwidth]{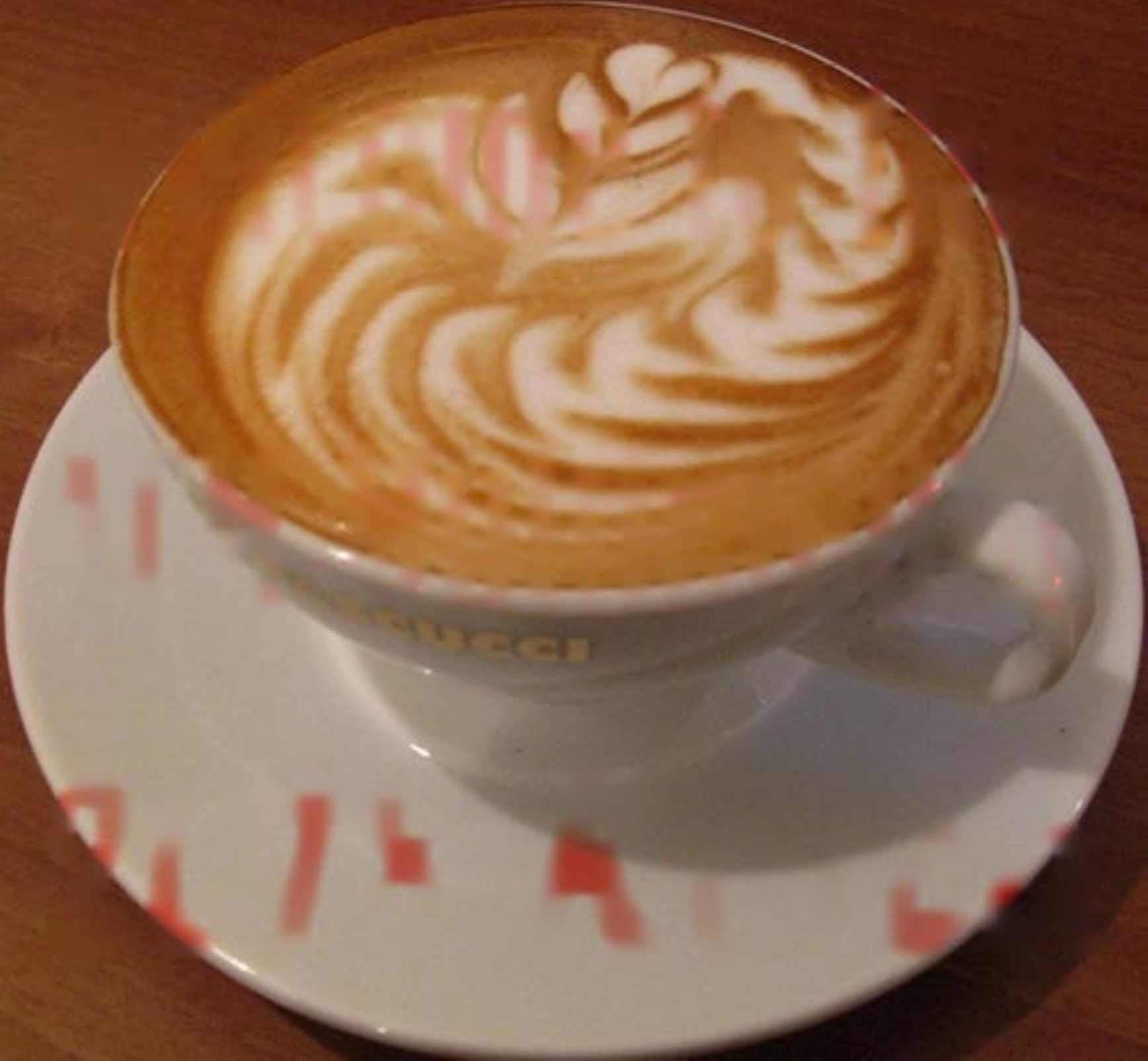}
		&\includegraphics[width=0.15\textwidth]{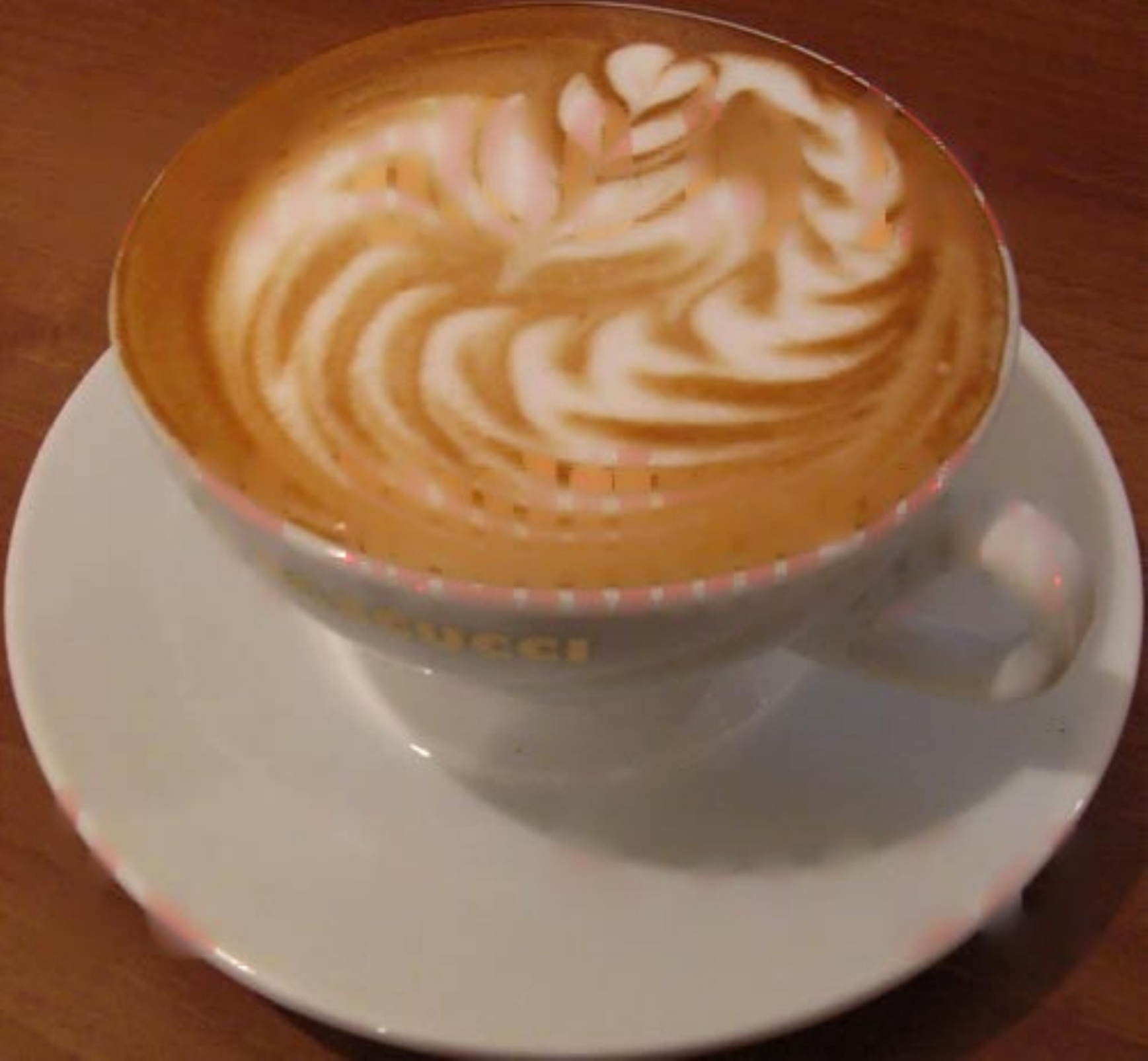}
		&\includegraphics[width=0.15\textwidth]{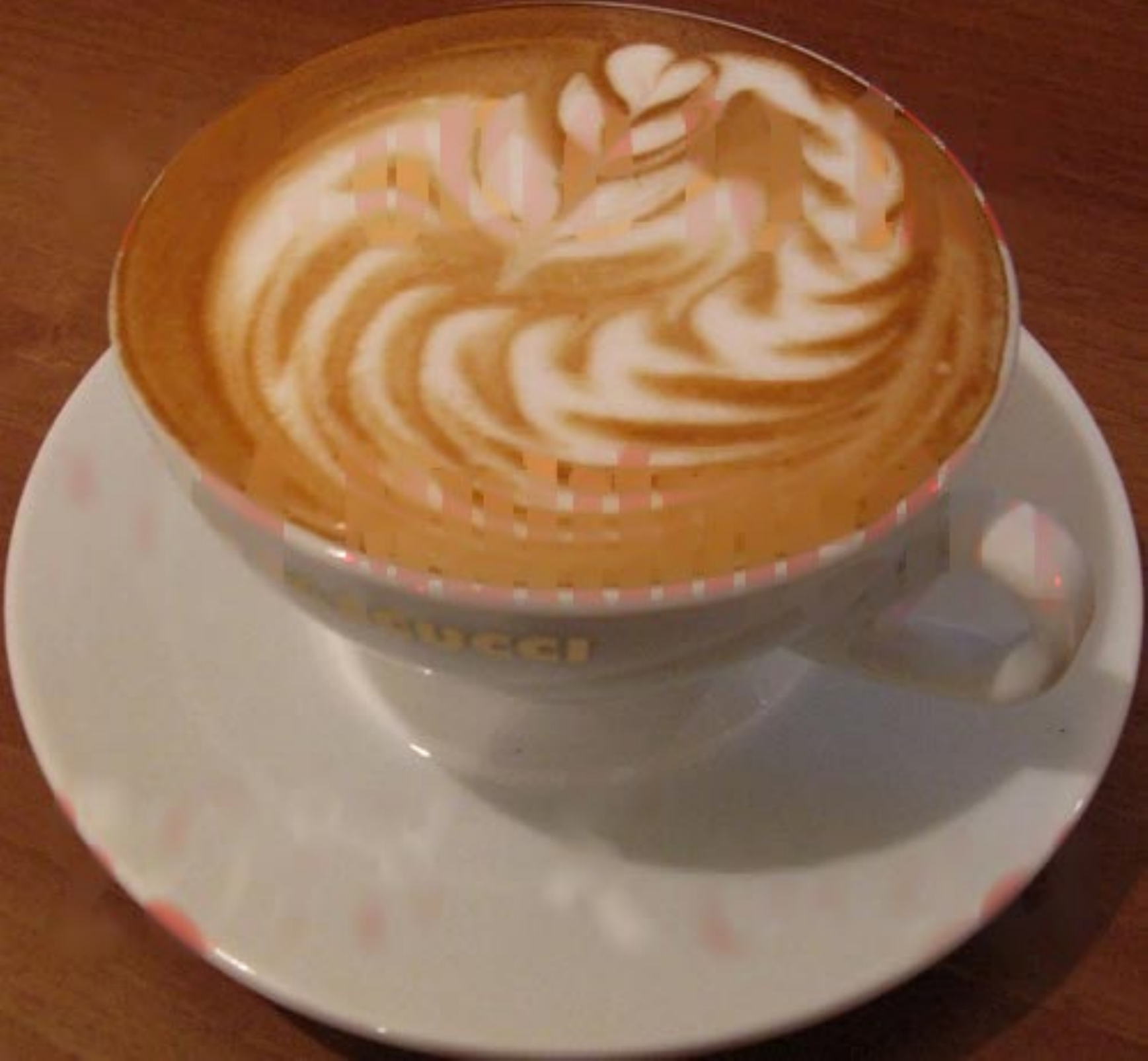}
		&\includegraphics[width=0.15\textwidth]{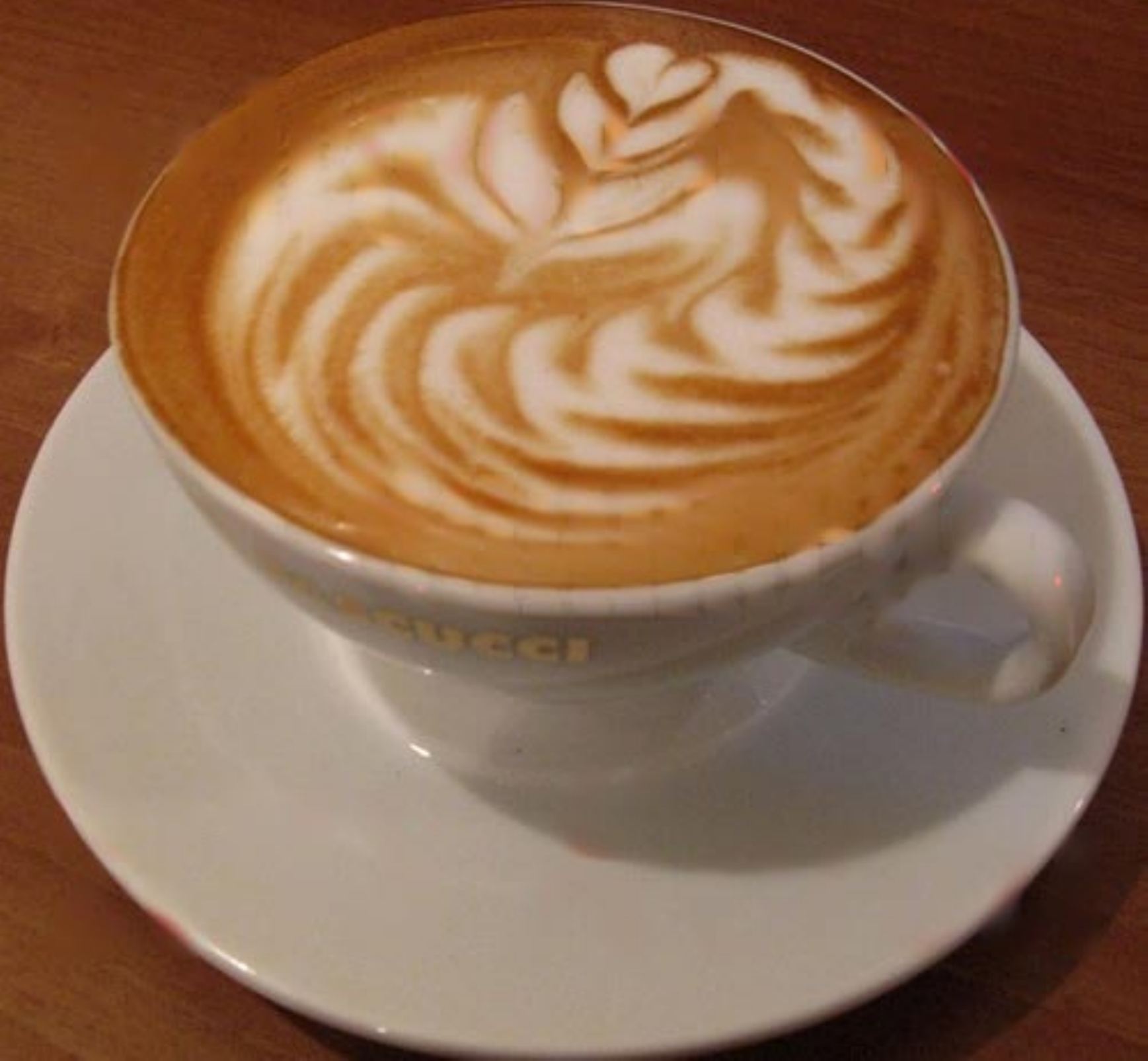}\\
		PSNR~/~SSIM & 24.93~/~0.9700 & 27.83~/~0.9825 & 31.26~/~0.9874 & 31.21~/~0.9884 & \textbf{35.81~/~0.9953} \\
		\includegraphics[width=0.15\textwidth]{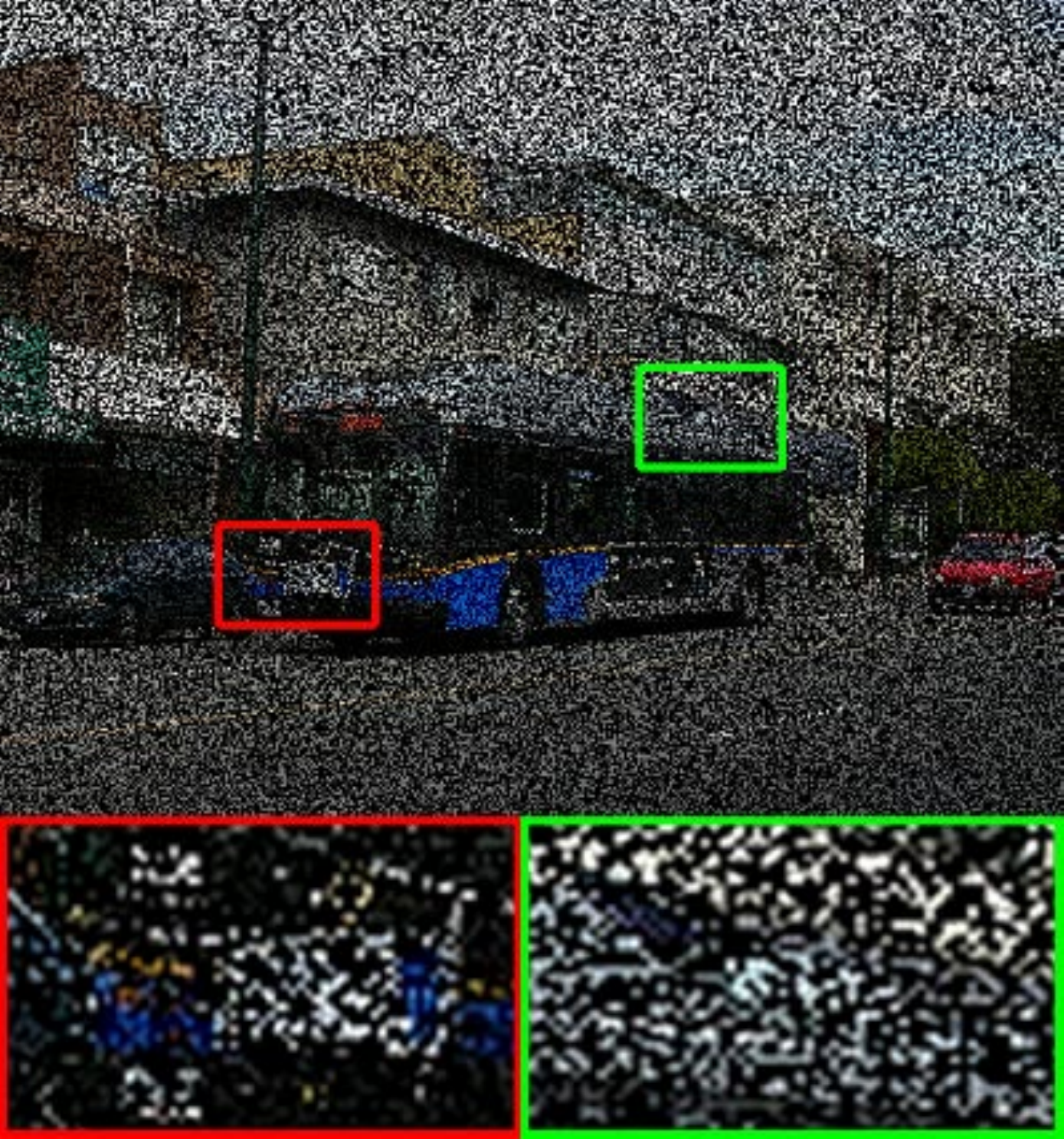}
		&\includegraphics[width=0.15\textwidth]{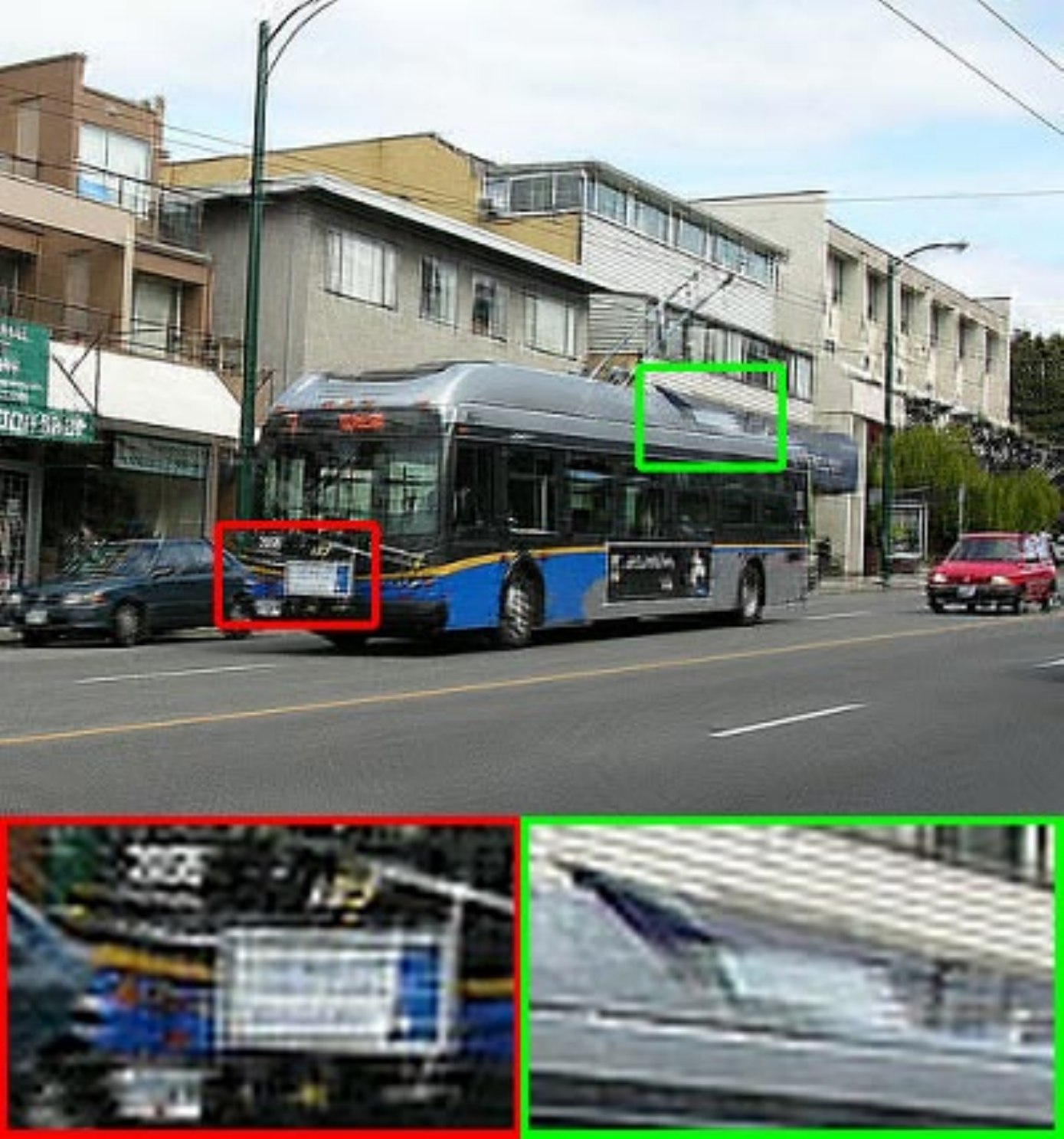}
		&\includegraphics[width=0.15\textwidth]{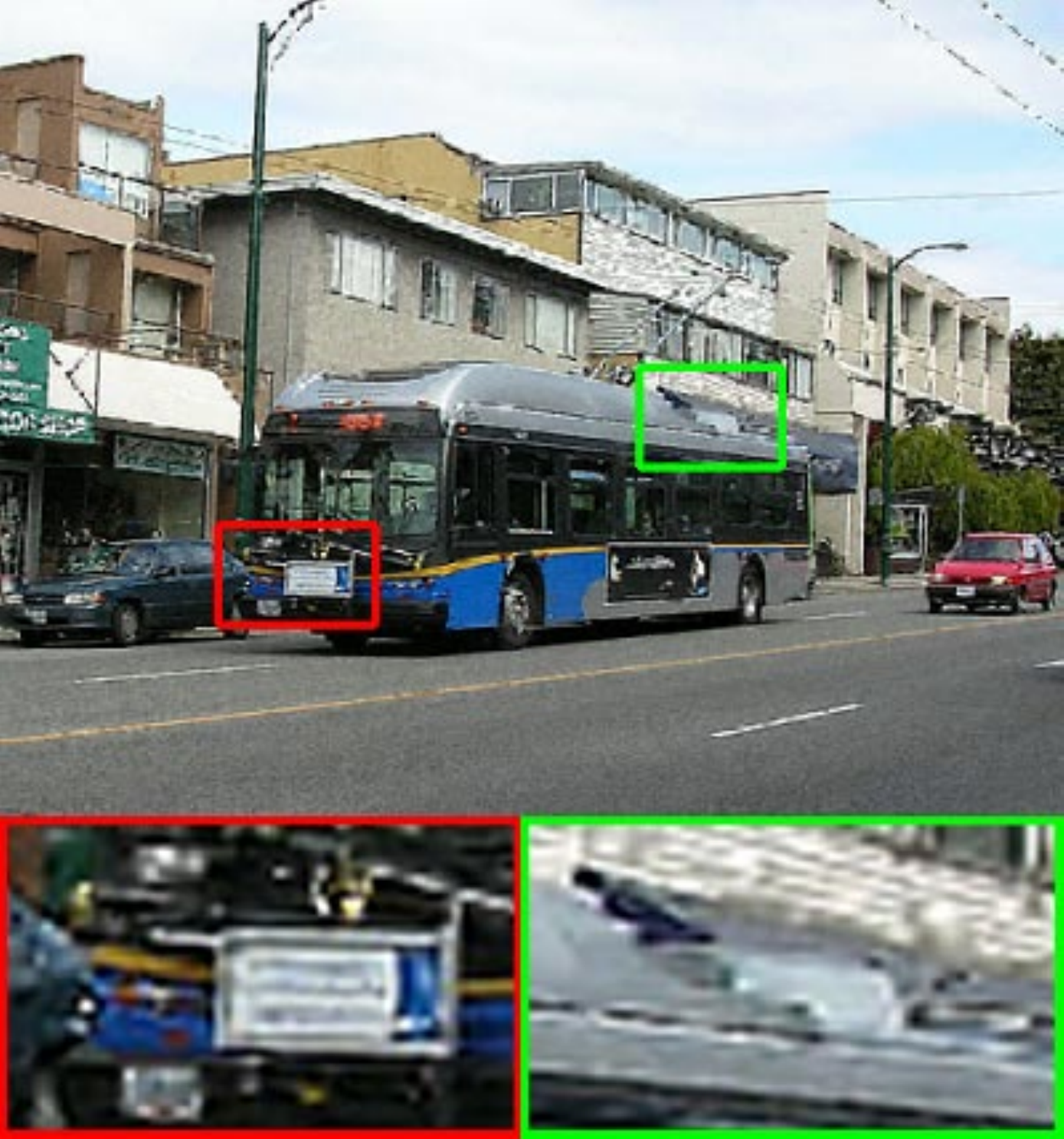}
		&\includegraphics[width=0.15\textwidth]{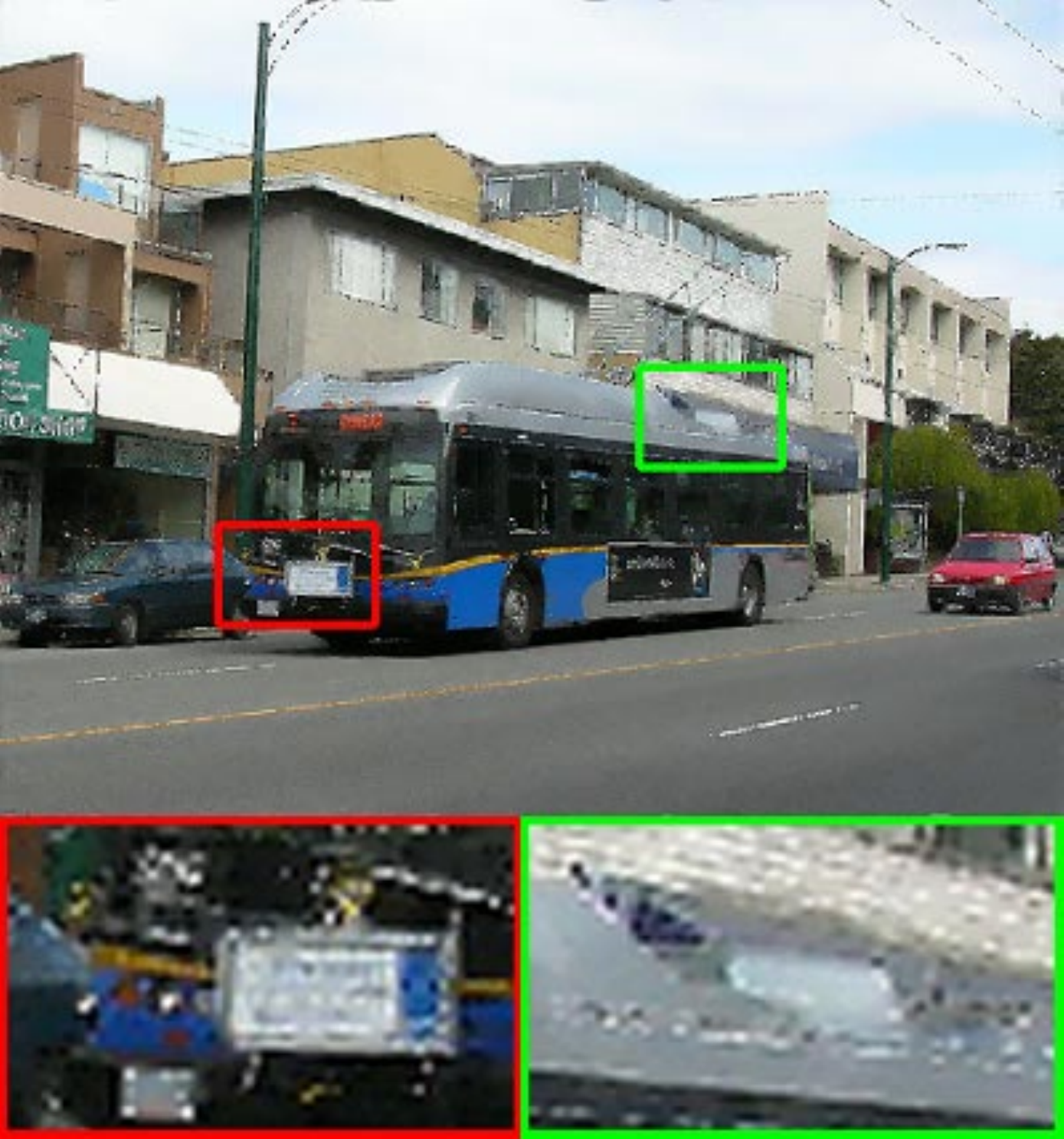}
		&\includegraphics[width=0.15\textwidth]{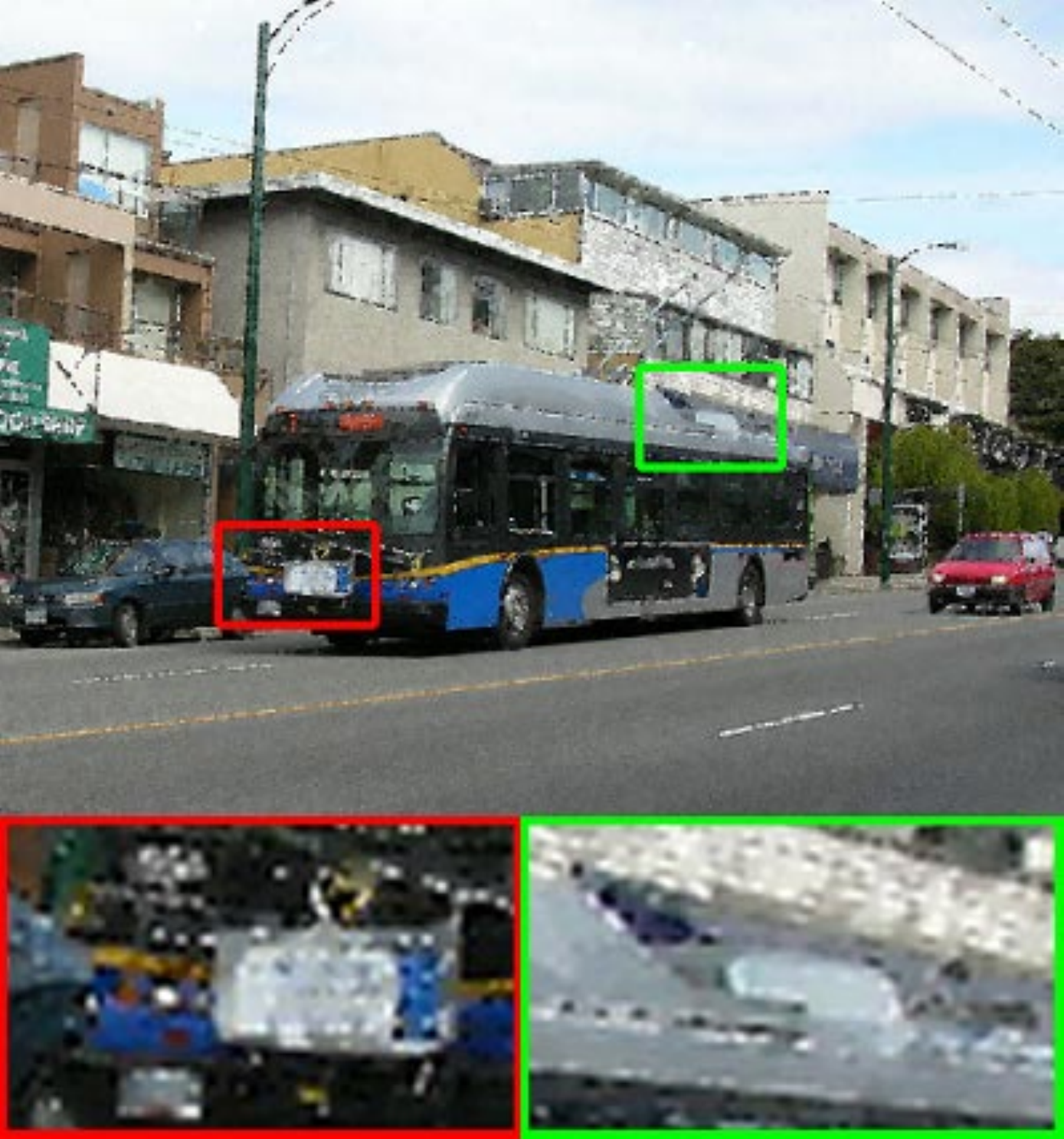}
		&\includegraphics[width=0.15\textwidth]{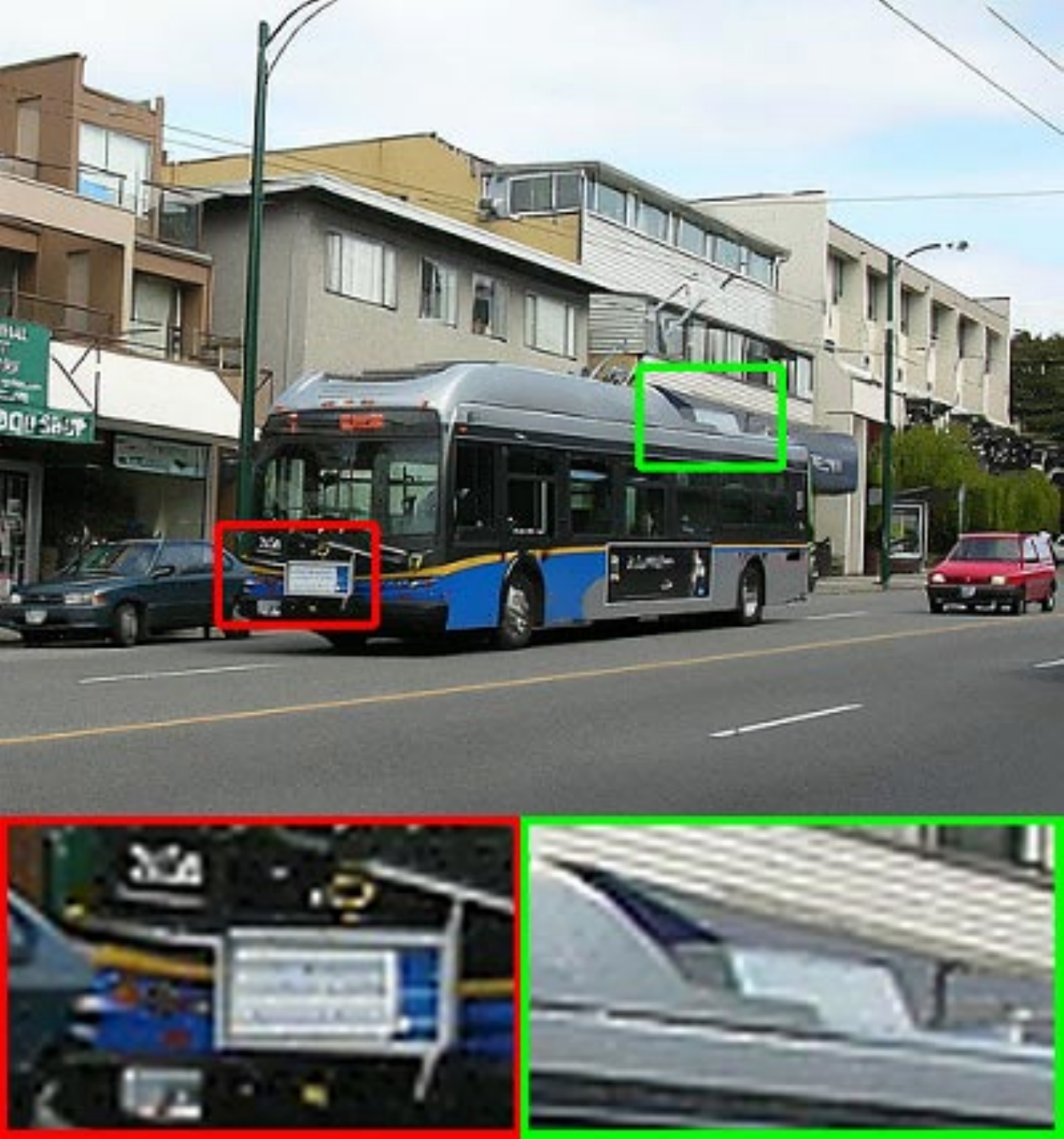}\\
		PSNR~/~SSIM & 24.13~/~0.87 & 22.91~/~0.81 & 21.95~/~0.80 & 21.77~/~0.87 & \textbf{24.45~/~0.89} \\
		(a) Input & (b) ALOHA~\cite{jin2015annihilating} & (c) FoE~\cite{roth2009fields} & (d) ISDSB~\cite{he2014iterative}  & (e) TV~\cite{getreuer2012total} & (f) Ours\\	
	\end{tabular}
	\caption{Image interpolation results with different masks (i.e., text and $\mathbf{60\%}$ random missing pixels). Quantitative metrics are reported below each image.}
	\label{fig:inpaintting result}
\end{figure*}

\begin{figure}[tb]
	\centering \begin{tabular}{c@{\extracolsep{0.2em}}c@{\extracolsep{0.2em}}c@{\extracolsep{0.2em}}c@{\extracolsep{0.2em}}c@{\extracolsep{0.2em}}c}
		\includegraphics[width=0.145\textwidth]{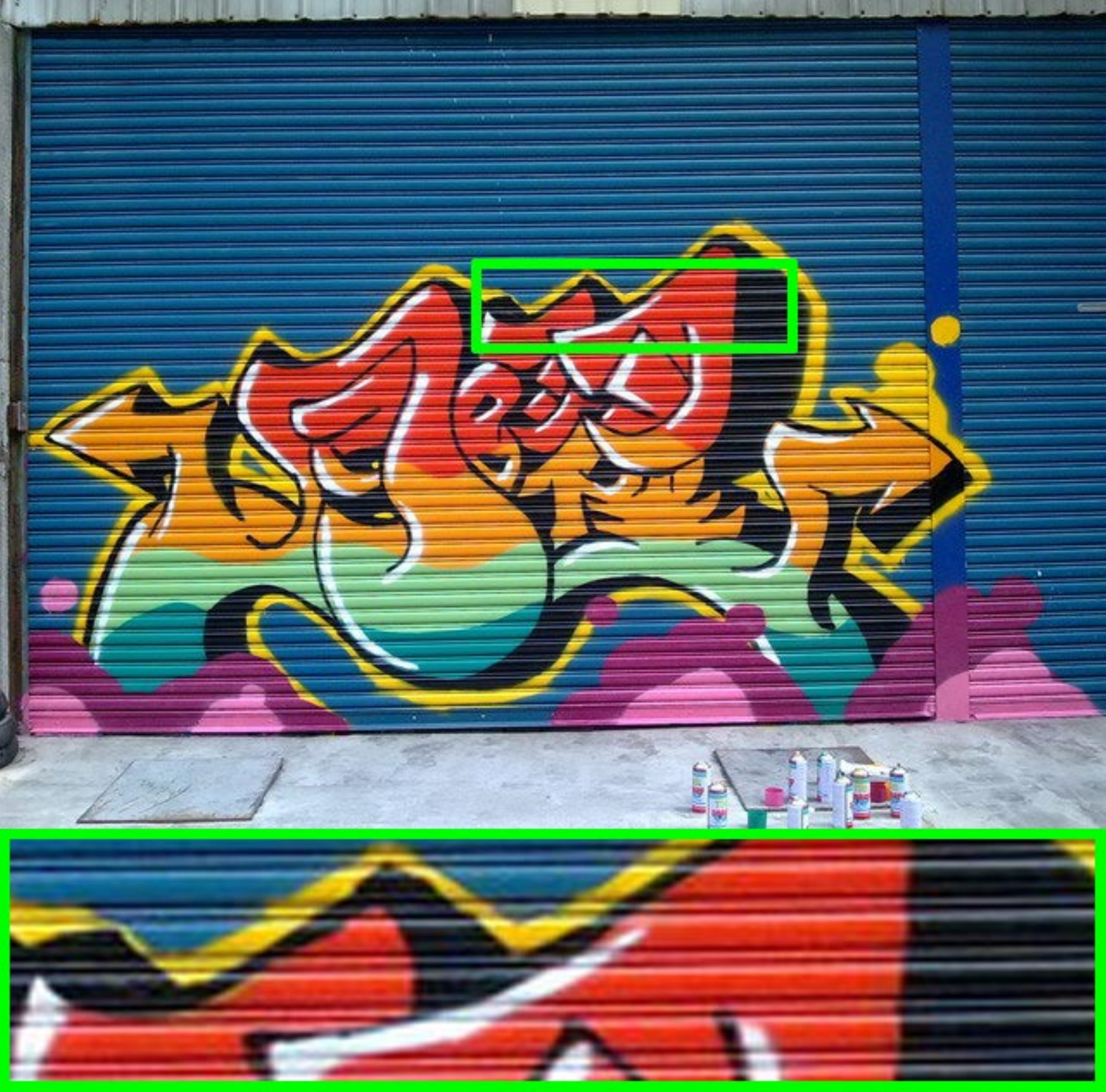}
		&\includegraphics[width=0.145\textwidth]{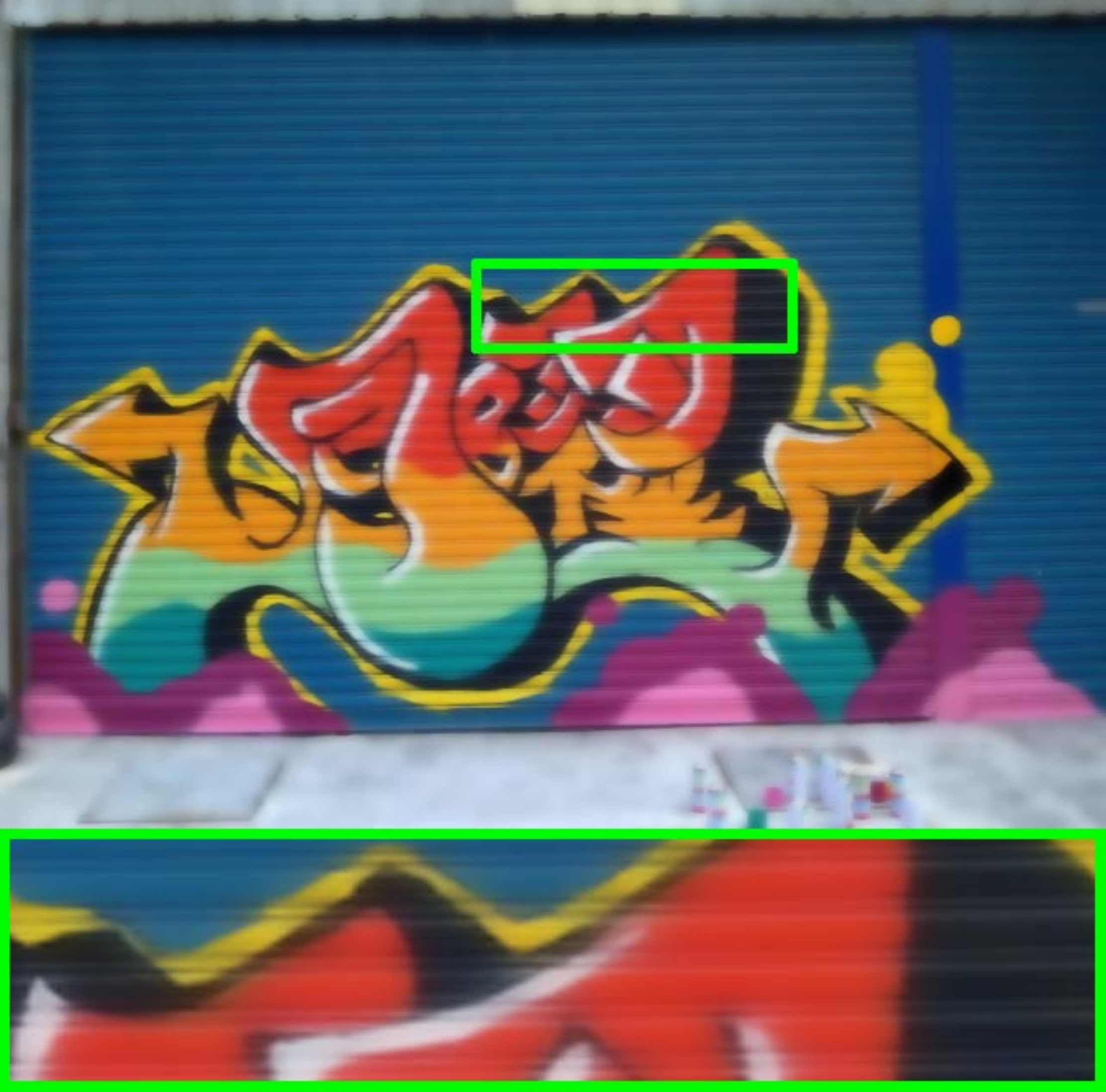}
		&\includegraphics[width=0.145\textwidth]{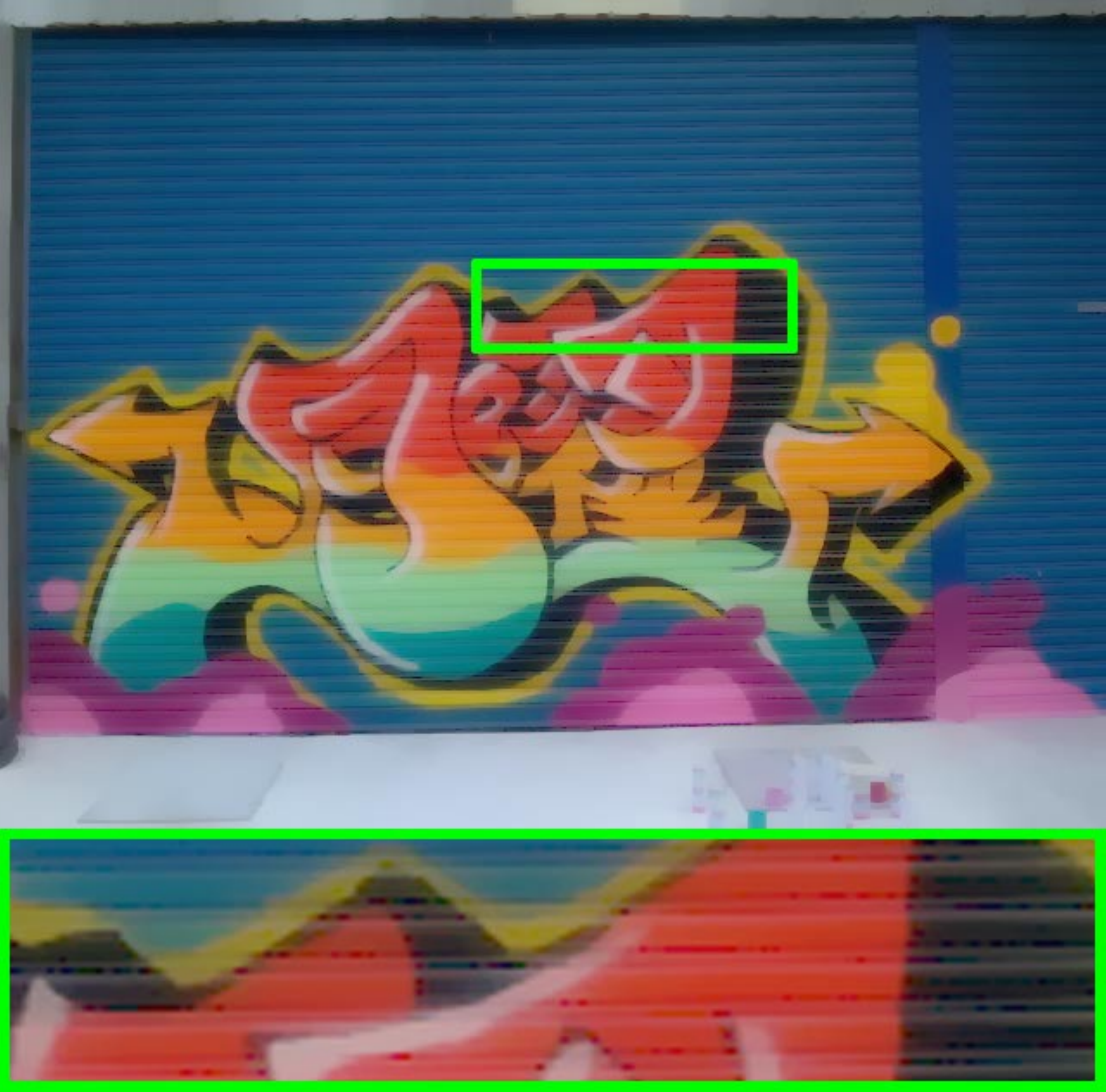}\\
		(a) Input & (b) BLF~\cite{tomasi1998bilateral} & (c) WLS~\cite{farbman2008edge} \\
		\includegraphics[width=0.145\textwidth]{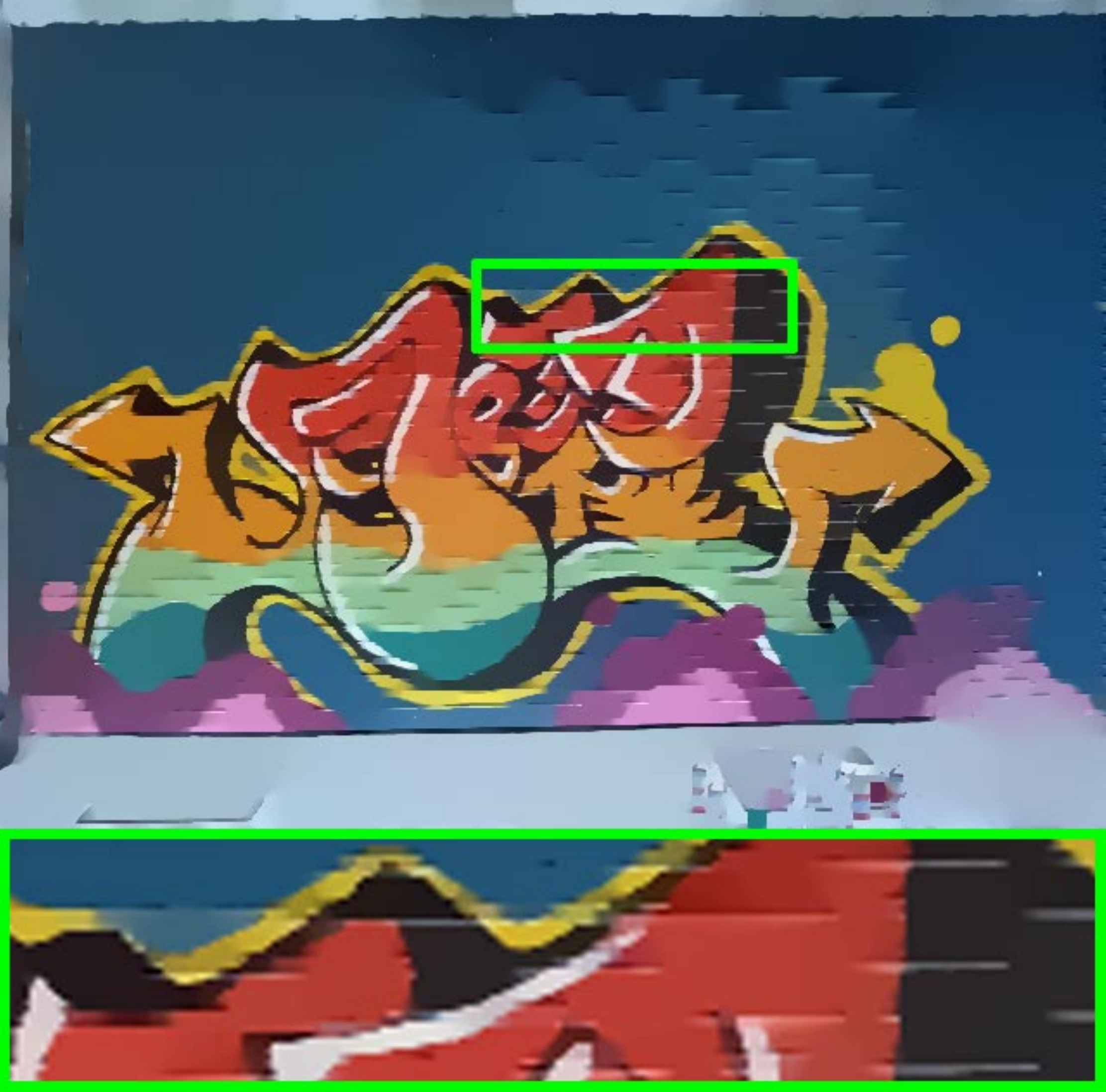}
		&\includegraphics[width=0.145\textwidth]{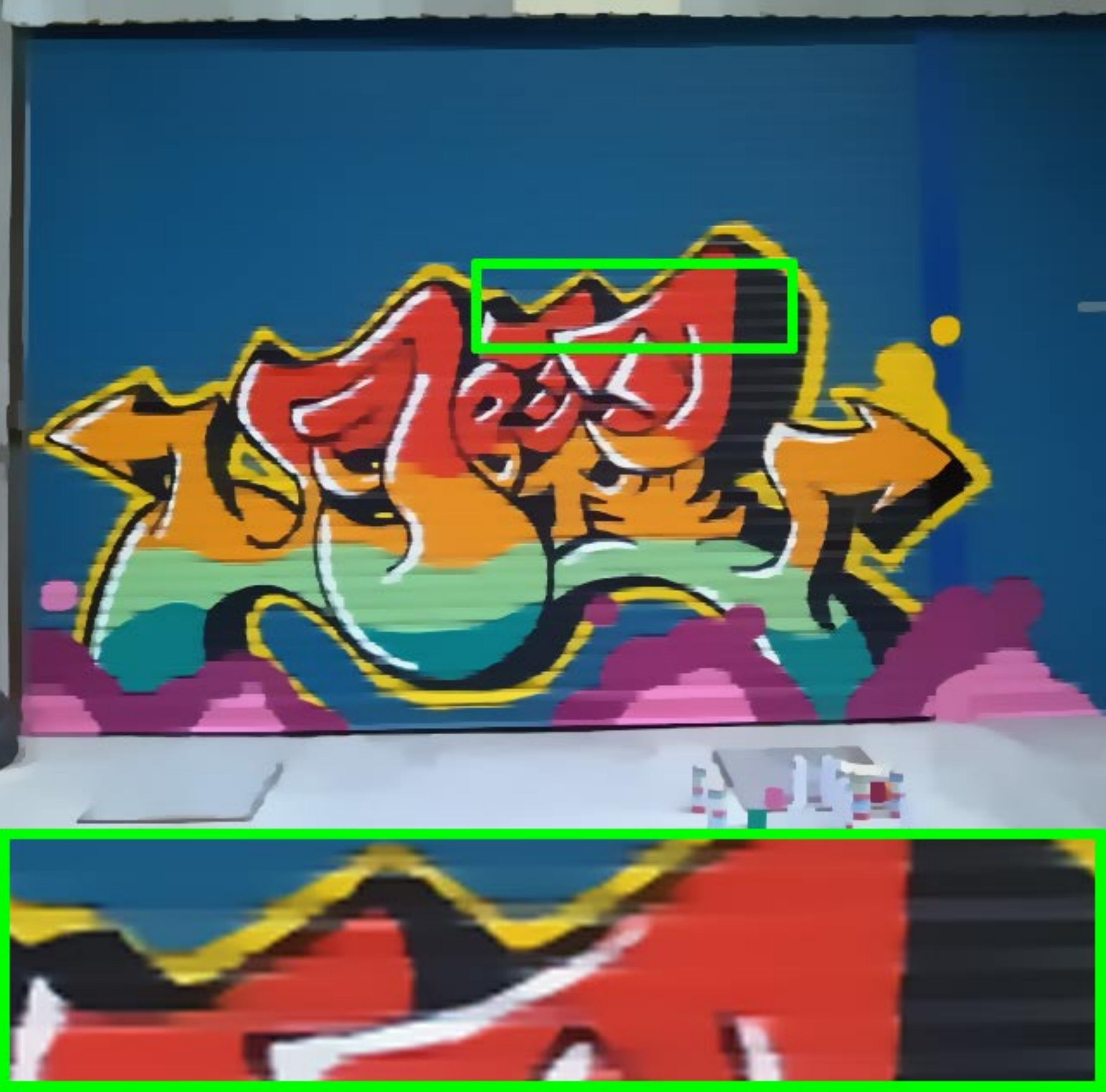}
		&\includegraphics[width=0.145\textwidth]{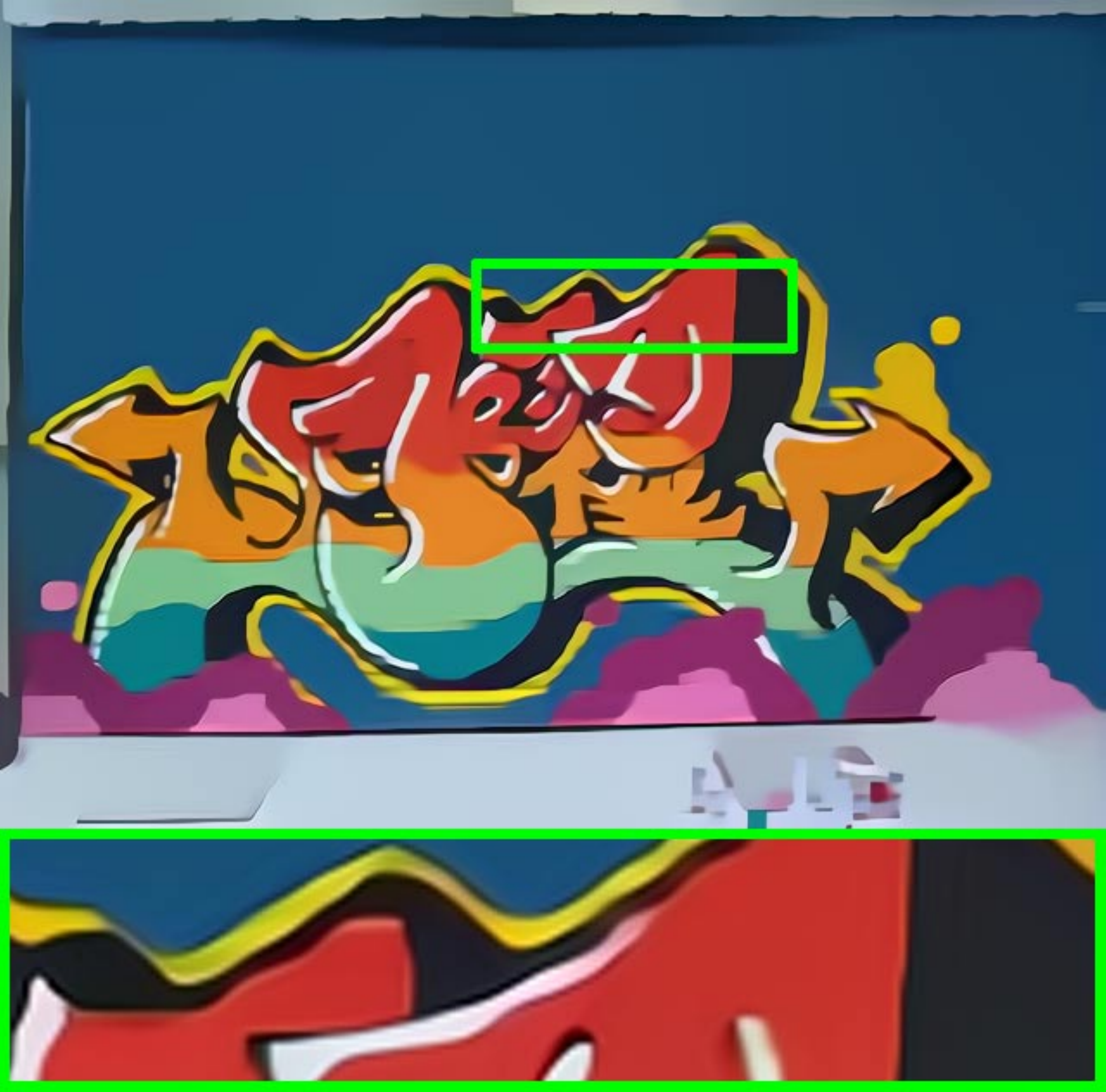}\\
		(d) $L_0$~\cite{xu2011image}  & (e) RTV~\cite{tsmoothing2012} & (f) Ours\\	
	\end{tabular}
	\caption{Edge-preserved smoothing results on an example with abundant textures.}
	\label{fig:smoothing result}
\end{figure}

\subsubsection{Edge-Preserved Smoothing}
Edge-preserved image smoothing is a fundamental tool for image editing and processing, such as pencil sketch rendering~\cite{Lu2012Combining} and cartoon artifact removal~\cite{xu2011image}. 
Here we compare our method with state-of-the-art image smoothing approaches, including
the classic BLF~\cite{tomasi1998bilateral}, WLS~\cite{farbman2008edge} and recently proposed $L_0$~\cite{xu2011image}, RTV~\cite{tsmoothing2012}. Fig.~\ref{fig:smoothing result} illustrates the results on an example image collected by~\cite{tsmoothing2012}. It can be seen that GCM removes most of the horizontal shutter door textures, while there still exists some horizontal lines in the results of other methods.

\begin{figure}[tb]   
	\centering \begin{tabular}{c@{\extracolsep{0.5em}}c@{\extracolsep{0.5em}}c@{\extracolsep{0.5em}}c}							     
		\includegraphics[width=0.11\textwidth]{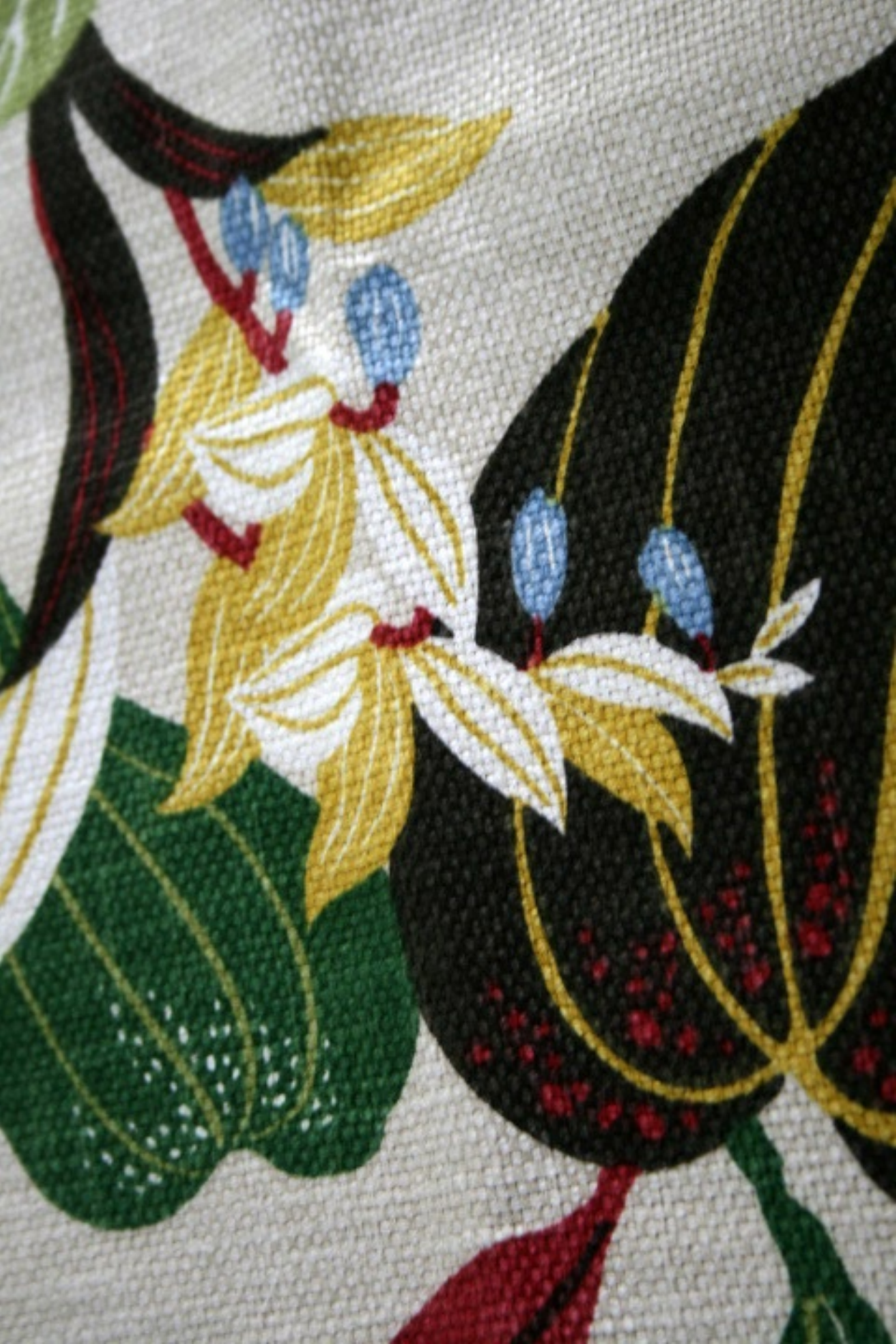}		
		&\includegraphics[width=0.11\textwidth]{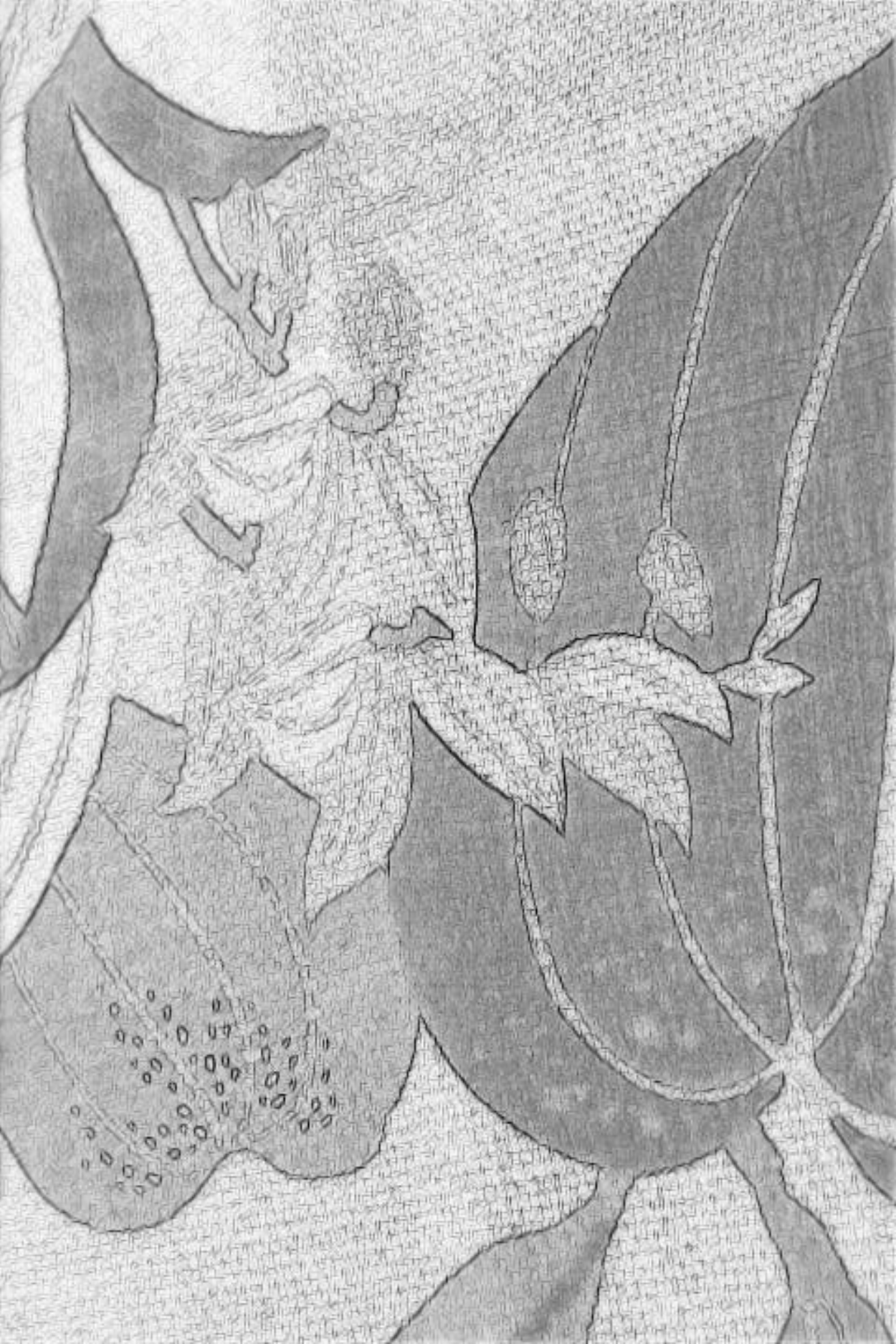}
		&\includegraphics[width=0.11\textwidth]{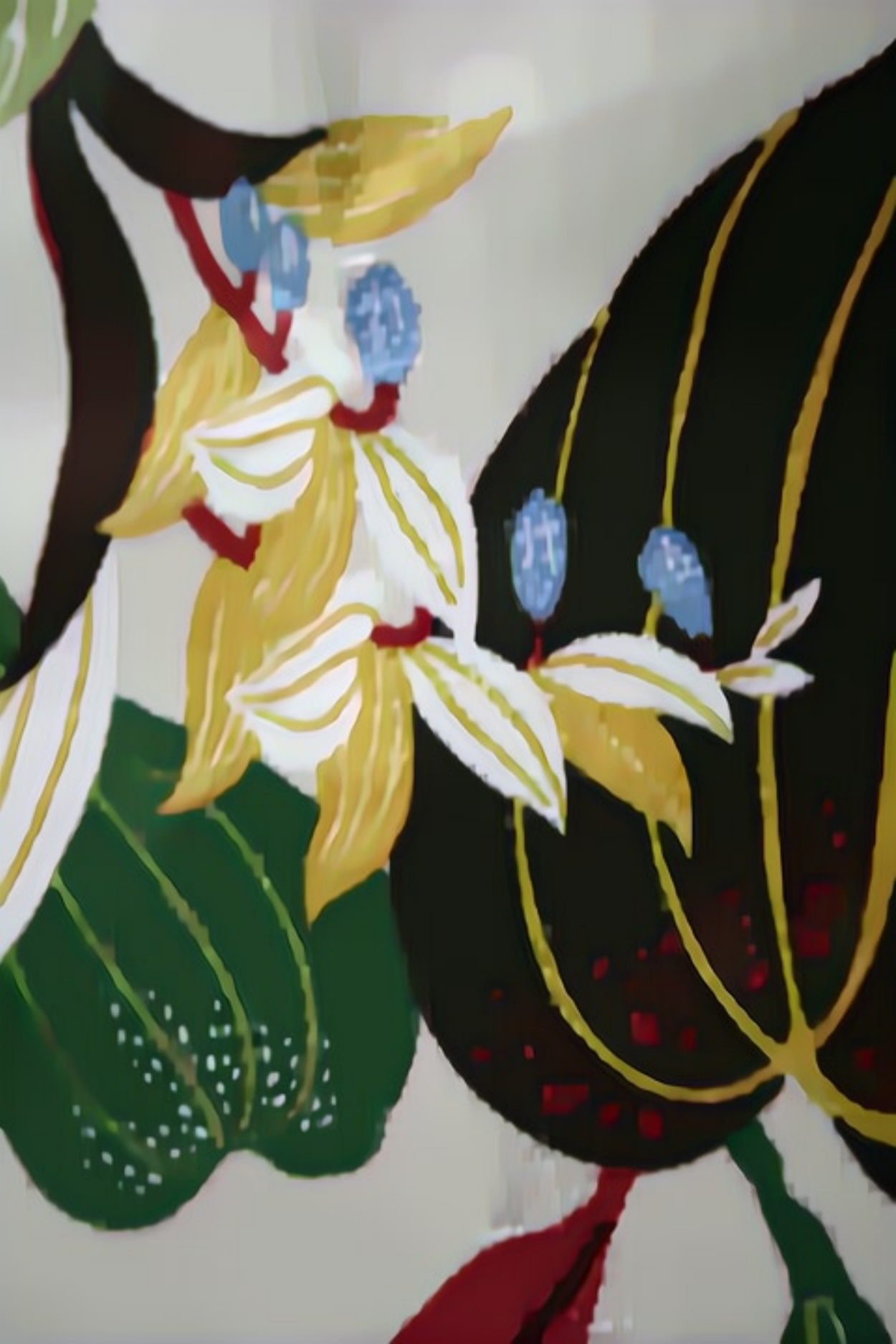}
		&\includegraphics[width=0.11\textwidth]{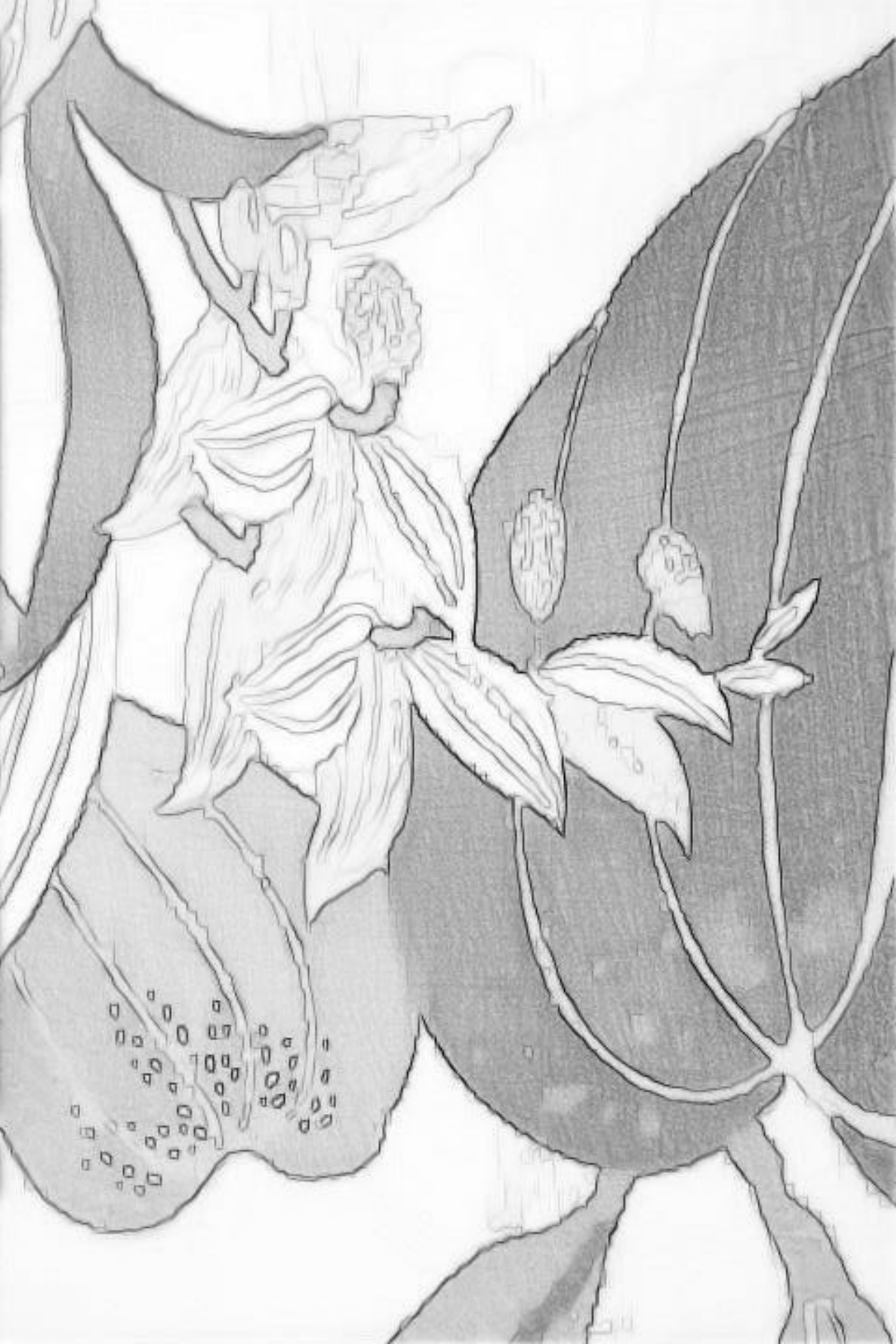}\\
		(a) & (b)  &(c) & (d) \\
	\end{tabular}
	\caption{Pencil sketch rendering results on a non-photorealistic image. The input and its rendered sketch are plotted in sufigures (a) and (b), respectively. The smoothed result of GCM and the corresponding pencil sketch rendering are plotted in subfigures (c) and (d), respectively.}
	\label{fig:pencil sketch result}
\end{figure}
 
To further illustrate the efficiency of our edge-preserved smoothing results, we also employ the method in ~\cite{Lu2012Combining} to perform pencil sketch rendering based on the smoothed results of GCM. In Fig.~\ref{fig:pencil sketch result}, we compare the sketching results based on the original image and our results. It can be seen that GCM suppresses low-amplitude details and enhances high contrast edges, resulting to much better pencil sketch rendering .

\section{Conclusions}
This paper proposed GCM, a collaborative learning framework to estimate the latent image structures. By 
integrating the learnable-architecture-based Generator and the model-driven Corrector in a principled manner, we obtained a convergent image propagation, which can promote kernel estimation for blind image deblurring. As a nontrivial byproduct,
we also extended GCM to address other related vision and multimedia applications. Experimental results demonstrated that GCM achieved better performance than other state-of-the-art approaches on all the test applications.

\appendix
\section{Proof of Theorem~\ref{theo:convergence}}
\label{app:appendix}
To present our proof in a clear manner, here we reorganize the results in Theorem~\ref{theo:convergence} as the following two successive theorems, referring to ``non-increasing properties of the objectives'' (i.e., Theorem~\ref{prop:non-increasing}) and the ``critical point convergence'' (i.e., Theorem~\ref{prop:critical-point}), respectively.
Moreover, our theoretical analysis is based on some mild and widely used function assumptions. That is, we assume that $f$ is Lipschitz smooth, $\phi$ is lower semi-continuous, and $\Psi$ is coercive\footnote{Please follow the references~\cite{rockafellar2009variational,attouch2009convergence} for their formal definitions.}. Fortunately, it is easy to check that all these assumptions are satisfied for the functions considered in this work.

\subsection{Non-increasing Properties of the Objectives}
\begin{thm}\label{prop:non-increasing}
	If $\mu^{t} < 1/L$, 
	both $\{\mathbf{u}^{t}\}$ and $\{\mathbf{v}^{t}\}$ are the sequence generated by GCM, we have the objectives $\{\Psi(\mathbf{u}^{t})\}$ is a non-increasing sequence and satisfied the following relationship:
	\begin{displaymath}
	\Psi(\mathbf{u}^{t+1}) \leq  \Psi(\mathbf{v}^{t+1}) \leq \Psi(\mathbf{u}^{t}), \ \forall \ t \in \{0, 1, 2,\dots\}.
	\end{displaymath}
	
\end{thm}

\begin{proof}
	Firstly, from monotony criterion (step~\ref{step:criterion-begin} to step~\ref{step:criterion-end} in Alg.~\ref{alg:deconvolution}), we have $\Psi(\mathbf{v}^{t+1}) \leq \Psi(\mathbf{u}^{t})$ obviously.
	Then considering the proximal operator in step~\ref{step:pg_u_new}, we will prove $\Psi(\mathbf{u}^{t+1}) \leq \Psi(\mathbf{v}^{t+1})$. 
	From step~\ref{step:pg_u_new} in Alg.~\ref{alg:deconvolution}, we have $\mathbf{u}^{t+1}$ is the optimal solution of following energy function:
	\begin{equation}
	\min\limits_{\mathbf{u}} \phi(\mathbf{u}) + \langle \nabla f(\mathbf{v}^{t+1}), \mathbf{u} - \mathbf{v}^{t+1} \rangle + \frac{1}{2\mu^t}\| \mathbf{u} - \mathbf{v}^{t+1}\|^2.\label{eq:pg-energy}
	\end{equation} 
	Thus, it is easy to obtain the inequality:
	\begin{equation}
	\phi(\mathbf{u}^{t+1}) + \langle \nabla f(\mathbf{v}^{t+1}), \mathbf{u}^{t+1} - \mathbf{v}^{t+1} \rangle + \frac{\mu^t}{2}\| \mathbf{u}^{t+1} - \mathbf{v}^{t+1}\|^2 \leq \phi(\mathbf{v}^{t+1}).\label{eq:criterion}
	\end{equation} 
	On the other hand, we can obtain another unequal relationship by Lipschitz smooth of $f$, i.e.,
	\begin{equation}
	f(\mathbf{u}^{t+1}) \leq f(\mathbf{v}^{t+1}) + \langle \nabla f(\mathbf{v}^{t+1}), \mathbf{u}^{t+1} - \mathbf{v}^{t+1} \rangle + \frac{L}{2} \| \mathbf{u}^{t+1} - \mathbf{v}^{t+1}\|^2,\label{eq:proximal-u-new}
	\end{equation}
	where $L$ is the Lipschitz constant.
	Combing Eqs.~\eqref{eq:criterion} and~\eqref{eq:proximal-u-new}, we have
	\begin{equation}
	\Psi(\mathbf{u}^{t+1}) \leq \Psi(\mathbf{v}^{t+1}) - \left( \frac{1}{2\mu^t} - \frac{L}{2} 	\right) \| \mathbf{u}^{t+1} - \mathbf{v}^{t+1}\|^2.\label{eq:value-inequality}
	\end{equation}
	Setting $\mu^t < \frac{1}{L}$, we have $\Psi(\mathbf{u}^{t+1}) \leq \Psi(\mathbf{v}^{t+1})$.
	So far, we get the following relationship of objectives:
	\begin{displaymath}
	\Psi(\mathbf{u}^{t+1}) \leq \Psi(\mathbf{v}^{t+1}) \leq \Psi(\mathbf{u}^{t}), \ \forall \ t \in \{0,1,2,\dots\},
	\end{displaymath}
	which implies $\{\Psi(\mathbf{u}^{t})\}$ is a non-increasing sequence.
\end{proof}

\subsection{Critical Point Convergence}
\begin{thm}\label{prop:critical-point}
	If $\{\mathbf{u}^t\}$ be the image sequence generated by GCM, we have any accumulation point of $\{\mathbf{u}^t\}$ is the critical point of $\Psi$ (i.e., it satisfies the first-order necessary optimal condition of Eq.~\eqref{eq:model-u}). 
\end{thm}

\begin{proof}
	In this proof, we first verify the existence of accumulation points, then prove any accumulation point of image sequence is the critical point of $\Psi$.
	From Theorem~\ref{prop:non-increasing}, we obtain the non-increasing sequence $\{\Psi(\mathbf{u}^{t})\}$. Since $\Psi$ is coercive, we have the following important inequalities:
	\begin{equation} 
	-\infty < \inf \Psi \leq  \Psi(\mathbf{u}^{t+1}) \leq \Psi(\mathbf{v}^{t+1}) \leq \Psi(\mathbf{u}^{t}) \leq \Psi(\mathbf{u}^{1}).\label{eq:non-increasing}
	\end{equation}  
	It not only indicates sequences $\{\Psi(\mathbf{u}^{t})\}$ and $\{\Psi(\mathbf{v}^{t})\}$ are bounded, but also means the image sequences $\{\mathbf{u}^{t}\}$ and $\{\mathbf{v}^{t}\}$ have accumulation points.
	From Eq.~\eqref{eq:non-increasing}, we find that $\{\Psi(\mathbf{u}^{t})\}$ and $\{\Psi(\mathbf{v}^{t})\}$ share the same limit value $\Psi^{*}$, i.e.,
	\begin{equation} 
	\lim\limits_{t \to \infty} \Psi(\mathbf{u}^{t}) = \lim\limits_{t \to \infty} \Psi(\mathbf{v}^{t}) = \Psi^{*}.\label{eq:value-limit}
	\end{equation}  
	Considering Eq.~\eqref{eq:value-inequality} and \eqref{eq:non-increasing}, the following inequalities are established:
	\begin{equation} 
	\begin{array}{l}
	\quad \left( \frac{1}{2\mu^t} - \frac{L}{2}\right) \|\mathbf{u}^{t+1} - \mathbf{v}^{t+1}\|^2 \\
	\leq \Psi(\mathbf{v}^{t+1}) - \Psi(\mathbf{u}^{t+1}) \\
	\leq \Psi(\mathbf{u}^{t}) - \Psi(\mathbf{u}^{t+1}).
	\end{array}
	\end{equation}   
	Then sum over $t$ to obtain
	\begin{equation} 
	\min\limits_{t}\left\{ \frac{1}{2\mu^t} - \frac{L}{2}  \right\} \sum_{t=0}^{\infty}  \|\mathbf{u}^{t+1} - \mathbf{v}^{t+1}\|^2  \leq \Psi(\mathbf{u}^{0}) - \Psi^{*} < \infty,
	\end{equation} 
	which implies $ \|\mathbf{u}^{t+1} - \mathbf{v}^{t+1}\| \to 0 $ when $t \to \infty$. Thus $\{\mathbf{u}^{t}\}$ and $\{\mathbf{v}^{t}\}$ share the same accumulation points. Assuming the set of accumulation points is $\Omega$ and $\mathbf{u}^{*}$ is one of its elements, i.e., $\mathbf{u}^{t_j} \to \mathbf{u}^{*}$ when $j \to \infty$. 
	Using Eq.~\eqref{eq:pg-energy}, we have the similar inequality with Eq.~\eqref{eq:criterion} as following:
	\begin{equation} 
	\begin{array}{l}
	\quad \phi(\mathbf{u}^{t+1}) + \langle \nabla f(\mathbf{v}^{t+1}), \mathbf{u}^{t+1} - \mathbf{v}^{t+1} \rangle + \frac{\mu^t}{2}\| \mathbf{u}^{t+1} - \mathbf{v}^{t+1}\|^2 \\
	\leq \phi(\mathbf{u}^{*}) + \langle \nabla f(\mathbf{v}^{t+1}), \mathbf{u}^{*} - \mathbf{v}^{t+1} \rangle + \frac{\mu^t}{2}\| \mathbf{u}^{*} - \mathbf{v}^{t+1}\|^2
	\end{array}
	\end{equation} 
	Let $t_j = t+1$, we have $ \lim\sup\limits_{j \to \infty} \phi(\mathbf{u}^{t_j}) \leq \phi(\mathbf{u}^{*})$ by taking $\lim\sup$ on both side of above inequality when $j \to \infty$. Furthermore, since $\phi$ is lower semi-continuous and $\mathbf{u}^{t_j} \to \mathbf{u}^{*}$, which follows $ \lim\sup\limits_{j \to \infty} \phi(\mathbf{u}^{t_j}) \geq \phi(\mathbf{u}^{*})$. Thus $\lim\limits_{j\to\infty} \phi(\mathbf{u}^{t_j}) = \phi(\mathbf{u}^{*}) $ is successful. 
	Note that Lipschitz smooth of $f$ implies $f$ is continuity, which yields $\lim\limits_{j\to\infty} f(\mathbf{u}^{t_j}) = f(\mathbf{u}^{*})$. Thus we conclude 
	\begin{equation} 
	\lim\limits_{j\to\infty} \Psi(\mathbf{u}^{t_j}) = \Psi(\mathbf{u}^{*}).\label{eq:value-Psi}
	\end{equation} 
	Recall that $\lim\limits_{j\to\infty} \Psi(\mathbf{u}^{t+1}) = \Psi^{*} $ in Eq.~\eqref{eq:value-limit}, we have 
	$$\Psi(\mathbf{u}^{*}) = \Psi^{*} .$$
	By first-order necessary optimal condition of Eq.~\eqref{eq:pg-energy} and $t_j = t+1$, we have
	\begin{displaymath}
	\begin{array}{l}
	\quad 0 \in \partial \phi(\mathbf{u}^{t_j}) + \nabla f(\mathbf{v}^{t_j}) + \frac{1}{\mu^{t}} (\mathbf{u}^{t_j} - \mathbf{v}^{t_j} ) \\
	\Leftrightarrow \nabla f(\mathbf{u}^{t_j}) - \nabla f(\mathbf{v}^{t_j}) - \frac{1}{\mu^{t}} (\mathbf{u}^{t_j} - \mathbf{v}^{t_j} ) \in \partial \Psi(\mathbf{u}^{t_j})\\
	\Rightarrow \| \nabla f(\mathbf{u}^{t_j}) - \nabla f(\mathbf{v}^{t_j}) - \frac{1}{\mu^{t}} (\mathbf{u}^{t_j} - \mathbf{v}^{t_j} ) \| \\
	\quad \leq (L +\frac{1}{\mu^{t}} ) \|\mathbf{u}^{t_j} -\mathbf{v}^{t_j} \| \to 0, \ \text{as} \ j \to \infty.
	\end{array}
	\end{displaymath}
	This together with Eq.~\eqref{eq:value-Psi} concludes that 
	\begin{displaymath}
	0 \in \partial \Psi(\mathbf{u}^{*}), \ \forall \ \mathbf{u}^{*} \in \Omega.
	\end{displaymath}
	Therefore, we have that the accumulation point $\mathbf{u}^{*}$ satisfies first-order necessary optimal condition and thus is the critical point of $\Psi$.
	
\end{proof}

\begin{acks}
	This work is partially supported by the National Natural Science Foundation of China (Nos. 61672125, 61733002, 61572096, 61432003 and 61632019), and the Fundamental Research Funds for the Central Universities.
\end{acks}